\documentclass[10pt]{article}

\usepackage[margin=1in]{geometry}
\usepackage{natbib}
\usepackage{parskip}

\usepackage[utf8]{inputenc}
\usepackage[T1]{fontenc}
\usepackage{lmodern}
\usepackage{booktabs}
\usepackage{microtype}
\usepackage[dvipsnames]{xcolor}
\usepackage{amsmath,amssymb}
\usepackage{amsthm}
\usepackage{mathtools}
\usepackage{enumitem}
\usepackage{tikz}
\usetikzlibrary{arrows.meta,positioning,calc}
\usepackage{algorithm}
\usepackage{placeins}
\usepackage{algpseudocode}
\usepackage[colorlinks, linkcolor=blue, citecolor=ForestGreen]{hyperref}
\usepackage{xurl}
\usepackage{etoc}
\usepackage{caption}
\allowdisplaybreaks
\ifdefined\pdfminorversion\pdfminorversion=7\fi
\floatstyle{ruled}
\restylefloat{algorithm}
\algrenewcommand\algorithmicrequire{\textbf{Input:}}
\algrenewcommand\algorithmicensure{\textbf{Output:}}

\newtheorem{assumption}{Assumption}
\newtheorem{lemma}{Lemma}
\newtheorem{theorem}{Theorem}
\newtheorem{proposition}{Proposition}
\newtheorem{corollary}{Corollary}
\newtheorem{remark}{Remark}
\newtheorem{example}{Example}

\newcommand{\E}{\mathbb{E}}
\newcommand{\Prob}{\mathbb{P}}
\newcommand{\Var}{\mathbb{V}}
\newcommand{\cS}{\mathcal{S}}
\newcommand{\cA}{\mathcal{A}}
\newcommand{\cG}{\mathcal{G}}
\newcommand{\cV}{\mathcal{V}}
\newcommand{\cE}{\mathcal{E}}
\newcommand{\cF}{\mathcal{F}}
\newcommand{\cM}{\mathcal{M}}
\newcommand{\supp}{\mathrm{supp}}
\newcommand{\VarStar}{\mathrm{Var}^{\star}}
\DeclareMathOperator*{\argmax}{arg\,max}
\newcommand{\cR}{\mathcal{R}}
\newcommand{\Viol}{\mathrm{Viol}}
\newcommand{\Ind}{\mathbf{1}}
\newcommand{\spann}{\mathrm{span}}
\newcommand{\cC}{\mathcal{C}}

\providecommand{\Est}{\cE_{\mathrm{est}}}
\providecommand{\Cert}{\cE_{\mathrm{cert}}}
\providecommand{\Mart}{\cE_{\mathrm{mart}}}
\providecommand{\MainEvent}{\cE_{\mathrm{main}}}

\makeatletter
\newcommand{\SVAEtheoremheadingtarget}{%
  \def\Hy@theorem@makelinktarget##1{%
    \let\SVAE@originalthmhead\thmhead
    \def\thmhead####1####2####3{%
      \MakeLinkTarget*{##1}%
      \SVAE@originalthmhead{####1}{####2}{####3}%
    }%
  }%
}
\makeatother%

\title{Instance-Dependent Regret for CMDPs with Step-Wise Constraints}
\author{%
\begin{tabular}{c@{\hspace{2em}}c}
Qian Zuo & Francesco Emanuele Stradi \\
{\small University of Edinburgh} & {\small Politecnico di Milano} \\
{\small\texttt{qian.zuo@ed.ac.uk}} & {\small\texttt{francescoemanuele.stradi@polimi.it}} \\[1em]
Leyang Xue & Sattar Vakili \\
{\small University of Edinburgh} & {\small University College London} \\
{\small\texttt{lxue2@ed.ac.uk}} & {\small\texttt{s.vakili@ucl.ac.uk}}
\end{tabular}%
}%
\date{}

\begin{document}
\etocdepthtag{main}

\maketitle

\begin{abstract}
We study online learning in episodic tabular constrained Markov decision processes with \emph{step-wise safety constraints}. In such a setting, the constraints induce a safe subgraph that shapes the variance of cumulative rewards under feasible policies and, consequently, the difficulty of learning. Exploiting this structure, however, requires learning which actions are safe while controlling constraint violations. We propose \emph{Safe Variance-Adaptive Exploration} (SVAE), an efficient algorithm that learns candidate safe subgraphs and performs variance-adaptive optimistic planning within them. With high probability, SVAE achieves cumulative
regret of order
$\widetilde{\mathcal O}(
\sqrt{SAH\min\{\Var_\Sigma,K\VarStar\}}
+S\sqrt{AH^3\min\{K,\cC\}}
+S^2AH^2)$ over $K$ episodes, where $H$ is the horizon of a single episode, while $S$ and $A$ are the numbers of states and actions, respectively. Here, $\VarStar$ is the maximum return variance among safe policies, $\Var_\Sigma$ is the variance accumulated before the first unsafe action is encountered, and $\cC$ captures the statistical complexity of eliminating actions incorrectly considered potentially safe. SVAE additionally attains $\widetilde{\mathcal{O}}(H\sqrt{SAK}+S^2AH^2)$ step-wise constraint violation and a gap-dependent violation bound that is polylogarithmic in $K$. Finally, we establish a lower bound showing that dependence on these instance-specific quantities is unavoidable.

\end{abstract}

\section{Introduction}
\label{sec:intro}

Constrained Markov decision processes (CMDPs) are widely used to model sequential decision making in robotics, autonomous driving, and wireless networks~\citep{achiam2017constrained,liu2021learning,bura2022dope}.
The objective is to maximize expected cumulative reward subject to safety or resource constraints, commonly imposed on expected cumulative costs~\citep{altman1999constrained,miryoosefi2022simple,stradi2025learning}.
The potential harm from individual unsafe decisions motivates step-wise safety constraints, which impose a mean-cost threshold at each decision~\citep{xiong2023provably,shi2023near}.
Violation is measured by summing positive excesses above this threshold without cancellation~\citep{wei2024instantaneous}.
In this work, we study online learning in episodic tabular constrained MDPs under this kind of \emph{step-wise} constraint.

Instance-dependent regret bounds distinguish the statistical difficulty of individual problems---namely, instances---and provide sharper guarantees on easier ones~\citep{zanette2019tighter,dann2021beyond}.
In unconstrained reinforcement learning, variance-dependent analyses relate regret to the uncertainty induced by rewards and transitions~\citep{talebi2018variance,zhou2023sharp,zhao2023variance,chen2025sharp}.
Such bounds can retain worst-case guarantees while improving on low-variance instances, including nearly constant regret in deterministic environments~\citep{zhou2023sharp,zamir2026optimal}.
In CMDPs, unknown costs and transitions determine which policies are feasible~\citep{bura2022dope,yu2025improved}.
Exploratory policies can be unsafe, so their returns need not share the low variance of safe policies.
Whether the benefits of variance-dependent
learning can be retained under step-wise constraints is still unclear.

Step-wise safety constraints can be represented by a safe subgraph that retains actions satisfying the cost threshold and leading only to states with feasible future decisions~\citep{xiong2023provably,shi2023near}.
Our key observation is that even though learning under unknown constraints generally poses additional challenges, safety constraints can still simplify reward learning by removing stochastic branches or reducing differences in
expected future rewards across possible next states under feasible policies. %
Even when the rewards and the transition kernel are fixed, changing costs can alter the safe subgraph and thereby reshape variances.
However, small cost gaps and rare transitions to states with no feasible policy can still make safety identification difficult.
Understanding how the safe subgraph relates to the variance and the statistical complexity is therefore central to instance-dependent regret analysis. These observations lead to the following questions:
\begin{quote}
\begin{center}
\emph{(i) Is it possible to design a computationally efficient algorithm that achieves variance-dependent regret bounds and sublinear cumulative step-wise violation?\\ (ii) How does the induced safe subgraph shape the statistical complexity of learning under step-wise constraints?}
\end{center}
\end{quote}
To answer these questions, we first characterize how the induced safe subgraph shapes the variances and the safety identification complexity. We then propose Safe Variance-Adaptive Exploration (SVAE), which separates safe-subgraph learning from variance-adaptive planning.
The algorithm learns candidate safe actions using variance-independent cost lower confidence bounds and observed successors, and plans optimistically over these sets using empirical reward and transition variances. We further evaluate SVAE and examine the effect of safe-subgraph properties on regret in synthetic experiments (Appendix~\ref{app:experiments}).

Our contributions can be summarized as follows.
\begin{itemize}[leftmargin=*,itemsep=3pt]

\item \textbf{Safe-subgraph structure and variance}
In Section~\ref{sec:topology}, we represent step-wise safety constraints through their induced safe subgraph and define two distinct variance quantities: $(i)$
the maximum safe-policy return variance $\VarStar$ is the largest variance of the total episode return over all safe policies, and $(ii)$ the trajectory-based total variance $\Var_\Sigma$ sums reward variances and transition variances before the first unsafe action in each episode.
We establish a tight bound on transition variance in terms of stochastic branching depth and  differences in expected
cumulative rewards across successor states within the safe subgraph.%

\item \textbf{Safety identification complexity}
In Section~\ref{subsec:safety-identification}, we introduce a reward-independent instance quantity $\cC$ %
to measure the difficulty of ruling out unsafe actions at states reachable under safe policies.
Lower bounds for valid elimination establish its necessity.

\item \textbf{Efficient variance-adaptive exploration} In Section~\ref{sec:algorithm}, we develop SVAE, an efficient algorithm that exploits low return variance within the unknown safe subgraph without prior knowledge of a safe policy.
SVAE maintains candidate action sets that retain all safe actions with high probability and uses the resulting candidate subgraph to adapt its exploration bonuses to empirical reward and next-state value variances.

\item \textbf{Instance-dependent learning guarantees}
In Section~\ref{subsec:regret}, we show that our algorithm attains, for a CMDP with $S$ states, $A$ actions, horizon $H$, a cumulative regret bound over $K$ episodes of order
$\widetilde{\mathcal O}(
\sqrt{SAH\min\{\Var_\Sigma,K\VarStar\}}
+S\sqrt{AH^3\min\{K,\cC\}}
+S^2AH^2)$.
Here $\widetilde{\mathcal O}$ hides polylogarithmic factors.
Regret is polylogarithmic in $K$ on fixed instances with $\VarStar=0$, while both regret and cumulative step-wise violation retain $\widetilde{\mathcal O}(\sqrt K)$ worst-case dependence. Step-wise violation is also polylogarithmic in $K$ for a fixed positive cost gap (Theorem~\ref{thm:violation}).

\item \textbf{Regret lower bound}
In Section~\ref{sec:lower-bound}, we establish a regret lower bound, showing that both return variance and safety identification dependencies are unavoidable.
\end{itemize}

\subsection{Related Work}
\label{sec:related-work}
\paragraph{Safe reinforcement learning}
Safe RL is a fundamental area of reinforcement learning, with diverse
constraint formulations and measures of constraint violation.
For episode-level constraints, regret and constraint-violation guarantees
have been established in tabular settings~\citep{efroni2020exploration,
liu2021learning,bura2022dope,wei2022triple,stradi24a,stradi2025strong}
and with linear function approximation~\citep{ding2021provably,
ghosh2022provably,kitamura2025provably}.
For instantaneous constraints, \citet{amani2021safe} and
\citet{roknilamouki2025provably} ensure constraint satisfaction with high
probability using known safe actions or policies.
Without a known safe policy, \citet{xiong2023provably} and
\citet{wei2024instantaneous} bound cumulative positive step-wise violation.
\citet{xiong2023provably} study tabular MDPs with costs that depend
only on the state.
\citet{wei2024instantaneous} consider linear MDPs and assume that
a feasible action exists at every state and step.
We study tabular MDPs with unknown state--action costs and require only
the existence of a safe policy from the initial state.

\paragraph{Variance-dependent regret}
Variance-dependent regret bounds can improve on worst-case rates in low-variance
environments. Such guarantees have been established for episodic tabular
MDPs~\citep{zhou2023sharp}, linear-mixture MDPs~\citep{zhao2023variance}, and
infinite-horizon MDPs~\citep{zamir2026optimal}, with recent work also combining
variance and suboptimality gaps~\citep{chen2025sharp}.
In particular, \citet{zhou2023sharp} derive bounds involving a maximum
policy-return variance measure and a variance sum along the learner's trajectories.
For CMDPs, \citet{yu2025improved} use total-variance analysis to improve the
horizon dependence of a worst-case regret bound under episode-level constraints.
\citet{jiang2024achieving} obtain a different form of instance dependence through
an optimal-basis characterization of the CMDP linear program.
We control step-wise violation and derive a regret bound that
separates safe-policy return variance from the complexity of ruling out actions
outside the true safe set.

\paragraph{Graph structure and safe exploration}
Graph structure has been used to guide safe exploration and refine regret bounds.
Safe-region methods use reachability and returnability to expand the region that
can be explored safely~\citep{turchetta2016safe,wachi2020safe}.
\citet{shi2023near} use safe subgraphs for learning under instantaneous
constraints, assuming known transition supports and a known seed safe subgraph.
For unconstrained MDPs, \citet{bourel2020tightening} combine variance-aware
concentration with transition-support adaptation and obtain regret bounds
involving local diameter and effective support.
Our structural analysis examines how safety constraints change the graph
available to safe policies and how this change affects their transition variance.
We use this relation to bound the transition-variance contribution to regret in
terms of stochastic branching and continuation-value differences within the safe subgraph.

Due to space constraints, additional related works are provided in
Appendix~\ref{app:additional-related-work}.%

\section{Problem Formulation}
\label{sec:formulation}

\paragraph{Episodic CMDP}
We consider an episodic constrained Markov decision process (CMDP)
$\cM=(\cS,\cA,H,P,r,c,\tau)$ with a state space $\cS$ of size $S$,
an action space $\cA$ of size $A$, and horizon $H$.
For a positive integer $m$, let $[m]:=\{1,\ldots,m\}$.
The transition kernel is
$P=\{P_h:\cS\times\cA\rightarrow\Delta_{\cS}\}_{h=1}^H$,
where $\Delta_{\cS}$ denotes the set of probability distributions on $\cS$,
and $P_h(s'\mid s,a)$ denotes the transition probability from $(s,a)$
to $s'$ at step $h$.
$r=\{r_h:\cS\times\cA\rightarrow[0,1]\}_{h=1}^H$ represents the reward
function, and $r_h(s,a)$ is the reward of taking action $a$ in state $s$
at step $h$.
A non-stationary policy
$\pi=\{\pi_h:\cS\rightarrow\cA\}_{h=1}^H\in\Pi$
specifies an action $\pi_h(s)$ at every state and step, and $\Pi$ denotes
the set of all such policies.
The learner interacts with the CMDP for $K$ episodes, each starting
from a fixed initial state $s_1\in\cS$.
For any episode $k$, the learner first chooses a policy $\pi^k$
based on previous observations.
At step $h\in[H]$ of episode $k$, the learner observes a state $s_h^k$,
chooses an action $a_h^k$, and transitions to a next state
$s_{h+1}^k$ according to $P_h(\cdot\mid s_h^k,a_h^k)$.
For a policy $\pi\in\Pi$, the state-value function $V_h^\pi(s)$ and
the state-action value function $Q_h^\pi(s,a)$ are defined as:
\[
V_h^\pi(s):=\E^\pi\!\left[\sum_{t=h}^H r_t(s_t,a_t)\,\middle|\,s_h=s\right],
\qquad
Q_h^\pi(s,a):=\E^\pi\!\left[\sum_{t=h}^H r_t(s_t,a_t)\,\middle|\,s_h=s,a_h=a\right].
\]
The expectations start from $s$ at step $h$ and follow $\pi$,
with the first action fixed to $a$ for $Q_h^\pi(s,a)$.
We set $V_{H+1}^\pi\equiv0$ and let
$\mu_h^\pi(s):=\Prob^\pi(s_h=s\mid s_1)$ denote the step-$h$
state occupancy under $\pi$, for $h\in[H+1]$.

Similarly, $c=\{c_h: \mathcal{S}\times\mathcal{A}\rightarrow [0,1]\}_{h=1}^H$ denotes the cost function, and $c_h(s,a)$ is the cost of taking action $a$ in state $s$ at step $h$. Let $\tau\in[0,1]$ denote the safety threshold. Under stochastic settings, the non-stationary reward and cost of state-action pair $(s,a)$ at step $h$ are $R_h(s,a)$ and $C_h(s,a)$ supported on $[0,1]$, with means $r_h(s,a)$ and $c_h(s,a)$, respectively. Conditional on the selected state--action pairs,
reward and cost observations are independent across visits, and each observation is
conditionally independent of the next state. Write $R_h^k$ and $C_h^k$ for the samples observed at step $h$
of episode $k$. The threshold $\tau$ is known, whereas
$P=\{P_h\}_{h=1}^H$, $r=\{r_h\}_{h=1}^H$, and $c=\{c_h\}_{h=1}^H$ are unknown.

\paragraph{Safe subgraph}
The CMDP induces a layered directed graph $\cG=(\cV,\cE)$ with vertex set
$\cV=\cS\times[H+1]$ \citep{jin2020learning}.
For $h\in[H]$, the edge $((s,h),a,(s',h+1))$
belongs to $\cE$ exactly when $P_h(s'\mid s,a)>0$.
The safe subgraph
$\cG_{\mathrm{safe}}=(\cV_{\mathrm{safe}},\cE_{\mathrm{safe}})$ is defined
through safe actions and safe vertices by backward recursion.
Set $\cV_{\mathrm{safe},H+1}:=\cS$.
For $h=H,\ldots,1$, define recursively %
\begin{align}
\cA_{\mathrm{safe},h}(s)
&:=\bigl\{a\in\cA:c_h(s,a)\le\tau,\
\supp(P_h(\cdot\mid s,a))\subseteq\cV_{\mathrm{safe},h+1}\bigr\},
\label{eq:safe-action}\\
\cV_{\mathrm{safe},h}
&:=\{s\in\cS:\cA_{\mathrm{safe},h}(s)\ne\varnothing\}.
\notag
\end{align}
An action $a$ is safe at $(s,h)$ if $c_h(s,a)\le\tau$ and all its
successor vertices are safe.
A vertex $(s,h)$ is safe when $s\in\cV_{\mathrm{safe},h}$.
The vertex and edge sets of $\cG_{\mathrm{safe}}$ are
$\cV_{\mathrm{safe}}:=\bigl\{(s,h)\in\cV:s\in\cV_{\mathrm{safe},h}\bigr\}$ and
$\cE_{\mathrm{safe}}
:=\bigl\{((s,h),a,(s',h+1))\in\cE:
 a\in\cA_{\mathrm{safe},h}(s)\bigr\}$.

For a policy $\pi\in\Pi$, let
$\mathcal{G}^\pi=(\cV^\pi,\cE^\pi)$ denote the subgraph generated by
$\pi$ from $s_1$, containing the vertices and edges
visited with positive probability~\citep{shi2023near}.
Its vertex and edge sets are
$\cV^\pi:=\{(s,h)\in\cV:\mu_h^\pi(s)>0\}$ and
$\cE^\pi:=\{((s,h),a,(s',h+1))\in\cE:
(s,h)\in\cV^\pi,\ a=\pi_h(s)\}$, respectively.
Thus, $\mathcal{G}^\pi$ includes all positive-probability successors
of the action selected at each reachable vertex.
We write $\mathcal{G}^\pi\subseteq\cG_{\mathrm{safe}}$ when
$\cV^\pi\subseteq\cV_{\mathrm{safe}}$ and
$\cE^\pi\subseteq\cE_{\mathrm{safe}}$. Therefore, the safe policy class is
$\Pi_{\mathrm{safe}}:=\{\pi\in\Pi:
\mathcal{G}^\pi\subseteq\cG_{\mathrm{safe}}\}$. Let $\pi^\star\in\argmax_{\pi\in\Pi_{\mathrm{safe}}}V_1^\pi(s_1)$
be an optimal safe policy, and write $V_h^\star(s)=V_h^{\pi^\star}(s)$.

\paragraph{Performance metrics}
For the policies $\pi^1,\ldots,\pi^K$ executed by the learner, we define the cumulative regret and cumulative step-wise violation as:
\begin{align}
\cR(K)
&:=\sum_{k=1}^K\bigl(V_1^{\star}(s_1)-V_1^{\pi^k}(s_1)\bigr),
\label{eq:regret}\\
\Viol(K)
&:=\sum_{k=1}^K\sum_{h=1}^H
\bigl[c_h(s_h^k,a_h^k)-\tau\bigr]_+,
\label{eq:viol}
\end{align}
where $[\cdot]_+:=\max\{\cdot,0\}$.
The regret $\cR(K)$ %
compares the learner's cumulative expected return
with that of the optimal safe policy, while $\Viol(K)$ accumulates
positive mean-cost excesses along the learner's trajectories.
\citet{xiong2023provably} study step-wise violation with state-only costs
$c(s)$, whereas our costs $c_h(s,a)$ depend on both the state and action
and may vary across steps.
Unlike constraints on expected cumulative costs, Equation~\eqref{eq:viol} does not allow excess costs
at one step to be offset by lower costs at another.
We allow violations during exploration and control their cumulative
magnitude, rather than require zero violation throughout learning as in
\citet{shi2023near}.
Our goal is to obtain high-probability sublinear upper bounds in $K$ on regret
and violation, with regret adapting to the variance of safe policies and
the difficulty of identifying actions outside the safe action sets.

Our results generally require only that the initial state admit an
unknown safe policy, which is a mild condition in the literature~\citep{amani2021safe,xiong2023provably,shi2023near,
wei2024instantaneous,roknilamouki2025provably}.
\begin{assumption}[Feasibility]
\label{ass:init}
The initial state satisfies $s_1\in\cV_{\mathrm{safe},1}$.
\end{assumption}

\section{Safety-Induced Variance and Identification Complexity}
\label{sec:topology}

Safety constraints shape the statistical difficulty of reinforcement learning,
not only the set of feasible policies. They can impede exploration, but by
restricting behavior they can also remove sources of reward uncertainty and
simplify learning inside the feasible region. Under step-wise safety, these two
effects are organized by the hidden safe subgraph. We show that its stochastic
branching and successor-value differences control the transition variance
available to reward learning, while actions outside the graph must be excluded
based on cost or successor-viability evidence. This yields tight structural
variance bounds together with a complexity for the separate
task of safety identification.

\subsection{Variance under Step-Wise Safety}
\label{sec:variance}
We first introduce two variance quantities for reward uncertainty under step-wise safety.
\paragraph{Maximum safe-policy return variance} We measure the return variability available to safe policies by the \textit{maximum safe-policy return variance}. Write $\Var(X)$ for the variance of $X$ and
$\Var(q,f):=\sum_xq(x)f(x)^2-(\sum_xq(x)f(x))^2$ for the variance of
$f$ under a distribution $q$. We define:
\begin{equation}
\VarStar:=\sup_{\pi\in\Pi_{\mathrm{safe}}}
\E^\pi\!\left[\sum_{h=1}^H
\left\{\Var(R_h(s_h,a_h))+
\Var\!\left(P_h(\cdot\mid s_h,a_h),V_{h+1}^\pi\right)\right\}
\,\middle|\,s_1\right].
\label{eq:safe-max-return-variance}
\end{equation}
Equation~\eqref{eq:safe-max-return-variance} captures two sources of variance.
The reward term $\Var(R_h(s_h,a_h))$ measures noise in the reward
observation, while
$\Var(P_h(\cdot\mid s_h,a_h),V_{h+1}^{\pi})$
measures the transition-induced variance in continuation value. Their expected sum equals the variance of the total
return under $\pi$, so $\VarStar$ is the largest total-return variance
over safe policies. It is a fixed property of the CMDP instance. %

\paragraph{Trajectory-based total variance} The preceding quantity is uniform over the safe policy class, but it can be conservative for a particular interaction sequence.  A complementary view therefore uses the
conditional variances actually encountered during interaction. For episode $k$, let
$\tau_k$ be the first step at which the selected action lies
outside the true safe set,
with $\tau_k=H+1$ if no such action is selected. We define the \emph{trajectory-based total variance} as:
\begin{equation}
\Var_\Sigma:=\sum_{k=1}^K\sum_{h<\tau_k}
\left\{\Var(R_h(s_h^k,a_h^k))+
\Var\!\left(P_h(\cdot\mid s_h^k,a_h^k),V_{h+1}^\star\right)\right\}.
\label{eq:trajectory-variance}
\end{equation}
This quantity depends on the learner’s realized trajectories and is therefore algorithm dependent. It captures the local variance at state-action pairs visited before the first unsafe action is selected and measures transition
uncertainty using the optimal safe value function. This truncation
separates the interaction within the safe subgraph from
the additional difficulty of identifying actions outside it.

\begin{remark}
The two quantities capture different aspects of reward-learning uncertainty and neither generally subsumes the other.
Similar policy-level and trajectory-level variance characterizations underlie variance-adaptive guarantees in unconstrained RL~\citep{azar2017minimax,zanette2019tighter,zhou2023sharp,zhao2023variance}, and the corresponding quantities can be incomparable~\citep{zhou2023sharp}. The additional challenge under step-wise safety is that the region supporting these variance quantities is itself unknown, which leads to the separate safety identification complexity studied in Section~\ref{subsec:safety-identification}.
\end{remark}

\subsection{Structure of the Safe Subgraph and Transition Variance}
\label{subsec:branching}
Both variance measures introduced above are closely tied to the safe subgraph. We next study how the structure of the safe subgraph affects these quantities. In fact, it acts most directly on transition variance, through how often stochastic transitions can occur along a safe path and how large the value differences across their successors can be. We capture these two effects by the branching depth of the safe subgraph and successor-value spans.

Let $\gamma=(s_1,a_1,\ldots,s_H,a_H,s_{H+1})$ range over directed paths in
$\cG_{\mathrm{safe}}$ from $(s_1,1)$ to step $H+1$. The
\emph{branching depth} is the largest number of stochastic transitions on one
such path, i.e., $B(\cG_{\mathrm{safe}}):=
\max_\gamma\sum_{h=1}^H
\Ind\!\left\{|\supp(P_h(\cdot\mid s_h,a_h))|>1\right\}$. For a nonempty finite set $X$ and a real-valued function $f$ on $X$, write
$\spann_X(f):=\max_{x\in X}f(x)-\min_{x\in X}f(x)$. For a policy $\pi\in\Pi_\mathrm{safe}$, the \emph{successor-value span} $S^\pi$ measures
the largest difference in continuation values among possible successors of
an action selected by $\pi$ at any reachable vertex. Formally, $S^\pi
:=
\max_{\substack{(s,h)\in\cV^\pi,\ h\in[H]}}
\spann_{\supp(P_h(\cdot\mid s,\pi_h(s)))}
\!(V_{h+1}^\pi)$. Over the same safe paths, define the \emph{optimal successor-value span} by
$S^\star:=\max_{\gamma,\ h\in[H]}
\spann_{\supp(P_h(\cdot\mid s_h,a_h))}
\!(V_{h+1}^\star)$.
This quantity measures the largest span of optimal safe continuation values
across successors of actions on safe paths. The following result shows how these spans and the branching depth control transition variance.
\begin{theorem}[Safe-subgraph variance bound]
\label{thm:variance-topology}
Under Assumption~\ref{ass:init}, every $\pi\in\Pi_{\mathrm{safe}}$ satisfies:
\begin{equation}
\E^\pi\!\left[\sum_{h=1}^H
\Var\!\left(P_h(\cdot\mid s_h,a_h),V_{h+1}^\pi\right)
\,\middle|\,s_1\right]
\le \frac14(S^\pi)^2B(\cG_{\mathrm{safe}}).
\label{eq:variance-topology-perpolicy}
\end{equation}
Moreover, every realized interaction sequence satisfies
$\sum_{k=1}^K\sum_{h=1}^{\tau_k-1}
\Var\!\left(P_h(\cdot\mid s_h^k,a_h^k),V_{h+1}^\star\right)
\le \frac K4(S^\star)^2B(\cG_{\mathrm{safe}})$.
\end{theorem}
The theorem bounds each safe policy's transition variance using $S^\pi$
and the trajectory-based transition variance using $S^\star$. Small
branching depth limits how often transition randomness can affect continuation
value, while a small successor-value span limits the contribution of each
stochastic transition.
The full proof is given in Appendix~\ref{app:variance-topology}.

Figure~\ref{fig:topologies} shows how changing only the costs can change the
safe subgraph, thereby removing or retaining stochastic branches and
changing the continuation-value differences across their successors. The following proposition shows that the bound
in Equation~\eqref{eq:variance-topology-perpolicy} is tight. In Appendix~\ref{app:n-gadget-tightness}, we provide the complete analysis and proof.

\begin{figure}[t]
\centering
\includegraphics[width=0.90\linewidth]{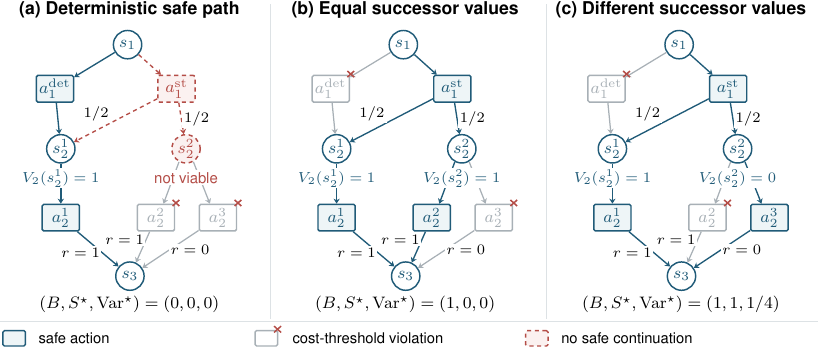}
\caption{
\textbf{The structure of the safe subgraph reshapes transition variance.}
All panels share the same reward function and transition kernel and differ only
in their costs. Changing only the costs removes the stochastic
branch in panel~(a) and retains it in panels~(b)--(c). Its successors have equal
continuation values and zero transition variance in panel~(b), but unequal
values and positive transition variance in panel~(c). With $\tau=1/2$, crossed
actions have cost $1>\tau$, while all other actions have cost $0$. Blue solid elements denote the safe subgraph, gray solid elements denote actions excluded by the cost threshold
and their transitions, and red dashed elements denote
exclusion because a successor has no safe continuation.
}
\label{fig:topologies}
\end{figure}%

\begin{proposition}
\label{prop:gadget-main-tightness}
For every $n\ge1$, there exists a family of CMDPs with $H=2n$ that share the
same transition kernel and deterministic rewards and differ only in the cost function $c$. For every $b\in\{0,\ldots,n\}$, the family contains an instance such
that, for every $\pi\in\Pi_{\mathrm{safe}}$, $B(\cG_{\mathrm{safe}})=b$,
$S^\pi=\Ind\{b>0\}$, $
\E^\pi\!\left[
\sum_{h=1}^H
\Var\!\left(P_h(\cdot\mid s_h,a_h),V_{h+1}^\pi\right)
\,\middle|\,s_1
\right]=\frac{b}{4}$.
Hence, Equation~\eqref{eq:variance-topology-perpolicy} holds with equality and
$\VarStar=b/4$.
\end{proposition}

\subsection{Safety Identification Complexity}
\label{subsec:safety-identification}
The previous results reveal a potential benefit of safety: removing random branches or
successor-value variation can reduce reward uncertainty. However, even a low-variance safe subgraph can still be difficult to identify. %
Intuitively, a rare non-viable branch can
contribute little transition variance
and yet require many executions to discover. This creates a distinct safety identification problem.

For an unsafe action, we first define a local complexity that captures the number of executions needed to identify it as unsafe. The action can be excluded either directly, when its mean cost
exceeds the threshold, or recursively, through a successor state
at which every action is unsafe. %
Formally, for every step $h\in[H]$, state $s\in\cS$, and action
$a\notin\cA_{\mathrm{safe},h}(s)$, we define:
\begin{equation}
\kappa_h(s,a):=
\begin{cases}
\{c_h(s,a)-\tau\}^{-2}, & c_h(s,a)>\tau,\\[1mm]
\displaystyle
\min_{\substack{s'\in\supp(P_h(\cdot\mid s,a))\\
s'\notin\cV_{\mathrm{safe},h+1}}}
\dfrac{\kappa_{h+1}(s')}{P_h(s'\mid s,a)}, & c_h(s,a)\le\tau,
\end{cases}
\qquad
\kappa_h(s):=\sum_{a\notin\cA_{\mathrm{safe},h}(s)}\kappa_h(s,a).
\label{eq:recursive-kappa}
\end{equation}
This definition mirrors the recursive structure of safety. A positive cost gap
yields an inverse-squared term. Otherwise, one non-viable
successor suffices to exclude the parent action, giving the minimum over such
successors, while reaching that successor incurs the factor
$1/P_h(s'\mid s,a)$. At each state $s$, $\kappa_h(s)$ sums complexities of its unsafe actions.

\paragraph{Instance-dependent identification complexity}
The first unsafe action can only be selected after a safe path. Hence its
vertex is reachable under some safe policy. This motivates the \emph{safety identification
complexity}, which we define as follows:
\begin{equation}
\cC(\cM)
:=
\sum_{(s,h)\in\bigcup_{\pi\in\Pi_{\mathrm{safe}}}\cV^\pi}
\kappa_h(s),
\label{eq:C-wit}
\end{equation}
where we set $\kappa_{H+1}(s)=0$ for any state $s$. We write $\cC$ for $\cC(\cM)$ when the instance is clear from context. Thus $\cC$ quantifies the difficulty of ruling out unsafe
actions at the reachable boundary of the safe subgraph. Unlike the
variance quantities, it is a reward-independent instance quantity determined
by $(P,c,\tau,s_1)$.

The safety identification complexity also captures an unavoidable statistical
cost of eliminating unsafe actions. Even on a simple tree-structured family,
any $\delta$-valid elimination rule requires
$\Omega\!\left(\kappa_h(s,a)\log(1/\delta)\right)$ expected executions to
eliminate an unsafe root action, while for pairwise-disjoint trees these
requirements add to
$\Omega\!\left(\cC\log(1/\delta)\right)$ expected executions. Thus, the cost-gap and rare-successor effects in
$\kappa$, together with their accumulation in $\cC$,
capture genuine safety identification difficulty. The construction and proof
are given in Appendix~\ref{app:witness-elimination}.

\section{Safe Variance-Adaptive Exploration}
\label{sec:algorithm}

In this section, we provide our main algorithm, \emph{Safe Variance-Adaptive Exploration} (SVAE).
Section~\ref{sec:topology} shows that the safe subgraph can reduce reward uncertainty, while identifying it remains difficult even when the return variance is small. On the reward side, empirical variance can sharpen optimistic exploration, whereas on the cost side, empirical variance alone cannot reliably rule out unsafe actions. An unsafe action can exhibit zero empirical cost variance until a rare high-cost outcome is observed, as shown in Appendix~\ref{app:silent-cliff}.
SVAE therefore adopts an asymmetric design, combining
variance-independent cost lower confidence bounds for safe-subgraph learning with variance-adaptive bonuses
for reward exploration.
We present SVAE in Algorithm~\ref{alg:svae}.

\begin{algorithm}[t]
\caption{Safe Variance-Adaptive Exploration (SVAE)}
\label{alg:svae}
\begin{algorithmic}[1]
\small
\Require $\cS$, $\cA$, $H$, $K$, $s_1$, $\tau$, and $\delta$
\State Set $\iota\gets\log(32S^2AHK/\delta)$ and initialize all empirical statistics
\For{$k=1,2,\ldots,K$}
    \State Form $n_h^k$, $\widehat r_h^k$, $\widehat c_h^k$, $\widehat P_h^k$, $\widehat\sigma_{R,h}^{2,k}$, and $\widehat\supp_h^k$ for all $(s,a,h)\in\cS\times\cA\times[H]$
    \State Set $\underline c_h^k$ by~\eqref{eq:cost-lcb} for all $(s,a,h)\in\cS\times\cA\times[H]$
    \State $s_1^k\gets s_1$, $\hat\cV_{\mathrm{safe},H+1}^k\gets\cS$, and $V_{H+1}^k(\cdot)\gets0$
    \For{$h=H,H-1,\ldots,1$}
        \State $\hat\cA_{\mathrm{safe},h}^k(s)\gets\{a\in\cA:\underline c_h^k(s,a)\le\tau,\ \widehat\supp_h^k(s,a)\subseteq\hat\cV_{\mathrm{safe},h+1}^k\}$ for all $s\in\cS$
        \State $\hat\cV_{\mathrm{safe},h}^k\gets\{s\in\cS:\hat\cA_{\mathrm{safe},h}^k(s)\ne\varnothing\}$ \hfill {\color{blue}$\triangleright$\,\textit{Safe-subgraph learning}}
        \State Compute $b_h^k(s,a)$ by~\eqref{eq:variance-bonus} for all $(s,a)\in\cS\times\cA$ \hfill {\color{blue}$\triangleright$\,\textit{Variance-adaptive bonus}}
        \State $Q_h^k(s,a)\gets\min\{H-h+1,\widehat r_h^k(s,a)+\widehat P_h^k(\cdot\mid s,a)^\top V_{h+1}^k+b_h^k(s,a)\}$ for all $(s,a)\in\cS\times\cA$
        \For{$s\in\cS$}
            \State $\displaystyle
            \pi_h^k(s)\in
            \begin{cases}
            \argmax_{a\in\hat\cA_{\mathrm{safe},h}^k(s)} Q_h^k(s,a),
            & \hat\cA_{\mathrm{safe},h}^k(s)\ne\varnothing,\\
            \argmax_{a\in\cA} Q_h^k(s,a),
            & \text{otherwise},
            \end{cases}$
            \State $V_h^k(s)\gets Q_h^k(s,\pi_h^k(s))$
        \EndFor
    \EndFor
    \For{$h=1,2,\ldots,H$} \hfill {\color{blue}$\triangleright$\,\textit{Policy execution}}
        \State Play $a_h^k\gets\pi_h^k(s_h^k)$, observe $(R_h^k,C_h^k,s_{h+1}^k)$, and update the empirical statistics
    \EndFor
\EndFor
\end{algorithmic}
\end{algorithm}

\paragraph{Safe-subgraph learning}
\label{subsec:outer}

SVAE maintains an optimistic outer approximation of the unknown safe subgraph. Fix a confidence parameter $\delta\in(0,1)$ and set
$\iota:=\log(32S^2AHK/\delta)$.
Let $N_h^k(s,a)$ and $N_h^k(s,a,s')$
denote, respectively, the numbers of visits to $(s,a)$
and transitions from $(s,a)$ to $s'$ at step $h$
before episode $k$.
Write $n_h^k(s,a):=\max\{N_h^k(s,a),1\}$.
Let $\widehat r_h^k$, $\widehat c_h^k$, and
$\widehat P_h^k$ denote the corresponding empirical
estimates of $r_h$, $c_h$, and $P_h$, respectively, with detailed definitions given in
Appendix~\ref{app:empirical-estimators}.
The observed successor set is
$\widehat\supp_h^k(s,a):=
\{s'\in\cS:N_h^k(s,a,s')>0\}$.

We then compute the optimistic cost estimates for step $h\in [H]$, state $s\in\mathcal{S}$, and action $a\in \cA$ as:
\begin{equation}
\underline c_h^k(s,a)
:=\widehat c_h^k(s,a)-\sqrt{2\iota/n_h^k(s,a)}.
\label{eq:cost-lcb}
\end{equation}
SVAE constructs the candidate action sets
$\hat\cA_{\mathrm{safe},h}^k(s)$ and state sets
$\hat\cV_{\mathrm{safe},h}^k$ by applying the backward
recursion in Equation~\eqref{eq:safe-action}, with optimistic cost
estimates in place of true costs and observed successors
in place of the full successor sets.
The following lemma shows that this recursion retains
all safe actions and states with high probability.
Its proof is given in Appendix~\ref{app:safe-superset}.

\begin{lemma}[Optimism of the candidate sets]
\label{lem:safe-superset}
With probability at least $1-\delta/8$, it holds %
$\cA_{\mathrm{safe},h}(s)
\subseteq
\hat\cA_{\mathrm{safe},h}^k(s)$ and
$\cV_{\mathrm{safe},h}
\subseteq
\hat\cV_{\mathrm{safe},h}^k$ for every episode
$k\in[K]$, step $h\in[H]$, and state $s\in\cS$.
\end{lemma}
\paragraph{Variance-adaptive planning}
\label{subsec:inner}
SVAE then performs optimistic value iteration on the estimated
safe subgraph using variance-adaptive bonuses. For each step $h$ and state-action pair
$(s,a)$, define the empirical reward variance
$\widehat\sigma_{R,h}^{2,k}(s,a):=
\frac{1}{n_h^k(s,a)}
\sum_{\ell<k}
(R_h^\ell-\widehat r_h^k(s,a))^2
\Ind\{(s_h^\ell,a_h^\ell)=(s,a)\}$, where $\mathbf{1}\{\cdot\}$ denotes the indicator function.
For any function $V:\cS\to\mathbb{R}$, write
$\widehat P_h^kV(s,a):=
\sum_{s'\in\cS}\widehat P_h^k(s'\mid s,a)V(s')$.
Let $Q_h^k$ and $V_h^k$ denote the state-action value and state-value
estimates computed in Algorithm~\ref{alg:svae}.
SVAE uses the following empirical-Bernstein exploration bonus:
\begin{equation}
\begin{split}
b_h^k(s,a)
={}&
4\sqrt{
\frac{\Var(\widehat P_h^k(\cdot\mid s,a),V_{h+1}^k)\,\iota}
     {n_h^k(s,a)}
}
+2\sqrt{
\frac{\widehat\sigma_{R,h}^{2,k}(s,a)\,\iota}
     {n_h^k(s,a)}
}+\frac{\{16(H-h)+5\}\iota}{n_h^k(s,a)}.
\end{split}
\label{eq:variance-bonus}
\end{equation}
The first two terms account for transition and reward
uncertainty through the empirical variances of estimated
next-state values and reward observations, respectively.
The last term provides a finite-sample correction even
when the empirical variances vanish.

Using these bonuses, SVAE computes the state-action value
estimates $Q_h^k$ by optimistic Bellman updates in the same backward pass that
constructs the estimated safe subgraph.
At each state $s$, it chooses $\pi_h^k(s)$ to maximize
$Q_h^k(s,\cdot)$ over $\hat\cA_{\mathrm{safe},h}^k(s)$
when this set is nonempty, and over $\cA$ otherwise,
with $V_h^k(s)=Q_h^k(s,\pi_h^k(s))$. SVAE then executes the resulting policy $\pi^k$ for one
episode and uses the observed rewards, costs, and successor
states to update the empirical statistics for the next episode.

\begin{remark}%
Variance-adaptive RL uses empirical next-state value variance, and in some cases reward variance, to sharpen optimism in unconstrained MDPs
\citep{azar2017minimax,zanette2019tighter,zhou2023sharp}.
In SVAE, the continuation values entering the transition-variance term are obtained by planning over candidate safe action sets learned online.
Hence, the bonus adapts to the return uncertainty under the evolving safety-induced graph structure.
\end{remark}%

\section{Learning Guarantees}
\label{sec:main_results}

In this section, we establish an instance-dependent regret bound and cumulative step-wise violation bounds for SVAE. We show how the variance and identification complexity introduced in Section~\ref{sec:topology} govern its learning guarantees.
Throughout, we use $\Omega$ and $\Theta$ to hide universal constants.
The notation $\widetilde{\mathcal O}$ additionally hides polylogarithmic
factors in $S,A,H,K$, and $1/\delta$.

\subsection{Regret and Constraint Violation}
\label{subsec:regret}
\label{subsec:violation}
We first present the cumulative regret attained by Algorithm~\ref{alg:svae} over $K$ episodes.

\begin{theorem}[Regret bound]
\label{thm:regret}
Suppose Assumption~\ref{ass:init} holds.
For any integer $K\ge1$ and confidence parameter $\delta\in(0,1)$,
SVAE (Algorithm~\ref{alg:svae}) attains:
\[
\cR(K)\le\widetilde{\mathcal O}\!\left(
\sqrt{SAH\min\!\left\{\Var_\Sigma,K\VarStar\right\}}
+S\sqrt{AH^3\min\{K,\cC\}}
+S^2AH^2
\right),
\]
with probability at least $1-\delta$.
\end{theorem}

Theorem~\ref{thm:regret} separates the variance-dependent contribution to regret from the cost of safety identification, and gives polylogarithmic regret on fixed instances with $\VarStar=0$. By taking the minimum of $\Var_\Sigma$ and $K\VarStar$, the bound refines the purely instance-dependent guarantee and can yield a smaller variance-dependent term when the learner encounters low-variance trajectories. When $\tau=1$, our setting degenerates to unconstrained RL, and the bound recovers the variance-dependent term of MVP-V~\citep{zhou2023sharp} in the time-inhomogeneous setting.
Theorem~\ref{thm:regret} also yields
$\widetilde{\mathcal O}(S\sqrt{AH^3K})$ regret in the worst case, matching the leading
regret of OMDPD~\citep{zhu2025optimistic}. For uniformly bounded $\cC$
and sufficiently large $K$, the leading regret term becomes
$\widetilde{\mathcal O}(\sqrt{SAH^3K})$, which matches the
corresponding minimax rate
~\citep{zhang2025settling}.
The proof is deferred to Appendix~\ref{app:regret-proof}.

For deterministic rewards, combining Theorem~\ref{thm:regret} with Theorem~\ref{thm:variance-topology} and the universal bound $\VarStar\le H^2/4$ yields the following graph-dependent corollary.

\begin{corollary}[Graph-dependent regret]
\label{cor:deterministic-topology-regret}
Under the conditions of Theorem~\ref{thm:regret},
if the rewards are deterministic, then with probability
at least $1-\delta$, the reward regret of SVAE satisfies
$\cR(K)\le\widetilde{\mathcal O}(
\sqrt{SAHK\min\{H^2,(S^\star)^2B(\cG_{\mathrm{safe}})\}}
+S\sqrt{AH^3\min\{K,\cC\}}
+S^2AH^2)$.
\end{corollary}

Corollary~\ref{cor:deterministic-topology-regret} bounds the variance-dependent regret term through the branching depth of the safe
subgraph and the successor-value span of the optimal safe value function.
In particular, with deterministic rewards, this term vanishes when $S^\star=0$, even if the transitions are stochastic.

\paragraph{Constraint violation}
We then bound the cumulative step-wise constraint violation of SVAE.
In addition to a general bound, we can obtain a %
gap-dependent bound under the following assumption. The gap in the assumption controls how easily costs exceeding the threshold can be
identified. Similar gap-dependent assumptions have been studied
in safe RL and bandits
\citep{%
chen2022strategies,gangrade2024safe}.

\begin{assumption}[Positive cost gap]
\label{ass:margin}
There exists $\Delta_c>0$ such that
$c_h(s,a)-\tau\ge\Delta_c$ for every
$h\in[H]$ and $(s,a)\in\mathcal S\times\mathcal A$
with $c_h(s,a)>\tau$.
\end{assumption}
We now give the cumulative step-wise constraint violation bounds of our algorithm.

\begin{theorem}[Cumulative constraint violation]
\label{thm:violation}
Under Assumption~\ref{ass:init}, with probability at least $1-\delta$, SVAE satisfies $\Viol(K)
\le
\widetilde{\mathcal{O}}(H\sqrt{SAK}+S^2AH^2)$.
If Assumption~\ref{ass:margin} also holds, then on the same event, $\Viol(K)
\le
\widetilde{\mathcal{O}}(SAH/{\Delta_c}+S^2AH^2)$.
\end{theorem}

Theorem~\ref{thm:violation} gives a $\widetilde{\mathcal O}(\sqrt K)$ cumulative constraint violation in general and a violation bound that is polylogarithmic in $K$ for a fixed positive cost gap. For state-only costs, pooling cost observations across actions and
steps gives $\widetilde{\mathcal O}(\sqrt{SKH}+S^2AH^2)$ violation or $\widetilde{\mathcal O}(S/\Delta_c+S^2AH^2)$,
which recovers the rates of SUCBVI~\citep{xiong2023provably}. If every state admits a safe action at every step, the bound
becomes $\widetilde{\mathcal O}(H\sqrt{SAK})$,
matching the bound of LSVI-AE~\citep{wei2024instantaneous} in the tabular setting. Under Assumption~\ref{ass:margin}, if every unsafe action
at a vertex reachable under a safe policy has cost exceeding
$\tau$, we then have $\cC\le SAH/\Delta_c^2$, as proved in
Appendix~\ref{app:immediate-cost-specialization}.  %

\subsection{Regret Lower Bounds}
\label{sec:lower-bound}

We next establish a regret lower bound that captures the
dependence on both reward variance and safety identification complexity.

\begin{theorem}[Lower bound]
\label{thm:instance-lower-bound}
Let $S\ge8$, $A\ge2$, $H\ge64$, and $K\ge SA$ be integers.
Fix $v\in[0,H/128]$ and $\Gamma\in[16(A-1),SAH/64]$ satisfying $\min\{K,\Gamma\}\ge16SA/H$.
For any learning algorithm, there exists a CMDP $\mathcal M$ with $S$ states, $A$ actions, and horizon $H$ satisfying Assumption~\ref{ass:init}, such that $\VarStar(\mathcal M)\in[3v/4,v]$, $\cC(\mathcal M)=\Theta(\Gamma)$, and the algorithm's expected regret after $K$ episodes satisfies $\mathbb E_{\mathcal M}[\mathcal R(K)]=\Omega(\sqrt{SAHKv}+\sqrt{SAH\min\{K,\Gamma\}})$.
Moreover, $\Var_\Sigma\in[3Kv/4,Kv]$ for every realized interaction sequence on $\mathcal M$.
\end{theorem}

Theorem~\ref{thm:instance-lower-bound} demonstrates that both maximum safe-policy return variance and safety identification complexity contribute to the
difficulty of learning.
In particular, zero safe-policy return variance does not
eliminate the regret incurred by identifying safe actions.
The proof uses two families of CMDPs with the prescribed
variance and identification complexity.
One requires distinguishing actions with similar expected
rewards, while the other requires identifying a hidden safe
action.
Taking the larger of the two lower bounds yields the stated
result.
The full proof is provided in Appendix~\ref{app:instance-lower-bound-proof}.

\bibliographystyle{plainnat}
\bibliography{refs}

\clearpage
\appendix

\numberwithin{equation}{section}
\numberwithin{lemma}{section}
\numberwithin{proposition}{section}
\numberwithin{corollary}{section}

\begingroup
\hypersetup{linktoc=page}
\etocsettagdepth{main}{none}
\etocsettagdepth{appendix}{subsection}
\etocdepthtag{appendix}
\renewcommand{\contentsname}{Appendix Contents}
\setcounter{tocdepth}{2}
\tableofcontents
\endgroup
\clearpage

\section{Additional Related Work}
\label{app:additional-related-work}

\paragraph{Constrained policy optimization}
Policy optimization methods approach constrained learning through different
policy updates and penalty constructions.
\citet{achiam2017constrained} use trust-region updates in Constrained Policy
Optimization, while \citet{tessler2019reward} use an adaptive penalty in Reward
Constrained Policy Optimization to guide learning toward constraint satisfaction.
For discounted CMDPs, \citet{ding2025convergence} analyze the convergence and
sample complexity of natural policy gradient primal--dual methods with tabular
and function-approximation parametrizations.
In online CMDPs with adversarial losses and constraints,
\citet{stradi2025policy} provide the first policy optimization primal-dual method to attain best-of-both-worlds results.

\paragraph{Safe exploration and policy evaluation}
Safety during data collection is distinct from feasibility of the policy returned
after exploration. \citet{miryoosefi2022simple} reduce constrained RL and
approachability to reward-free RL oracles without requiring safe exploration.
\citet{huang2023sweet} instead use a known safe baseline to satisfy an expected
cumulative-cost budget in every exploration episode with high probability, while
\citet{xiong2023provably} study reward-free exploration with cumulative step-wise
violation control. For conservative exploration, \citet{li2023near} maintain an
episode-wise performance threshold with high probability using a known baseline
policy or one learned from suitable offline data.
For policy evaluation, \citet{mukherjee2022revar} use reward-variance estimates
to guide sampling in tree-structured MDPs.
\citet{mukherjee2024saver} impose a cumulative safety constraint relative to a
known default policy and, under their tractability assumptions, bound the mean
squared error of SaVeR. In these evaluation problems, the goal is to reduce
estimation error for a fixed target policy, rather than learn which safe policy
to execute.

\paragraph{Safe control}
Control-theoretic and formal methods reason about safety through the future
consequences of an action. \citet{berkenkamp2017safe} combine learned dynamics
with Lyapunov analysis to expand a region of attraction while maintaining
stability. \citet{alshiekh2018shielding} use a shield to restrict or correct actions
according to a temporal-logic specification. Reachability-based methods
characterize states from which a violation can still be avoided.
\citet{ganai2023respo} combine reachability estimation with policy optimization,
while \citet{qadri2025your} show that inverse constraint learning can recover a
dynamics-dependent backward reachable tube rather than the underlying failure set.
These approaches study safety through stability, specifications, or reachability,
rather than through the variance-dependent regret of an online learner.
They provide a related view of why local cost feasibility alone need not ensure
safe continuation.

\paragraph{Variance-dependent reinforcement learning}
Variance has been used both to sharpen worst-case regret bounds and to obtain
guarantees that adapt to a particular instance.
\citet{azar2017minimax} use empirical next-state value variance in finite-horizon
exploration bonuses, while \citet{talebi2018variance} establish variance-aware
regret bounds for undiscounted MDPs.
\citet{zhou2023horizon} study variance-dependent learning in latent MDPs whose
contexts are revealed in hindsight. With known transitions,
\citet{li2026data} obtain data-dependent regret bounds for adversarial losses and
variance-dependent bounds for stochastic losses.
A different trajectory model appears in branching RL, where an episode generates
a tree rather than a single path and the analysis requires a corresponding law
of total variance~\citep{du2022branching}.
That model differs from the stochastic transitions along a single trajectory
considered here. Thus, both the variance quantity and its role in a guarantee
depend on the underlying MDP and observation model.

\paragraph{Instance-dependent sample complexity}
Sample-complexity bounds for CMDPs capture several sources of difficulty beyond
return variance. With generative-model access, \citet{vaswani2022near}
distinguish relaxed and strict feasibility in discounted CMDPs and quantify the
role of the Slater constant in learning a near-optimal policy with no constraint
violation. \citet{wei2026average} study constrained average-reward MDPs under
the same sampling model, with bounds involving bias span, transient time, and a
feasibility margin. \citet{jiang2024achieving} instead identify nonzero basic
variables and supporting constraints that determine an optimal occupancy measure
of the CMDP linear program.
These analyses concern feasibility margins, long-run dynamics, or the variables
and constraints that characterize an optimal occupancy measure.
Our safety identification complexity measures the difficulty of ruling out actions outside the
true safe set using local cost observations and evidence that a successor admits
no safe continuation.

\section{Preliminaries}
\label{app:estimation-optimism}

We first specify the empirical estimates and implementation of SVAE.
We then fix an optimal safe policy, define the confidence events, and
prove the candidate-set inclusions and optimism used in the analysis.

Fix $K\ge1$ and $\delta\in(0,1)$. Recall that
$\iota=\log(32S^2AHK/\delta)$, and set $L_K:=2+\log K$.
Throughout the appendix, $\mathcal{O}(\cdot)$ hides universal
multiplicative constants.
We write $P_hf(s,a):=\sum_{s'\in\cS}P_h(s'\mid s,a)f(s')$.
A state $s$ is viable at step $h$ if $s\in\cV_{\mathrm{safe},h}$.

\subsection{Empirical Estimates}
\label{app:empirical-estimators}

The estimates used in episode $k\in[K]$ of Algorithm~\ref{alg:svae}
use only observations from episodes $\ell<k$.
For $s,s'\in\cS$, $a\in\cA$, and $h\in[H]$, define
\begin{align*}
N_h^k(s,a)
&:=\sum_{\ell<k}\Ind\{(s_h^\ell,a_h^\ell)=(s,a)\},\\
N_h^k(s,a,s')
&:=\sum_{\ell<k}\Ind\{(s_h^\ell,a_h^\ell,s_{h+1}^\ell)=(s,a,s')\}.
\end{align*}
Write $n_h^k(s,a):=\max\{N_h^k(s,a),1\}$.
The empirical reward and cost estimates are
\begin{equation}
\begin{aligned}
\widehat r_h^k(s,a)
&:=\frac{1}{n_h^k(s,a)}\sum_{\ell<k}
R_h^\ell\Ind\{(s_h^\ell,a_h^\ell)=(s,a)\},\\
\widehat c_h^k(s,a)
&:=\frac{1}{n_h^k(s,a)}\sum_{\ell<k}
C_h^\ell\Ind\{(s_h^\ell,a_h^\ell)=(s,a)\}.
\end{aligned}
\label{eq:empirical-reward-cost}
\end{equation}
For $N_h^k(s,a)\ge1$, the empirical transition probabilities are
\begin{equation}
\widehat P_h^k(s'\mid s,a)
:=\frac{N_h^k(s,a,s')}{N_h^k(s,a)}.
\label{eq:empirical-transition}
\end{equation}
At zero count, let $\widehat P_h^k(\cdot\mid s,a)$ be an arbitrary
fixed probability distribution on $\cS$.
The empirical reward variance and observed successor set from
Section~\ref{sec:algorithm} are
\begin{align*}
\widehat\sigma_{R,h}^{2,k}(s,a)
&:=\frac{1}{n_h^k(s,a)}\sum_{\ell<k}
\bigl(R_h^\ell-\widehat r_h^k(s,a)\bigr)^2
\Ind\{(s_h^\ell,a_h^\ell)=(s,a)\},\\
\widehat\supp_h^k(s,a)
&:=\{s'\in\cS:N_h^k(s,a,s')>0\}.
\end{align*}
Thus, at zero count, $\widehat r_h^k=\widehat c_h^k=
\widehat\sigma_{R,h}^{2,k}=0$ and the observed successor set is empty.
The initialization of $\widehat P_h^k$ does not add elements to
the observed successor set.
The linear term in~\eqref{eq:variance-bonus} then gives
$Q_h^k(s,a)=H-h+1$.

Before episode $k$, SVAE computes its candidate sets, bonuses, values,
and policy from these estimates. Fix a deterministic rule for breaking
ties in each maximization. At step $h$ of the dynamic program,
$\hat\cV_{\mathrm{safe},h+1}^k$ and $V_{h+1}^k$ are already available,
so the candidate sets and Bellman updates can be computed in a single
backward pass. The candidate sets, $\pi^k$, $Q_h^k$, and $V_h^k$ remain
fixed during execution. New observations update the statistics for
episode $k+1$.

The estimates can be maintained using visit counts, transition counts,
and sums of rewards, costs, and squared rewards. In particular,
\[
\widehat\sigma_{R,h}^{2,k}(s,a)
=\frac{1}{n_h^k(s,a)}\sum_{\ell<k}(R_h^\ell)^2
\Ind\{(s_h^\ell,a_h^\ell)=(s,a)\}
-\bigl(\widehat r_h^k(s,a)\bigr)^2.
\]
Each observation updates these stored quantities in constant time.
With dense transition arrays, checking observed successors and
computing transition means and variances take $\mathcal O(S)$
operations for each $(s,a,h)$, while all action maximizations take
$\mathcal O(SAH)$ operations. Each episode therefore requires
$\mathcal O(S^2AH)$ arithmetic operations, including its $H$
observation updates. The total is $\mathcal O(KS^2AH)$ operations
with $\mathcal O(S^2AH)$ memory.

Let $\cF_h^k$ be the $\sigma$-field generated by the history just before
the reward, cost, and next state at step $h$ of episode $k$ are observed.
Thus $s_h^k$ and $a_h^k$ are $\cF_h^k$-measurable, while
$\pi^k$, $Q_h^k$, and $V_h^k$ are fixed given $\cF_1^k$. Conditional on $\cF_h^k$,
$R_h^k$ and $C_h^k$ have means $r_h(s_h^k,a_h^k)$ and
$c_h(s_h^k,a_h^k)$, respectively, and the next state has distribution
$P_h(\cdot\mid s_h^k,a_h^k)$.
We order these histories by episode and
step, and let $\cF_{H+1}^k$ include the final observations of episode $k$.

\subsection{Optimal Safe Policy}
\label{app:proof-objects}

The Bellman comparisons below use a policy that is optimal at every
viable state. The next lemma supplies such a policy. We take it as
$\pi^\star$ in Section~\ref{sec:formulation} and retain this choice in
$V^\star$, $\Var_\Sigma$,
and $S^\star$.

\begin{lemma}
\label{lem:safe-comparator-publication}
Under Assumption~\ref{ass:init}, there exists a deterministic policy
$\pi^\star\in\Pi_{\mathrm{safe}}$
that selects a safe action at every viable state. With
$V_{H+1}^\star\equiv0$, it satisfies
\begin{equation}
V_h^\star(s)
=
\max_{a\in\cA_{\mathrm{safe},h}(s)}
\left\{r_h(s,a)+P_hV_{h+1}^\star(s,a)\right\}
\label{eq:safe-bellman-publication}
\end{equation}
for every $h\in[H]$ and $s\in\cV_{\mathrm{safe},h}$.
For every policy $\pi\in\Pi$ that selects safe actions at all viable
states, $V_h^\pi(s)\le V_h^\star(s)$ for all $h\in[H]$ and
$s\in\cV_{\mathrm{safe},h}$.
In particular, $\pi^\star$ maximizes $V_1^\pi(s_1)$ over
$\pi\in\Pi_{\mathrm{safe}}$.
\end{lemma}

\begin{proof}
Set $V_{H+1}^\star=0$ and define $V_h^\star$ recursively by the
right-hand side of~\eqref{eq:safe-bellman-publication}.  Because the
maximizing set is finite and nonempty at every viable state, choose one
maximizer as $\pi_h^\star(s)$, and define the policy arbitrarily at states outside
$\cV_{\mathrm{safe},h}$. Every selected safe action has all successors in
$\cV_{\mathrm{safe},h+1}$. Thus the recursion determines the value of
$\pi^\star$ at every viable state, independently of its actions elsewhere.
Assumption~\ref{ass:init} therefore gives
$\pi^\star\in\Pi_{\mathrm{safe}}$.

Let $\pi$ select a safe action at every viable state.  If
$V_{h+1}^\pi\le V_{h+1}^\star$ on $\cV_{\mathrm{safe},h+1}$, then
\begin{align*}
V_h^\pi(s)
&=r_h(s,\pi_h(s))+P_hV_{h+1}^\pi(s,\pi_h(s))\\
&\le r_h(s,\pi_h(s))+P_hV_{h+1}^\star(s,\pi_h(s))
\le V_h^\star(s).
\end{align*}
Backward induction proves the comparison. For any $\pi\in\Pi_{\mathrm{safe}}$,
choose an arbitrary safe action at each viable state that $\pi$ does not
reach from $s_1$. This completes $\pi$ to a policy selecting safe actions
at all viable states without changing its trajectory distribution or
value from $s_1$. It follows that $V_1^\pi(s_1)\le V_1^\star(s_1)$ for every
$\pi\in\Pi_{\mathrm{safe}}$.
\end{proof}

\subsection{Confidence Events}
\label{app:uniform-events}

The cost and support bounds below preserve the safe action sets. The
reward and transition bounds control optimism and Bellman residuals.
All intersections below range over $s,s'\in\cS$, $a\in\cA$,
$h\in[H]$, and $k\in[K]$. Define
\begin{align*}
\cE_c
&:=\bigcap_{s,a,h,k}
\left\{
|\widehat c_h^k(s,a)-c_h(s,a)|
\le\sqrt{\frac{2\iota}{n_h^k(s,a)}}
\right\},\\
\cE_P
&:=\bigcap_{s,s',a,h,k}
\left\{
\begin{aligned}
&\Ind\{N_h^k(s,a)\ge1\}
|\widehat P_h^k(s'\mid s,a)-P_h(s'\mid s,a)|\\
&\qquad\le\sqrt{\frac{2P_h(s'\mid s,a)\iota}{n_h^k(s,a)}}
+\frac{\Ind\{P_h(s'\mid s,a)>0\}\iota}{n_h^k(s,a)}
\end{aligned}
\right\},\\
\cE_{\mathrm{supp}}
&:=\bigcap_{s,a,h,k}
\left\{
\widehat\supp_h^k(s,a)\subseteq\supp(P_h(\cdot\mid s,a))
\right\}.
\end{align*}
For $N:=N_h^k(s,a)\ge2$, consider the three empirical Bernstein bounds
\begin{align}
|\widehat r_h^k(s,a)-r_h(s,a)|
&\le2\sqrt{\frac{\widehat\sigma_{R,h}^{2,k}(s,a)\iota}{N}}
+\frac{5\iota}{N},
\label{eq:reward-eb-publication}\\
|(\widehat P_h^k-P_h)V_{h+1}^\star(s,a)|
&\le2\sqrt{\frac{\Var(\widehat P_h^k(\cdot\mid s,a),V_{h+1}^\star)\iota}{N}}
+\frac{5(H-h)\iota}{N},
\label{eq:value-eb-publication}\\
\widehat\sigma_{R,h}^{2,k}(s,a)
&\le2\Var(R_h(s,a))+\frac{2\iota}{N}.
\label{eq:reward-var-comparison-publication}
\end{align}
Define $\cE_{\mathrm{EB}}$ as the event that
\eqref{eq:reward-eb-publication}--\eqref{eq:reward-var-comparison-publication}
hold for every $s,a,h,k$ with $N_h^k(s,a)\ge2$.
Set $\cE_{\mathrm{cs}}:=\cE_c\cap\cE_{\mathrm{supp}}$ and
$\Est:=\cE_{\mathrm{cs}}\cap\cE_P\cap\cE_{\mathrm{EB}}$.

\begin{lemma}
\label{lem:estimation-event-publication}
The events satisfy $\Prob(\cE_{\mathrm{cs}})\ge1-\delta/8$ and
$\Prob(\Est)\ge1-9\delta/16$.
\end{lemma}

\begin{proof}
For each $(s,a,h)$, generate an infinite sequence of independent copies
of its reward, cost, and next-state observations. Revealing the next
unused copy at each visit gives the same interaction distribution as
the observation model. We apply concentration bounds to the first $n$
copies for every deterministic $n\in[K]$, then take a union bound over
these counts before substituting the realized count $N_h^k(s,a)$.

For the cost and transition estimates, fix $(s,a,h)$ and $s'\in\cS$.
Write $C_i$ for the $i$th cost sample and let $Z_i$ indicate that the
$i$th next state is $s'$. With $p=P_h(s'\mid s,a)$,
Hoeffding's and Bernstein's inequalities give, for each $n\in[K]$,
\begin{align*}
\Prob\!\left(\left|\frac1n\sum_{i=1}^n C_i-c_h(s,a)\right|
>\sqrt{\frac{2\iota}{n}}\right)&\le2e^{-\iota},\\
\Prob\!\left(\left|\frac1n\sum_{i=1}^n Z_i-p\right|
>\sqrt{\frac{2p\iota}{n}}+\frac{2\iota}{3n}\right)&\le2e^{-\iota}.
\end{align*}
If $p=0$, all $Z_i$ are zero almost surely. At zero count, the cost
bound holds because $c_h(s,a)\le1\le\sqrt{2\iota}$.
Every observed successor belongs to the true support almost surely,
so $\Prob(\cE_{\mathrm{supp}})=1$.

For the reward and value estimates, consider independent, identically
distributed $X_1,\ldots,X_n\in[0,B]$. Let $\overline X_n$ and
$\widehat v_n=n^{-1}\sum_i(X_i-\overline X_n)^2$ be the sample mean
and variance. Applying \citet[Theorem~4]{maurer2009empirical} to
$X_i/B$ and $1-X_i/B$ gives, for $B>0$ and $n\ge2$,
\[
|\overline X_n-\E[X_1]|
\le\sqrt{\frac{2\widehat v_n\iota}{n-1}}
+\frac{7B\iota}{3(n-1)}
\le2\sqrt{\frac{\widehat v_n\iota}{n}}+\frac{5B\iota}{n}
\]
with failure probability at most $4e^{-\iota}$. The case $B=0$ is
immediate. Apply this bound to the reward samples with $B=1$ and to
$V_{h+1}^\star$ evaluated at the sampled next states with $B=H-h$. The latter
function is fixed by Lemma~\ref{lem:safe-comparator-publication}.
For the reward variance, write $R_i$ for the samples. The sample
variance is at most the average squared deviation from $r_h(s,a)$.
Bernstein's inequality therefore gives, with failure probability at most
$e^{-\iota}$,
\begin{align*}
\widehat v_n
&\le\frac1n\sum_{i=1}^n(R_i-r_h(s,a))^2\\
&\le\Var(R_h(s,a))
+\sqrt{\frac{2\Var(R_h(s,a))\iota}{n}}+\frac{2\iota}{3n}
\le2\Var(R_h(s,a))+\frac{2\iota}{n}.
\end{align*}
Here $(R_i-r_h(s,a))^2\in[0,1]$ has mean $\Var(R_h(s,a))$ and variance
at most that mean. For each fixed triple and sample count, the three
Bernstein bounds thus have total failure probability at most
$(4+4+1)e^{-\iota}=9e^{-\iota}$.

Taking a union bound over the deterministic sample counts, states,
actions, and steps gives
\begin{align*}
\Prob(\cE_{\mathrm{cs}}^c)
&\le2SAHKe^{-\iota}\le\delta/8,\\
\Prob(\Est^c)
&\le\bigl[2(S+1)+9\bigr]SAHKe^{-\iota}
\le13\delta/32\le9\delta/16.\qedhere
\end{align*}
\end{proof}

\subsection{Proof of Lemma~\ref{lem:safe-superset}}
\label{app:candidate-invariant}
\label{app:safe-superset}

\begin{proof}
We establish both inclusions on $\cE_{\mathrm{cs}}$, including the
state-set inclusion at step $H+1$.
On $\cE_{\mathrm{cs}}$, the cost bound gives
$\underline c_h^k(s,a)\le c_h(s,a)$ for every $s,a,h,k$.
Fix $k$ and argue backward from
$\cV_{\mathrm{safe},H+1}=\hat\cV_{\mathrm{safe},H+1}^k=\cS$.
If $a\in\cA_{\mathrm{safe},h}(s)$, then $\underline c_h^k(s,a)\le\tau$
and the induction hypothesis gives
\[
\widehat\supp_h^k(s,a)
\subseteq\supp(P_h(\cdot\mid s,a))
\subseteq\cV_{\mathrm{safe},h+1}
\subseteq\hat\cV_{\mathrm{safe},h+1}^k.
\]
Thus $a\in\hat\cA_{\mathrm{safe},h}^k(s)$. A nonempty true safe
action set therefore implies a nonempty candidate set, completing the
induction.
By Lemma~\ref{lem:estimation-event-publication},
$\Prob(\cE_{\mathrm{cs}})\ge1-\delta/8$, which proves the claim.
\end{proof}

\begin{samepage}
\subsection{Optimism}
\label{app:optimism}

The following consequence of \citet[Lemma~22]{zhou2023sharp} handles
the dependence of the transition bonus on the optimistic value vector.
For real vectors $p,z$ on the same finite index set, write
$pz:=\sum_i p_i z_i$.

\begin{lemma}
\label{lem:monotone-envelope-publication}
Let $p$ be a distribution on $[S]$, $B\ge0$, $x>0$, and $v,w\in[0,B]^S$ with
$v\le w$ on $\supp(p)$. Then
\begin{equation}
pw+\max\{4\sqrt{\Var(p,w)x},16Bx\}
\ge pv+2\sqrt{\Var(p,v)x}+8Bx.
\label{eq:monotone-envelope-publication}
\end{equation}
\end{lemma}
\end{samepage}

\begin{proof}
For $B>0$, apply \citet[Lemma~22]{zhou2023sharp} on $\supp(p)$ with
$c_1=4$, $c_2=16$, and $\iota/n=x$. Its monotonicity property applies
because $c_1^2=c_2$. Combining it with
$\max\{4a,16b\}\ge2a+8b$ for $a,b\ge0$ proves the claim. For $B=0$,
both sides are zero.
\end{proof}

Combining candidate-set preservation with the monotone backup gives
optimism on the viable states.

\begin{lemma}[Optimism]
\label{lem:optimism-publication}
On $\Est$, $V_h^k(s)\ge V_h^\star(s)$ for every $k\in[K]$,
$h\in[H+1]$, and $s\in\cV_{\mathrm{safe},h}$.
\end{lemma}

\begin{proof}
Fix $k$ and use backward induction from $h=H+1$. For a viable state
$s$, let $a^\star=\pi_h^\star(s)$. By
the inclusions proved in Appendix~\ref{app:safe-superset}, $a^\star$ is a
candidate. Write $N:=N_h^k(s,a^\star)$ and suppress
$(s,a^\star,h,k)$ from the remaining empirical quantities.
If $N\le1$, the linear bonus gives
$Q_h^k(s,a^\star)=H-h+1\ge V_h^\star(s)$.

If $N\ge2$, set $x=\iota/N$ and $q=\widehat P_h^k(\cdot\mid s,a^\star)$.
Since $\supp(q)\subseteq\cV_{\mathrm{safe},h+1}$, the induction hypothesis
holds there. Lemma~\ref{lem:monotone-envelope-publication} and
\eqref{eq:value-eb-publication} yield
\begin{align*}
qV_{h+1}^k+4\sqrt{\Var(q,V_{h+1}^k)x}+16(H-h)x
&\ge qV_{h+1}^\star+2\sqrt{\Var(q,V_{h+1}^\star)x}+8(H-h)x\\
&\ge P_hV_{h+1}^\star.
\end{align*}
Also, \eqref{eq:reward-eb-publication} gives
$\widehat r_h^k+2\sqrt{\widehat\sigma_{R,h}^{2,k}x}+5x\ge r_h$.
The unclipped backup is therefore at least
$r_h+P_hV_{h+1}^\star=V_h^\star(s)$. Clipping at $H-h+1$ and maximizing
over candidates preserve this inequality.
\end{proof}

\section{Regret Analysis}
\label{app:reward-regret}
\label{app:regret-proof}

We prove Theorem~\ref{thm:regret} under Assumption~\ref{ass:init}.
The analysis separates each episode at its first action outside the
true safe set. Before that action, an auxiliary safe policy relates the
Bellman residuals to the safe-policy variance. We bound the number of
episodes containing such an action in terms of $\cC$ and use
a general residual bound for their remaining steps.

\subsection{Martingale Bounds}
\label{app:martingale-inequalities}

We use the following consequence of Freedman's inequality
\citep{freedman1975tail,howard2020timeuniform}.

\begin{lemma}[Martingale concentration]
\label{lem:freedman-publication}
Let $T\ge1$ be an integer, let $(\mathcal H_t)_{t=0}^T$ be a
filtration, and let $(\xi_t)_{t=1}^T$ be martingale differences with
respect to this filtration. Suppose $|\xi_t|\le b$ almost surely for
every $t\in[T]$, where $b\ge0$ is fixed. Then, for every
$\eta\in(0,1)$, with probability at least $1-\eta$,
\begin{equation*}
\sum_{t=1}^T\xi_t
\le
\mathcal{O}\!\left(
\sqrt{\left(\sum_{t=1}^T\E[\xi_t^2\mid\mathcal H_{t-1}]\right)\log\frac{T+1}{\eta}}
+b\log\frac{T+1}{\eta}
\right).
\end{equation*}
For the following consequences, let $B\ge0$ be fixed and, for each
$t\in[T]$, let $Y_t$ be $\mathcal H_t$-measurable with
$0\le Y_t\le B$ almost surely. Set
$\mu_t:=\E[Y_t\mid\mathcal H_{t-1}]$. For every $\eta\in(0,1)$,
each of the following inequalities holds with probability at least $1-\eta$:
\begin{align}
\sum_{t=1}^TY_t
&\le2\sum_{t=1}^T\mu_t
+\mathcal{O}\!\left(B\log\frac{T+1}{\eta}\right),
\label{eq:realized-upper-publication}\\
\sum_{t=1}^T\mu_t
&\le2\sum_{t=1}^TY_t
+\mathcal{O}\!\left(B\log\frac{T+1}{\eta}\right).
\label{eq:predictable-upper-publication}
\end{align}
\end{lemma}

\begin{proof}
The case $b=0$ is immediate. Otherwise, apply Freedman's inequality
at variance thresholds $b^2,2b^2,\ldots,2^{\lceil\log_2 T\rceil}b^2$,
with failure probability $\eta/(T+1)$ at each threshold. The conditional
variance sum is at most $Tb^2$, so these at most $T+1$ thresholds cover
all its values. A union bound yields the first inequality, with the
lowest interval absorbed by the term $b\log((T+1)/\eta)$. For the
first consequence, apply it to $\xi_t=Y_t-\mu_t$ and use
$\E[\xi_t^2\mid\mathcal H_{t-1}]
\le\E[Y_t^2\mid\mathcal H_{t-1}]\le B\mu_t$.
Young's inequality absorbs the resulting term
$\mathcal{O}(\sqrt{B(\sum_t\mu_t)\log((T+1)/\eta)})$.  Applying the same argument
to $\xi_t=\mu_t-Y_t$ proves the second consequence.
\end{proof}

For the applications below, $t\le KH$ and $\eta=\delta/16$ imply
$\log((t+1)/\eta)\le\log(32KH/\delta)\le2\iota$.

\subsection{Safety Identification}
\label{app:recursive-certificate}
\label{app:first-departure}

We use $\kappa_h(s,a)$ and $\cC$ from
Section~\ref{subsec:safety-identification} to bound the number of
selections of unsafe actions in the candidate sets and, in turn, the number of
episodes containing an unsafe action. The recursion for $\kappa_h$
ends at an action with a positive cost gap. Otherwise, it selects a
non-viable successor and sums the sample requirements for ruling out all
actions there.
The next lemma expands this recursion into paths ending at positive-gap
actions. Each leaf contributes its sample requirement
$\Delta_\ell^{-2}$ divided by its path probability $q_\ell$.

\begin{lemma}[Witness-tree representation]
\label{lem:witness-tree-publication}
Fix $h\in[H]$, $s\in\cS$, and $a\notin\cA_{\mathrm{safe},h}(s)$.
There is a finite witness tree rooted at $(s,a,h)$ with a nonempty
leaf set $\mathcal L$. For each leaf $\ell$, let $q_\ell\in(0,1]$
be the product of the transition probabilities on its path and
$\Delta_\ell\in(0,1]$ the cost gap of its terminal action. Then
\begin{equation}
\kappa_h(s,a)
=
\sum_{\ell\in\mathcal L}
\frac{1}{q_\ell\Delta_\ell^2}.
\label{eq:witness-tree-publication}
\end{equation}

For SVAE, each execution of $(s,a,h)$ as a candidate action is assigned
to the leaf determined by its episode policy, before the transitions
on that path are observed. Across the $K$ episodes, let $N_\ell$ count the executions assigned
to leaf $\ell$ for which every transition on that path occurs.
On $\cE_c$, the following bound holds for all $\ell\in\mathcal L$:
\begin{equation}
N_\ell
\le
2+\frac{8\iota}{\Delta_\ell^2}
\le
\frac{10\iota}{\Delta_\ell^2}.
\label{eq:witness-leaf-count-publication}
\end{equation}
\end{lemma}

\begin{proof}
We prove~\eqref{eq:witness-tree-publication} by backward induction.
For the direct-cost branch, take one leaf with
$q_\ell=1$ and $\Delta_\ell=c_h(s,a)-\tau$. Then $\kappa_h(s,a)=\frac{1}{q_\ell\Delta_\ell^2}$.

Suppose now that the recursive branch applies. The minimizing set in
\eqref{eq:recursive-kappa} is nonempty because
$a\notin\cA_{\mathrm{safe},h}(s)$ and $c_h(s,a)\le\tau$.
Choose a minimizing successor $y$. Since
$y\notin\cV_{\mathrm{safe},h+1}$,
$\cA_{\mathrm{safe},h+1}(y)=\varnothing$. By the induction hypothesis,
each $\kappa_{h+1}(y,b)$, $b\in\cA$, has a finite witness
representation. Taking the union of these representations and
multiplying every path probability by $P_h(y\mid s,a)$ gives
\[
\kappa_h(s,a)
=
\frac{\kappa_{h+1}(y)}{P_h(y\mid s,a)}
=
\sum_{\ell\in\mathcal L}
\frac{1}{q_\ell\Delta_\ell^2}.
\]
The construction terminates after at most $H-h+1$ layers. Each internal
action has the selected successor $y$, whose children are indexed by
all actions at $y$. A fixed episode policy selects the child
$\pi_{h+1}^k(y)$ at each such successor and hence determines a unique
leaf before any transition on the path is observed. The product of
these transition probabilities is $q_\ell$.

Fix $\ell\in\mathcal L$ and write $c=\tau+\Delta_\ell$ for the mean
cost of its terminal action. After $n\ge1$ visits, its lower confidence
bound $\underline c$ satisfies
$\underline c\ge c-2\sqrt{2\iota/n}$ on $\cE_c$.
Thus the terminal action is permanently removed by the cost test once
$n>8\iota/\Delta_\ell^2$. Before it is removed, it can be executed at
most $1+8\iota/\Delta_\ell^2$ times.

It remains to account for paths whose transitions have not all been
observed. After the terminal action is permanently rejected, a selected path to
it can be completed at most once more. The first such completion records
every transition on the path. In any later episode selecting the same
leaf, suppose the root action is a candidate. Its recorded successor
must then have a nonempty candidate set, so the algorithm selects a
candidate action there. Repeating this argument along the recorded path
would make the permanently rejected terminal action a candidate, a
contradiction. Recorded supports only grow, and the terminal cost bound holds at all
later sample counts. Adding this possible completion to the visits
before removal proves
\eqref{eq:witness-leaf-count-publication}.
\end{proof}

We now convert completed leaf paths into a bound on root executions.
Weighting a completed path by $1/q_\ell$ makes its conditional mean
equal to the indicator that the root is executed as a candidate.

\begin{lemma}
\label{lem:false-candidate-control}
For SVAE, there exists an event $\Cert$ with
$\Prob(\Cert)\ge1-\delta/16$ such that, on $\Est\cap\Cert$,
the following bound holds simultaneously for all $h\in[H]$,
$s\in\cS$, and $a\notin\cA_{\mathrm{safe},h}(s)$:
\begin{equation}
\sum_{k=1}^{K}
\Ind\!\left\{
 s_h^k=s,\ a_h^k=a,\
 a\in\hat\cA_{\mathrm{safe},h}^k(s)
\right\}
\le \mathcal{O}(\iota\kappa_h(s,a)).
\label{eq:recursive-action-count-revised}
\end{equation}
\end{lemma}

\begin{proof}
Fix $(s,a,h)$ and the representation in
Lemma~\ref{lem:witness-tree-publication}. Let $I_k$ denote the indicator
in~\eqref{eq:recursive-action-count-revised}. If $I_k=1$, let $\ell_k$ be the leaf determined by $\pi^k$, and let
$Z_k$ indicate that all transitions along its path occur. Set
$X_k=Z_k/q_{\ell_k}$ when $I_k=1$ and $X_k=0$ otherwise. Since $\pi^k$ is $\cF_1^k$-measurable,
$\ell_k$ and $q_{\ell_k}$ are $\cF_h^k$-measurable. The Markov
property and~\eqref{eq:witness-tree-publication} give $\E[X_k\mid\cF_h^k]=I_k$, and $0\le X_k\le\kappa_h(s,a)$,
where the second bound also uses $\Delta_{\ell_k}\le1$.
On $\cE_c$, \eqref{eq:witness-leaf-count-publication} gives
\begin{equation*}
\sum_{k=1}^{K}X_k
=
\sum_{\ell\in\mathcal L}\frac{N_\ell}{q_\ell}
\le
10\iota
\sum_{\ell\in\mathcal L}\frac{1}{q_\ell\Delta_\ell^2}
=
10\iota\kappa_h(s,a).
\end{equation*}
Apply~\eqref{eq:predictable-upper-publication} in episode order,
conditioning the $k$th increment on $\cF_h^k$. The variable $X_k$ is known by the end of episode $k$, before
$\cF_h^{k+1}$ for $k<K$. For the final increment, use the terminal
history $\cF_{H+1}^K$. Take $B=\kappa_h(s,a)$ and failure
probability $\delta/(16SAH)$. Let $\Cert$ be the intersection of the
events supplied by these applications over all unsafe $(s,a,h)$.
A union bound gives $\Prob(\Cert)\ge1-\delta/16$. On $\Cert$,
\[
\sum_{k=1}^{K}I_k
\le
2\sum_{k=1}^{K}X_k
+\mathcal{O}\!\left(\kappa_h(s,a)\log\frac{16SAH(K+1)}{\delta}\right)
\]
for every such triple. Since $\log(16SAH(K+1)/\delta)\le2\iota$,
the bound on $\sum_kX_k$ above yields, on $\Est\cap\Cert$,
$\sum_k I_k\le \mathcal{O}(\iota\kappa_h(s,a))$, as required.
\end{proof}

We now apply this count to the first unsafe action in each episode.
For episode $k$, define
\begin{equation*}
\tau_k
:=
\min\left\{h\in[H]:
a_h^k\notin\cA_{\mathrm{safe},h}(s_h^k)\right\},
\qquad
\min\varnothing:=H+1.
\end{equation*}
The event $\{\tau_k\le h\}$ belongs to $\cF_h^k$, so $\tau_k$ is a
stopping time for the within-episode history. All actions before $\tau_k$ are safe, so the state at step $\tau_k$
is reachable under a safe policy. The bound on unsafe candidate
selections therefore controls the number of episodes with $\tau_k\le H$.

\begingroup
\SVAEtheoremheadingtarget
\begin{proposition}
\label{prop:first-departure-publication}
On $\Est\cap\Cert$,
\begin{equation}
\sum_{k=1}^{K}\Ind\{\tau_k\le H\}
\le
\mathcal{O}\!\left(\iota\min\!\left\{K,\cC\right\}\right).
\label{eq:departure-realized-publication}
\end{equation}
\end{proposition}
\endgroup

\begin{proof}
If $\tau_k\le H$, the preceding safe actions reach
$(s_{\tau_k}^k,\tau_k)$ with positive probability. Complete those
actions with arbitrary safe actions at other viable states. The
resulting policy is safe, so this vertex is included in the definition
of $\cC$. Its true safe action set is nonempty and,
by the inclusions proved in Appendix~\ref{app:safe-superset}, so is its candidate
set. Thus $a_{\tau_k}^k$ is an unsafe candidate action.

Apply Lemma~\ref{lem:false-candidate-control} to each unsafe action at
a vertex reachable under a safe policy and sum the bounds. Each
episode with $\tau_k\le H$ contributes one such execution, giving
\begin{equation}
\sum_{k=1}^K\Ind\{\tau_k\le H\}
\le\mathcal{O}\!\left(
\iota\sum_{(s,h)\in\bigcup_{\pi\in\Pi_{\mathrm{safe}}}\cV^\pi}
\kappa_h(s)\right)
=\mathcal{O}(\iota \cC).
\label{eq:departure-charge-chain}
\end{equation}
The count is also at most $K$. Combining these bounds and using
$\iota\ge1$ proves the claim.
\end{proof}

\subsection{Bellman Residual Bounds}
\label{app:bellman-envelope}

We first control the transition estimation error, then derive the two
Bellman residual bounds used before and after the first unsafe action.
The next lemma converts coordinatewise transition bounds into mean
and variance bounds for every bounded value vector, including the
data-dependent value estimates used below.

\begin{lemma}
\label{lem:transition-interface-publication}
Let $B\ge0$, $x>0$, and let $p,q$ be distributions on $[S]$ satisfying
\begin{equation}
|q_i-p_i|
\le\sqrt{2p_ix}+\Ind\{p_i>0\}x
\qquad(i\in[S]).
\label{eq:coordinatewise-assumption-publication}
\end{equation}
Then, for every $z\in[0,B]^S$,
\begin{align}
|(q-p)z|
&\le2\sqrt{S\Var(p,z)x}+3SBx,
\label{eq:functional-mean-publication}\\
\Var(q,z)
&\le2\Var(p,z)+4SB^2x.
\label{eq:functional-variance-publication}
\end{align}
\end{lemma}

\begin{proof}
Let $d_i=z_i-pz$. Since $(q-p)\mathbf1=0$ and $|d_i|\le B$,
\begin{align*}
|(q-p)z|
&\le\sqrt{2x}\sum_i\sqrt{p_i}|d_i|+x\sum_{i:p_i>0}|d_i|\\
&\le\sqrt{2S\Var(p,z)x}+SBx,
\end{align*}
which implies~\eqref{eq:functional-mean-publication}. Also,
\begin{align*}
\Var(q,z)
&\le\sum_{i=1}^S q_i d_i^2\\
&\le\Var(p,z)+\sqrt{2x}\sum_i\sqrt{p_i}d_i^2
+x\sum_{i:p_i>0}d_i^2\\
&\le\Var(p,z)+B\sqrt{2S\Var(p,z)x}+SB^2x
\le2\Var(p,z)+2SB^2x.
\end{align*}
The last two inequalities use Cauchy--Schwarz and Young's inequality.
This proves~\eqref{eq:functional-variance-publication}.
\end{proof}

Before the first unsafe action, we bound the Bellman residual using
$V^\star$. For the remaining steps, we use a bound that applies to every
action. For $h\in[H+1]$, define
$D_h^k(s):=V_h^k(s)-V_h^\star(s)$ for $s\in\cV_{\mathrm{safe},h}$
and $D_h^k(s):=0$ otherwise. On $\Est$,
Lemma~\ref{lem:optimism-publication} gives $0\le D_h^k\le H-h+1$.

Fix a sufficiently large universal constant $C_0$. Define
\begin{align}
\beta_h^k(s,a)
:=&C_0\Bigg[
\sqrt{\frac{\{\Var(R_h(s,a))+
\Var(P_h(\cdot\mid s,a),V_{h+1}^\star)\}\iota}
{n_h^k(s,a)}}
\notag\\
&\quad+\sqrt{\frac{S\Var(P_h(\cdot\mid s,a),D_{h+1}^k)\iota}
{n_h^k(s,a)}}
+\frac{SH\iota}{n_h^k(s,a)}
\Bigg],
\label{eq:beta-publication}
\end{align}
and
\begin{equation}
\widetilde\beta_h^k(s,a)
:=C_0\left[
\sqrt{\frac{\{\Var(R_h(s,a))+
S\Var(P_h(\cdot\mid s,a),V_{h+1}^k)\}\iota}
{n_h^k(s,a)}}
+\frac{SH\iota}{n_h^k(s,a)}
\right].
\label{eq:all-action-beta-publication}
\end{equation}
The same $C_0$ is used in both definitions.
The three terms in $\beta_h^k$ account for the reward and optimal-value
variances, the next-state variance of $D_{h+1}^k$, and the linear
concentration terms. The bound $\widetilde\beta_h^k$ uses
$V_{h+1}^k$ directly.

\begin{lemma}[Bellman residual bounds]
\label{lem:bellman-envelope-publication}
On $\Est$, for every $k\in[K]$, $h\in[H]$, and $s\in\cS$,
the following two bounds hold. For every
$a\in\cA_{\mathrm{safe},h}(s)$,
\begin{equation}
Q_h^k(s,a)-r_h(s,a)-P_hV_{h+1}^k(s,a)
\le\beta_h^k(s,a).
\label{eq:bellman-residual-publication}
\end{equation}

For every $a\in\cA$,
\begin{equation}
Q_h^k(s,a)-r_h(s,a)-P_hV_{h+1}^k(s,a)
\le\widetilde\beta_h^k(s,a).
\label{eq:all-action-residual-publication}
\end{equation}
\end{lemma}

\begin{proof}
Fix $(s,a,h,k)$ and suppress these indices when unambiguous. Let
$N=N_h^k(s,a)$ and $B=H-h$. If $N\le1$, the residual is at most $H$
by clipping and nonnegative rewards. The linear term in either bound
covers this case. Suppose $N\ge2$, and set $x=\iota/N$,
$p=P_h(\cdot\mid s,a)$, and $q=\widehat P_h^k(\cdot\mid s,a)$.
The reward terms satisfy
\begin{align}
2\sqrt{\widehat\sigma_{R,h}^{2,k}x}+5x+|\widehat r_h^k-r_h|
&\le4\sqrt{\widehat\sigma_{R,h}^{2,k}x}+10x\notag\\
&\le \mathcal{O}\!\left(\sqrt{\Var(R_h(s,a))x}+x\right)
\label{eq:reward-residual-bound-complete}
\end{align}
by~\eqref{eq:reward-eb-publication} and
\eqref{eq:reward-var-comparison-publication}.

For a safe action $a\in\cA_{\mathrm{safe},h}(s)$, write
$V=V_{h+1}^\star$ and $D=D_{h+1}^k$. On the true support,
$V_{h+1}^k=V+D$ and $V,D\in[0,B]$. The empirical support is contained
in the true support. Clipping gives
\begin{equation}
Q_h^k-r_h-pV_{h+1}^k
\le b_h^k+(\widehat r_h^k-r_h)+(q-p)V+(q-p)D.
\label{eq:residual-expansion-publication}
\end{equation}
By~\eqref{eq:functional-variance-publication}, for $z=V,D$,
$\sqrt{\Var(q,z)x}\le\sqrt{2\Var(p,z)x}+2\sqrt S\,Bx$.
Using $\Var(q,V+D)\le2\Var(q,V)+2\Var(q,D)$, we obtain
\begin{align}
&4\sqrt{\Var(q,V+D)x}+16Bx+|(q-p)V|+|(q-p)D|\notag\\
&\quad\le(4\sqrt2+2)\sqrt{\Var(q,V)x}
+4\sqrt2\sqrt{\Var(q,D)x}
+2\sqrt{S\Var(p,D)x}+(21+3S)Bx\notag\\
&\quad\le \mathcal{O}\!\left(\sqrt{\Var(p,V)x}
+\sqrt{S\Var(p,D)x}+SBx\right).
\label{eq:transition-residual-final-complete}
\end{align}
The first inequality uses~\eqref{eq:value-eb-publication} for the fixed
function $V$ and~\eqref{eq:functional-mean-publication} for $D$.
The second uses the two variance comparisons above.
Combining~\eqref{eq:residual-expansion-publication},
\eqref{eq:reward-residual-bound-complete}, and
\eqref{eq:transition-residual-final-complete} proves
\eqref{eq:bellman-residual-publication}.

Now let $a$ be arbitrary and write $V=V_{h+1}^k$.
Lemma~\ref{lem:transition-interface-publication} applies directly to
this data-dependent function and gives
\[
4\sqrt{\Var(q,V)x}+16Bx+|(q-p)V|
\le \mathcal{O}\!\left(\sqrt{S\Var(p,V)x}+SBx\right).
\]
Together with~\eqref{eq:reward-residual-bound-complete} and the clipped
backup, this proves~\eqref{eq:all-action-residual-publication}.
A sufficiently large $C_0$ makes both residual bounds hold.
\end{proof}

\subsection{Regret Decomposition}
\label{app:safe-projection}

At viable states, we replace each unsafe action of $\pi^k$ by the
corresponding action of $\pi^\star$. This gives a safe policy that
agrees with $\pi^k$ before the first unsafe action and allows us to
use the safe-policy variance. For $h\in[H]$ and
$s\in\cV_{\mathrm{safe},h}$, define
\begin{equation}
\bar\pi_h^k(s)
:=
\begin{cases}
\pi_h^k(s),&\pi_h^k(s)\in\cA_{\mathrm{safe},h}(s),\\
\pi_h^\star(s),&\pi_h^k(s)\notin\cA_{\mathrm{safe},h}(s),
\end{cases}
\label{eq:safe-projection-publication}
\end{equation}
with arbitrary actions at states outside $\cV_{\mathrm{safe},h}$.
For $h\in[H+1]$, define
$E_h^k(s):=V_h^k(s)-V_h^{\bar\pi^k}(s)$ for
$s\in\cV_{\mathrm{safe},h}$ and $E_h^k(s):=0$ otherwise.
Thus $D_h^k$ compares with the optimal safe value, whereas
$E_h^k$ compares with the value of $\bar\pi^k$.
For every $h\in[H]$, define the Bellman residual at the executed pair
\begin{equation*}
f_h^k
:=V_h^k(s_h^k)-r_h(s_h^k,a_h^k)
-P_hV_{h+1}^k(s_h^k,a_h^k).
\end{equation*}

\begin{lemma}
\label{lem:safe-projection-publication}
For every $k\in[K]$, the policy $\bar\pi^k$ is
$\cF_1^k$-measurable and belongs to $\Pi_{\mathrm{safe}}$.
Couple $\pi^k$ and $\bar\pi^k$ using the same transition randomness.
Their states then agree through step $\tau_k$, and their actions
agree for all $h\in[H]$ with $h<\tau_k$.
\par\nopagebreak
For every $k\in[K]$ and $h\in[H]$ with $h<\tau_k$,
\begin{equation}
\begin{aligned}
D_h^k(s_h^k)-P_hD_{h+1}^k(s_h^k,a_h^k)&\le f_h^k,\\
E_h^k(s_h^k)-P_hE_{h+1}^k(s_h^k,a_h^k)&=f_h^k.
\end{aligned}
\label{eq:gap-drifts-publication}
\end{equation}
On $\Est$, $0\le D_h^k(s)\le E_h^k(s)\le H-h+1$
for every $k\in[K]$, $h\in[H+1]$, and
$s\in\cV_{\mathrm{safe},h}$.
\end{lemma}

\begin{proof}
At every viable state, both cases in
\eqref{eq:safe-projection-publication} select an action in
$\cA_{\mathrm{safe},h}(s)$. The support condition therefore makes $\bar\pi^k$ safe, and its
definition makes it $\cF_1^k$-measurable. For $h<\tau_k$, its action
agrees with the executed action, so the coupled trajectories agree
through the state at step $\tau_k$. Applying
\eqref{eq:safe-bellman-publication} to the safe action $a_h^k$ gives
the inequality for $D_h^k$ in~\eqref{eq:gap-drifts-publication}.
The Bellman equation for $\bar\pi^k$ gives the equality for $E_h^k$.
On $\Est$, Lemmas~\ref{lem:optimism-publication} and
\ref{lem:safe-comparator-publication} give
$V_h^k\ge V_h^\star\ge V_h^{\bar\pi^k}$ on $\cV_{\mathrm{safe},h}$.
Clipping and nonnegative rewards give the upper bound.
\end{proof}

Define
\begin{align*}
B_K
&:=\sum_{k=1}^K\sum_{h<\tau_k}
\min\{\beta_h^k(s_h^k,a_h^k),H\},\\
\widetilde B_K
&:=\sum_{k=1}^K\sum_{h=\tau_k}^{H}
\min\{\widetilde\beta_h^k(s_h^k,a_h^k),H\}.
\end{align*}
The second inner sum is empty when $\tau_k=H+1$. To pass from expected
residuals to these realized sums, define the episode contribution
\begin{equation*}
Y_k
:=\sum_{h<\tau_k}\min\{\beta_h^k(s_h^k,a_h^k),H\}
+\sum_{h=\tau_k}^{H}
\min\{\widetilde\beta_h^k(s_h^k,a_h^k),H\},
\qquad 0\le Y_k\le H^2.
\end{equation*}

\begin{proposition}
\label{prop:deterministic-regret-reduction-publication}
On $\Est$, the reward regret satisfies
$\cR(K)\le\sum_{k=1}^K\E[Y_k\mid\cF_1^k]$.
\end{proposition}

\begin{proof}
Fix the history at the start of episode $k$. The Bellman equations for
$\pi^k$, with $V_{H+1}^k=V_{H+1}^{\pi^k}=0$, give
\begin{equation}
V_1^k(s_1)-V_1^{\pi^k}(s_1)
=\E\!\left[\sum_{h=1}^H f_h^k\,\middle|\,\cF_1^k\right].
\label{eq:full-bellman-identity-publication}
\end{equation}
If the episode's estimates satisfy the inequalities defining $\Est$,
Lemma~\ref{lem:bellman-envelope-publication} bounds the residual at
every state and action. Hence, for every continuation trajectory,
$\sum_h f_h^k\le Y_k$: use $\beta_h^k$ before $\tau_k$,
$\widetilde\beta_h^k$ from $\tau_k$ onward, and $f_h^k\le H$ throughout.
On $\Est$, optimism therefore yields
\[
V_1^\star(s_1)-V_1^{\pi^k}(s_1)
\le V_1^k(s_1)-V_1^{\pi^k}(s_1)
\le\E[Y_k\mid\cF_1^k].
\]
Summing over $k$ proves the result.
\end{proof}

\begin{samepage}
\subsection{Variance Bounds}
\label{app:prefix-suffix-variance}

Proposition~\ref{prop:deterministic-regret-reduction-publication}
reduces regret to conditional expectations of residual sums.
We first control the visit-count and variance terms in these sums,
then collect the concentration bounds that relate them to their
conditional expectations. We write $\sum_{k,h}$ for
$\sum_{k=1}^K\sum_{h=1}^H$; a restriction such as $h<\tau_k$ is
applied within each episode.

The reward and optimal-value variances are fixed for each state,
action, and step, whereas the value-error variances can change between
episodes. The next lemma gives visit-count bounds for both types of weights.

\begin{lemma}[Visit-count sums]
\label{lem:visit-count-summation-publication}
Fix $(s,a,h)$ and enumerate its visits during the $K$ episodes.
Set $n_j:=\max\{j-1,1\}$, the visit count before visit $j$. For any subset $J$ of
these visit indices, write $m:=|J|$. Then
$\sum_{j\in J}n_j^{-1/2}\le2\sqrt m$ and
$\sum_{j\in J}n_j^{-1}\le L_K$.
\par\nopagebreak
Consequently, for any nonnegative weights $v_h(s,a)$ fixed across
episodes and any nonnegative weights $z_h^k(s,a)$, the following
bounds hold for every realized interaction sequence:
\begin{align*}
\sum_{k,h<\tau_k}
\sqrt{\frac{v_h(s_h^k,a_h^k)}{n_h^k(s_h^k,a_h^k)}}
&\le2\sqrt{SAH
\sum_{k,h<\tau_k}v_h(s_h^k,a_h^k)},\\
\sum_{k,h<\tau_k}
\sqrt{\frac{z_h^k(s_h^k,a_h^k)}{n_h^k(s_h^k,a_h^k)}}
&\le\sqrt{SAHL_K
\sum_{k,h<\tau_k}z_h^k(s_h^k,a_h^k)},\\
\sum_{k,h<\tau_k}
\frac1{n_h^k(s_h^k,a_h^k)}
&\le SAHL_K.
\end{align*}
The same bounds hold with $h\ge\tau_k$.
\end{lemma}
\end{samepage}

\begin{proof}
Both reciprocal sequences are nonincreasing, so the sum over a subset
of size $m$ is maximized by the first $m$ visits.  For $m\ge2$,
\[
2+\sum_{q=2}^{m-1}q^{-1/2}
\le2+\int_1^{m-1}x^{-1/2}\,dx
=2\sqrt{m-1}\le2\sqrt m.
\]
Similarly, $m\le K$ gives
$2+\sum_{q=2}^{m-1}q^{-1}\le2+\log K=L_K$.
The case $m\le1$ is immediate. For the fixed weights $v_h(s,a)$,
first apply the bound $2\sqrt m$ at each $(s,a,h)$, and then apply
Cauchy--Schwarz over the $SAH$ choices. For the varying weights,
apply Cauchy--Schwarz to all visited terms and use
$\sum_{k,h<\tau_k}1/n_h^k(s_h^k,a_h^k)\le SAHL_K$.
This proves all three bounds and also applies to $h\ge\tau_k$.
\end{proof}

After summing the visit-count factors, we need bounds on the
next-state variances of $D^k$ and $E^k$. The next lemma applies to
bounded nonnegative functions. Telescoping their squares before
$\tau_k$ leaves a drift term and a boundary term at $\tau_k$.
We will apply it to the positive parts of $D^k$ and $E^k$.

\begin{lemma}
\label{lem:safe-prefix-square-publication}
For each $k\in[K]$, let $g_h^k:\cS\to[0,H]$, $h\in[H+1]$, be
$\cF_1^k$-measurable functions. Assume $g_{H+1}^k\equiv0$ and
$g_h^k(s)=0$ outside $\cV_{\mathrm{safe},h}$.
Let $u_h^k$, $k\in[K]$ and $h\in[H]$, be random variables.
Suppose that, almost surely, for every $k\in[K]$ and
$h\in[H]$ with $h<\tau_k$, $0\le u_h^k\le H$ and
\[
g_h^k(s_h^k)-P_hg_{h+1}^k(s_h^k,a_h^k)\le u_h^k.
\]
Then, for every $\eta\in(0,1)$, with probability at least $1-\eta$,
\begin{align}
&\sum_{k,h<\tau_k}
\Var(P_h(\cdot\mid s_h^k,a_h^k),g_{h+1}^k)\le
\mathcal{O}\!\left(
H\sum_{k,h<\tau_k}u_h^k
+H^2\sum_{k=1}^K\Ind\{\tau_k\le H\}
+H^2\log\frac{KH+1}{\eta}
\right).
\label{eq:safe-prefix-square-publication}
\end{align}
\end{lemma}

\begin{proof}
Set
$I_h^k=\Ind\{h<\tau_k\}$ and
$q_h^k=P_hg_{h+1}^k(s_h^k,a_h^k)$. The indicator $I_h^k$ is measurable
after the step-$h$ action is chosen and before the successor is generated.
In episode and step order, define the martingale differences
\[
\xi_h^k
:=I_h^k\left\{
P_h[(g_{h+1}^k)^2](s_h^k,a_h^k)
-(g_{h+1}^k(s_{h+1}^k))^2
\right\}.
\]
Within episode $k$, telescoping gives
\begin{align*}
&\sum_{h<\tau_k}\left\{(g_{h+1}^k(s_{h+1}^k))^2-(q_h^k)^2\right\}\\
&\quad=\sum_{h<\tau_k}\left\{(g_h^k(s_h^k))^2-(q_h^k)^2\right\}
-(g_1^k(s_1))^2+(g_{\tau_k}^k(s_{\tau_k}^k))^2\\
&\quad\le2H\sum_{h<\tau_k}u_h^k+H^2\Ind\{\tau_k\le H\}.
\end{align*}
Indeed, $(g_h^k)^2-(q_h^k)^2\le2H[g_h^k-q_h^k]_+\le2Hu_h^k$.
The terminal square is zero when $\tau_k=H+1$ and at most $H^2$
otherwise. When $\tau_k=1$, both sums are empty and the two boundary squares cancel.

Let
\[
Z_g:=\sum_{k,h<\tau_k}
\Var(P_h(\cdot\mid s_h^k,a_h^k),g_{h+1}^k).
\]
Adding the martingale differences to the deterministic sum above
gives
\[
Z_g
\le\sum_{k,h}\xi_h^k
+2H\sum_{k,h<\tau_k}u_h^k
+H^2\sum_k\Ind\{\tau_k\le H\}.
\]
The increments have absolute value at most $H^2$. For any
$G\in[0,H]$ and an independent copy $G'$, the variance identity gives
\[
\Var(G^2)=\frac12\E[(G^2-(G')^2)^2]
\le2H^2\E[(G-G')^2]=4H^2\Var(G).
\]
Apply this conditionally to $G=g_{h+1}^k(s_{h+1}^k)$. The
predictable quadratic variation of the martingale is therefore at most
$4H^2Z_g$.  Lemma~\ref{lem:freedman-publication} yields
\[
\sum_{k,h}\xi_h^k
\le \mathcal{O}\!\left(
H\sqrt{Z_g\log\frac{KH+1}{\eta}}
+H^2\log\frac{KH+1}{\eta}
\right).
\]
Young's inequality absorbs the square-root term into $Z_g/2$ and proves
\eqref{eq:safe-prefix-square-publication}.
\end{proof}

For the steps from $\tau_k$ onward, we telescope the squares of $V^k$
up to step $H+1$. The terminal value is zero, and the square at
$\tau_k$ enters with a negative sign, so no positive boundary term remains.

\begin{lemma}
\label{lem:suffix-square-publication}
There exists an event $\cE_{\mathrm{suf}}$ with
$\Prob(\cE_{\mathrm{suf}})\ge1-\delta/16$.
On $\Est\cap\cE_{\mathrm{suf}}$, the accumulated variance satisfies
\begin{align}
&\sum_{k=1}^K\sum_{h=\tau_k}^{H}
\Var(P_h(\cdot\mid s_h^k,a_h^k),V_{h+1}^k)\le
\mathcal{O}\!\left(
H\widetilde B_K
+H^2\sum_{k=1}^K\Ind\{\tau_k\le H\}
+H^2\iota
\right).
\label{eq:suffix-square-publication}
\end{align}
\end{lemma}

\begin{proof}
Set $J_h^k:=\Ind\{\tau_k\le h\}$,
$q_h^k:=P_hV_{h+1}^k(s_h^k,a_h^k)$, and
\[
u_h^k:=\min\!\left\{H,
1+\min\{\widetilde\beta_h^k(s_h^k,a_h^k),H\}\right\}.
\]
The indicators $J_h^k$ are predictable, and $V^k$ is fixed before
episode $k$. In episode and step order, define the martingale differences
$\xi_h^k
:=J_h^k\left\{
P_h[(V_{h+1}^k)^2](s_h^k,a_h^k)
-(V_{h+1}^k(s_{h+1}^k))^2
\right\}$.
Let $Z$ denote the left-hand side of
\eqref{eq:suffix-square-publication}. The increments are bounded by
$H^2$. As in the proof of Lemma~\ref{lem:safe-prefix-square-publication},
their conditional variance sum is at most $4H^2Z$.
Apply Lemma~\ref{lem:freedman-publication} with failure probability
$\delta/16$, and let $\cE_{\mathrm{suf}}$ be the resulting event.
Since $\log(16(KH+1)/\delta)\le2\iota$, on $\cE_{\mathrm{suf}}$ we have
\[
\sum_{k,h}\xi_h^k
\le\mathcal{O}\!\left(H\sqrt{Z\iota}+H^2\iota\right).
\]

We now bound the drift on $\Est$. The residual bound
\eqref{eq:all-action-residual-publication}, $r_h\le1$, and clipping give
$V_h^k(s_h^k)-q_h^k\le u_h^k$. By the definition of $u_h^k$,
\begin{equation}
\sum_{k,h}J_h^ku_h^k
\le
\widetilde B_K
+H\sum_{k=1}^K\Ind\{\tau_k\le H\}.
\label{eq:suffix-drift-publication}
\end{equation}
If $\tau_k\le H$, telescoping gives
\begin{align*}
&\sum_{h=\tau_k}^{H}
\left\{(V_{h+1}^k(s_{h+1}^k))^2-(q_h^k)^2\right\}\\
&\quad=
\sum_{h=\tau_k}^{H}
\left\{(V_h^k(s_h^k))^2-(q_h^k)^2\right\}
-(V_{\tau_k}^k(s_{\tau_k}^k))^2
+(V_{H+1}^k(s_{H+1}^k))^2\\
&\quad\le2H\sum_{h=\tau_k}^{H}u_h^k.
\end{align*}
The inequality follows from
$(V_h^k)^2-(q_h^k)^2\le2H[V_h^k-q_h^k]_+\le2Hu_h^k$
and $V_{H+1}^k\equiv0$. For $\tau_k=H+1$, the sum is empty.
Adding the martingale differences and summing over episodes gives
$Z\le\sum_{k,h}\xi_h^k+2H\sum_{k,h}J_h^ku_h^k$.
On $\Est\cap\cE_{\mathrm{suf}}$, Young's inequality absorbs the
term containing $\sqrt Z$ into $Z/2$ and gives
\[
Z\le \mathcal{O}\!\left(H\sum_{k,h}J_h^ku_h^k+H^2\iota\right).
\]
Substituting~\eqref{eq:suffix-drift-publication} proves the claim.
\end{proof}

We next apply the preceding bounds to the residual sums and the
auxiliary safe policy. Define the accumulated next-state variances of
the two value differences by
\begin{align*}
Z_D
&:=\sum_{k,h<\tau_k}
\Var(P_h(\cdot\mid s_h^k,a_h^k),D_{h+1}^k),\\
Z_E
&:=\sum_{k,h<\tau_k}
\Var(P_h(\cdot\mid s_h^k,a_h^k),E_{h+1}^k).
\end{align*}
The auxiliary policy relates the variance accumulated before $\tau_k$
to a safe-policy return variance. For each episode, define
\begin{equation}
A_k
:=\sum_{h<\tau_k}
\left\{
\Var(R_h(s_h^k,a_h^k))
+\Var(P_h(\cdot\mid s_h^k,a_h^k),V_{h+1}^{\bar\pi^k})
\right\}.
\label{eq:AK-publication}
\end{equation}

The next lemma converts the expected residuals to realized sums and
bounds the accumulated variances needed in the regret proof.

\begin{samepage}
\begin{lemma}
\label{lem:martingale-event-publication}
There exists an event $\Mart$ with $\Prob(\Mart)\ge1-5\delta/16$
such that both of the following bounds hold on $\Mart$:
\begin{align}
\sum_{k=1}^K\E[Y_k\mid\cF_1^k]
&\le2(B_K+\widetilde B_K)+\mathcal{O}(H^2\iota),
\label{eq:Y-conversion-publication}\\
\sum_{k=1}^KA_k
&\le2K\VarStar+\mathcal{O}(H^3\iota).
\label{eq:AK-bound-publication}
\end{align}

On $\Est\cap\Mart$, the value-error variances also satisfy
\begin{equation}
Z_D+Z_E
\le
\mathcal{O}\!\left(
HB_K
+H^2\sum_{k=1}^K\Ind\{\tau_k\le H\}
+H^2\iota
\right),
\label{eq:two-square-publication}
\end{equation}
and \eqref{eq:suffix-square-publication} holds for the steps
$h\ge\tau_k$.
\end{lemma}
\end{samepage}

\begin{proof}
For the expected residual sum, apply
\eqref{eq:predictable-upper-publication} to $Y_k$ with $B=H^2$
and failure probability $\delta/16$, and let $\cE_Y$ be the
resulting event. Since $\sum_kY_k=B_K+\widetilde B_K$,
\eqref{eq:Y-conversion-publication} holds on $\cE_Y$, with
$\Prob(\cE_Y)\ge1-\delta/16$.

To control the safe-policy variance, condition on $\cF_1^k$ and
couple $\pi^k$ with $\bar\pi^k$ using the same transition randomness.
Both policies are fixed given this history.
Their states and actions agree for $h<\tau_k$. Since the summands
in~\eqref{eq:AK-publication} are nonnegative, $A_k$ is at most the
full variance sum along the trajectory of $\bar\pi^k$. Hence
\[
\E[A_k\mid\cF_1^k]
\le\E^{\bar\pi^k}\!\left[
\sum_{h=1}^H\left\{\Var(R_h(s_h,a_h))+
\Var(P_h(\cdot\mid s_h,a_h),V_{h+1}^{\bar\pi^k})\right\}
\,\middle|\,s_1\right]
\le\VarStar.
\]
Iterating the Bellman variance decomposition identifies the middle
expectation with the return variance of $\bar\pi^k$. The decomposition
uses the conditional independence of the reward and successor state. Each reward variance is at most
$1/4$, and each step-$h$ next-state value variance is at most $(H-h)^2/4$.
Thus $0\le A_k\le\frac14(H+\sum_{j=0}^{H-1}j^2)\le H^3/3$.
Apply~\eqref{eq:realized-upper-publication} with $B=H^3/3$
and failure probability $\delta/16$, and let $\cE_A$ be the resulting
event. On $\cE_A$,
\[
\sum_{k=1}^KA_k\le2\sum_{k=1}^K\E[A_k\mid\cF_1^k]+\mathcal{O}(H^3\iota)
\le2K\VarStar+\mathcal{O}(H^3\iota).
\]
Thus $\Prob(\cE_A)\ge1-\delta/16$.

For the value-error variances, first consider $[D_h^k]_+$ and
$[E_h^k]_+$. These functions are $\cF_1^k$-measurable, lie in
$[0,H]$ on every sample path, and vanish at step $H+1$. The drift relations in
Lemma~\ref{lem:safe-projection-publication}, together with
$[x]_+-[y]_+\le[x-y]_+$ and
$P_h[g_{h+1}]_+\ge[P_hg_{h+1}]_+$, give, for $h<\tau_k$,
\begin{align*}
[D_h^k(s_h^k)]_+-P_h[D_{h+1}^k]_+(s_h^k,a_h^k)
&\le[f_h^k]_+,\\
[E_h^k(s_h^k)]_+-P_h[E_{h+1}^k]_+(s_h^k,a_h^k)
&\le[f_h^k]_+.
\end{align*}
Since $0\le[f_h^k]_+\le H$, the assumptions of
Lemma~\ref{lem:safe-prefix-square-publication} hold with
$u_h^k=[f_h^k]_+$ on every sample path.
Apply Lemma~\ref{lem:safe-prefix-square-publication} with
$g_h^k=[D_h^k]_+$ and $g_h^k=[E_h^k]_+$, using this $u_h^k$ and
$\eta=\delta/16$. Let $\cE_D$ and $\cE_E$ be the respective events.
Each has probability at least $1-\delta/16$.
On $\Est$, Lemma~\ref{lem:safe-projection-publication} gives
$D^k,E^k\ge0$, and Lemma~\ref{lem:bellman-envelope-publication} gives
$[f_h^k]_+\le\min\{\beta_h^k(s_h^k,a_h^k),H\}$. Hence
\eqref{eq:two-square-publication} holds on
$\Est\cap\cE_D\cap\cE_E$.

Set $\Mart:=\cE_Y\cap\cE_A\cap\cE_D\cap\cE_E\cap\cE_{\mathrm{suf}}$,
where $\cE_{\mathrm{suf}}$ is supplied by
Lemma~\ref{lem:suffix-square-publication}.
A union bound gives $\Prob(\Mart)\ge1-5\delta/16$.
\end{proof}

\subsection{Proof of Theorem~\ref{thm:regret}}
\label{app:main-completion}

The preceding bounds reduce the proof to controlling $B_K$ and
$\widetilde B_K$. We bound $B_K$ using each of the two variance
quantities and take the smaller bound. The number of episodes
containing an unsafe action controls the remaining sum $\widetilde B_K$.

Define $\MainEvent:=\Est\cap\Cert\cap\Mart$.
Lemmas~\ref{lem:estimation-event-publication},
\ref{lem:false-candidate-control}, and
\ref{lem:martingale-event-publication} give
\begin{equation}
\Prob(\MainEvent)
\ge1-\frac{9+1+5}{16}\delta
\ge1-\delta.
\label{eq:master-event-prob-publication}
\end{equation}
We work on $\MainEvent$ throughout this subsection.

Proposition~\ref{prop:deterministic-regret-reduction-publication} and
\eqref{eq:Y-conversion-publication} give
\begin{equation}
\cR(K)
\le2(B_K+\widetilde B_K)+\mathcal{O}(H^2\iota).
\label{eq:full-regret-interface-publication}
\end{equation}
\paragraph{Bounding $B_K$}
Combining~\eqref{eq:two-square-publication} with
\eqref{eq:departure-realized-publication} yields
\begin{equation}
Z_D+Z_E
\le
\mathcal{O}\!\left(HB_K+H^2\iota
\left[\min\{K,\cC\}+1\right]\right).
\label{eq:two-square-closed-publication}
\end{equation}

For the steps $h<\tau_k$,
$V_{h+1}^\star=V_{h+1}^{\bar\pi^k}+E_{h+1}^k-D_{h+1}^k$.
The successor belongs to $\cV_{\mathrm{safe},h+1}$, so this identity holds on the
support of the executed transition. Applying
$\Var(X+Y+Z)\le3\{\Var(X)+\Var(Y)+\Var(Z)\}$ and including the
nonnegative reward variances gives
\begin{equation}
\Var_\Sigma\le3\left(\sum_{k=1}^KA_k+Z_D+Z_E\right).
\label{eq:reference-variance-publication}
\end{equation}
Lemma~\ref{lem:visit-count-summation-publication} and
\eqref{eq:beta-publication} give
\begin{equation}
B_K
\le \mathcal{O}\!\left(
\sqrt{SAH\iota\,\Var_\Sigma}
+\sqrt{S^2AH\iota L_K Z_D}
+S^2AH^2\iota L_K
\right).
\label{eq:BK-variance-interface-publication}
\end{equation}
For the bound in terms of $\Var_\Sigma$, substitute
\eqref{eq:two-square-closed-publication} into
\eqref{eq:BK-variance-interface-publication} to obtain
\begin{align*}
B_K\le \mathcal{O}\!\Bigl(&\sqrt{SAH\iota\,\Var_\Sigma}
+S\sqrt{AH^2\iota L_K B_K}\\
&+S\sqrt{AH^3\min\{K,\cC\}}\,\iota\sqrt{L_K}
+S^2AH^2\iota L_K\Bigr).
\end{align*}
Young's inequality bounds the term containing $\sqrt{B_K}$ by
$B_K/2+\mathcal{O}(S^2AH^2\iota L_K)$. Moving $B_K/2$ to the left yields
\begin{align}
B_K\le \mathcal{O}\!\Bigl(&\sqrt{SAH\iota\,\Var_\Sigma}+S\sqrt{AH^3\min\{K,\cC\}}\,\iota\sqrt{L_K}
+S^2AH^2\iota L_K\Bigr).
\label{eq:BK-trajectory-publication}
\end{align}
For the bound in terms of $\VarStar$,
\eqref{eq:reference-variance-publication} and
\eqref{eq:AK-bound-publication} give
\[
\sqrt{SAH\iota\,\Var_\Sigma}
\le\mathcal O\!\left(
\sqrt{SAHK\VarStar\,\iota}
+\sqrt{SAH\iota(Z_D+Z_E)}
+H^2\iota\sqrt{SA}\right).
\]
Substitute this inequality into~\eqref{eq:BK-variance-interface-publication}
and use~\eqref{eq:two-square-closed-publication} for $Z_D+Z_E$.
Since $S,L_K\ge1$, the resulting terms involving $B_K$ and
$\min\{K,\cC\}$ are bounded by the corresponding
terms above. Also, $H^2\iota\sqrt{SA}\le S^2AH^2\iota L_K$.
Applying the same absorption step gives
\begin{align}
B_K\le \mathcal{O}\!\Bigl(&\sqrt{SAHK\VarStar\,\iota}+S\sqrt{AH^3\min\{K,\cC\}}\,\iota\sqrt{L_K}
+S^2AH^2\iota L_K\Bigr).
\label{eq:BK-final-publication}
\end{align}
Both bounds hold on $\MainEvent$. Taking their minimum gives
\begin{align}
B_K
\le \mathcal{O}\!\Bigl(&
\sqrt{SAH\iota\,\min\{\Var_\Sigma,K\VarStar\}}+S\sqrt{AH^3\min\{K,\cC\}}\,\iota\sqrt{L_K}
+S^2AH^2\iota L_K\Bigr).
\label{eq:BK-min-publication}
\end{align}

\paragraph{Bounding $\widetilde B_K$}
The remaining steps occur only in episodes with $\tau_k\le H$.
Since each reward variance is at most $1/4$,
\[
\sum_{k=1}^K\sum_{h=\tau_k}^{H}\Var(R_h(s_h^k,a_h^k))
\le\frac H4\sum_{k=1}^K\Ind\{\tau_k\le H\}.
\]
Apply Lemma~\ref{lem:visit-count-summation-publication} to
\eqref{eq:all-action-beta-publication} over the visits with
$h\ge\tau_k$:
\begin{align*}
\widetilde B_K
&\le \mathcal{O}\!\Bigg(
\sqrt{SAH\iota
\sum_{k=1}^K\sum_{h=\tau_k}^{H}\Var(R_h(s_h^k,a_h^k))}\\
&\qquad
+S\sqrt{AH\iota L_K
\sum_{k=1}^K\sum_{h=\tau_k}^{H}
\Var(P_h(\cdot\mid s_h^k,a_h^k),V_{h+1}^k)}
+S^2AH^2\iota L_K\Bigg).
\end{align*}
The reward-variance bound above and~\eqref{eq:suffix-square-publication}
then give, using $S,A,H,L_K\ge1$,
\[
\widetilde B_K
\le \mathcal{O}\!\Bigl(
S\sqrt{AH^3\iota L_K\sum_{k=1}^K\Ind\{\tau_k\le H\}}
+S\sqrt{AH^2\iota L_K\,\widetilde B_K}
+S^2AH^2\iota L_K\Bigr).
\]
Young's inequality bounds the term containing
$\sqrt{\widetilde B_K}$ by
$\widetilde B_K/2+\mathcal{O}(S^2AH^2\iota L_K)$.
Move $\widetilde B_K/2$ to the left and use
\eqref{eq:departure-realized-publication} to obtain
\begin{equation}
\widetilde B_K
\le \mathcal{O}\!\left(
S\sqrt{AH^3\min\{K,\cC\}}\,\iota\sqrt{L_K}
+S^2AH^2\iota L_K\right).
\label{eq:suffix-final-publication}
\end{equation}

\begin{proof}[Proof of Theorem~\ref{thm:regret}]
Substitute~\eqref{eq:BK-min-publication} and
\eqref{eq:suffix-final-publication} into
\eqref{eq:full-regret-interface-publication}. Since
$H^2\iota\le S^2AH^2\iota L_K$, this gives
\begin{align}
\cR(K)
\le \mathcal{O}\!\Bigl(&
\sqrt{SAH\iota\,\min\{\Var_\Sigma,K\VarStar\}}%
+S\sqrt{AH^3\min\{K,\cC\}}\,\iota\sqrt{L_K}
+S^2AH^2\iota L_K\Bigr).
\label{eq:main-regret-trajectory-final}
\end{align}
The probability statement follows from
\eqref{eq:master-event-prob-publication}. Suppressing the logarithmic
factors in $\iota$ and $L_K$ proves Theorem~\ref{thm:regret}.
\end{proof}

\subsection{A Special Case with Direct Cost Violations}
\label{app:immediate-cost-specialization}

Suppose Assumptions~\ref{ass:init} and~\ref{ass:margin} hold, and
every unsafe action at a vertex reachable under a safe policy has
cost exceeding $\tau$. Then $\cC\le SAH/\Delta_c^2$,
and SVAE satisfies, with probability at least $1-\delta$,
\begin{equation}
\cR(K)\le\widetilde{\mathcal O}\!\left(
\sqrt{SAH\min\{\Var_\Sigma,K\VarStar\}}
+S\sqrt{AH^3\min\{K,SAH/\Delta_c^2\}}
+S^2AH^2\right).
\label{eq:immediate-cost-regret}
\end{equation}

\begin{proof}
For every action contributing to $\cC$, the
direct-cost branch of~\eqref{eq:recursive-kappa} gives
$\kappa_h(s,a)=(c_h(s,a)-\tau)^{-2}\le\Delta_c^{-2}$.
There are at most $SAH$ such actions, so
$\cC\le SAH/\Delta_c^2$.
Substitution into Theorem~\ref{thm:regret} yields
\eqref{eq:immediate-cost-regret} on the same event.
\end{proof}

\section{Constraint Violation}
\label{app:violation}

We prove Theorem~\ref{thm:violation} on
$\cE_{\mathrm{cs}}=\cE_c\cap\cE_{\mathrm{supp}}$, which has probability
at least $1-\delta/8$ by Lemma~\ref{lem:estimation-event-publication}.
We separate steps with nonempty candidate sets from those with empty
candidate sets. The cost bound controls the first contribution, and
newly observed transitions control the second.

Let $\mathrm{nemp}$ and $\mathrm{emp}$ denote nonempty and empty
candidate sets, respectively, and define
\begin{align}
\Viol_{\mathrm{nemp}}(K)
&:=
\sum_{k=1}^{K}\sum_{h=1}^{H}
\bigl[c_h(s_h^k,a_h^k)-\tau\bigr]_+
\Ind\!\left\{
\hat\cA_{\mathrm{safe},h}^k(s_h^k)\ne\varnothing
\right\},
\label{eq:vio-nemp-definition}\\
\Viol_{\mathrm{emp}}(K)
&:=
\sum_{k=1}^{K}\sum_{h=1}^{H}
\bigl[c_h(s_h^k,a_h^k)-\tau\bigr]_+
\Ind\!\left\{
\hat\cA_{\mathrm{safe},h}^k(s_h^k)=\varnothing
\right\}.
\notag
\end{align}
Then, deterministically,
\begin{equation}
\Viol(K)
=
\Viol_{\mathrm{nemp}}(K)
+
\Viol_{\mathrm{emp}}(K).
\label{eq:vio-exact-decomposition}
\end{equation}

\paragraph{Nonempty candidate sets}

We first bound $\Viol_{\mathrm{nemp}}(K)$. Every selected candidate
action passes the cost test. This controls both its one-step violation
and the number of times it can be selected with a positive cost gap.

\begin{lemma}
\label{lem:vio-nemp-contribution}
On $\cE_c$, both
$\Viol_{\mathrm{nemp}}(K)\le4\sqrt2\,H\sqrt{SAK\iota}$ and
the following gap-dependent bound hold:
\begin{equation}
\Viol_{\mathrm{nemp}}(K)
\le
\sum_{c_h(s,a)>\tau}
\left\{
 c_h(s,a)-\tau
 +
 \frac{8\iota}{c_h(s,a)-\tau}
\right\}.
\label{eq:vio-nemp-instance}
\end{equation}
The sum ranges over $(s,a,h)\in\cS\times\cA\times[H]$ with a positive cost gap.
\end{lemma}

\begin{proof}
When the candidate set is nonempty, Algorithm~\ref{alg:svae} selects
$a\in\hat\cA_{\mathrm{safe},h}^k(s)$, so
$\underline c_h^k(s,a)\le\tau$. On $\cE_c$, this implies
\begin{align*}
c_h(s,a)-\tau
&\le c_h(s,a)-\widehat c_h^k(s,a)
+\sqrt{\frac{2\iota}{n_h^k(s,a)}}
\le2\sqrt{\frac{2\iota}{n_h^k(s,a)}}.
\end{align*}
Since $n_h^k=\max\{N_h^k,1\}$, this also covers the zero-count case. Hence
\begin{equation}
[c_h(s,a)-\tau]_+
\le2\sqrt{\frac{2\iota}{n_h^k(s,a)}}.
\label{eq:vio-per-step-bound}
\end{equation}
Let $M_h(s,a)$ count executions of $(s,a)$ at step $h$ with a nonempty
candidate set. The reciprocal-count bound in
Lemma~\ref{lem:visit-count-summation-publication} applies to this
subset of visits. Summing~\eqref{eq:vio-per-step-bound} and using
$\sum_{s,a,h}M_h(s,a)\le KH$ gives
\begin{align*}
\Viol_{\mathrm{nemp}}(K)
&\le4\sqrt{2\iota}\sum_{s,a,h}\sqrt{M_h(s,a)}\\
&\le4\sqrt{2SAH\iota\sum_{s,a,h}M_h(s,a)}\\
&\le4\sqrt2\,H\sqrt{SAK\iota}.
\end{align*}

For the second bound, fix $(s,a,h)$ with
$\Delta:=c_h(s,a)-\tau>0$. If $M_h(s,a)\ge2$, the last counted
execution has at least $M_h(s,a)-1$ prior visits. Applying
\eqref{eq:vio-per-step-bound} there gives
$\Delta\le2\sqrt{2\iota/(M_h(s,a)-1)}$. Therefore, including the cases
$M_h(s,a)\le1$,
\begin{equation}
M_h(s,a)\le1+\frac{8\iota}{\Delta^2}.
\label{eq:vio-gap-selection-count}
\end{equation}
Each counted execution contributes $\Delta$. Multiplying by $\Delta$
and summing over positive-gap triples proves
\eqref{eq:vio-nemp-instance}.
\end{proof}

\paragraph{Empty candidate sets}

It remains to bound $\Viol_{\mathrm{emp}}(K)$.
An empty candidate set can follow a candidate action only through a
newly observed transition. We associate each block of empty candidate
sets with this transition and bound its length by the remaining horizon.

\begin{lemma}
\label{lem:vio-emp-contribution}
Suppose Assumption~\ref{ass:init} holds. Then, on
$\cE_{\mathrm{cs}}$, $\Viol_{\mathrm{emp}}(K)\le S^2AH(H-1)/2$.
\end{lemma}

\begin{proof}
In each episode $k$, partition the steps with empty candidate sets into
maximal contiguous blocks. A block $\{\ell,\ldots,r\}$ has $\ell\ge2$:
Assumption~\ref{ass:init} and
the inclusions proved in Appendix~\ref{app:safe-superset} guarantee a nonempty
candidate set at $s_1$. By maximality, the preceding action is a
candidate, so
\begin{equation}
\widehat\supp_{\ell-1}^k(s_{\ell-1}^k,a_{\ell-1}^k)
\subseteq\hat\cV_{\mathrm{safe},\ell}^k.
\label{eq:vio-parent-closure}
\end{equation}
The state $s_\ell^k$ has no candidate actions. Consequently,
\begin{equation}
s_\ell^k\notin\hat\cV_{\mathrm{safe},\ell}^k,
\qquad
s_\ell^k\notin\widehat\supp_{\ell-1}^k(s_{\ell-1}^k,a_{\ell-1}^k).
\label{eq:vio-block-state-outside}
\end{equation}
Thus $(s_{\ell-1}^k,a_{\ell-1}^k,s_\ell^k,\ell-1)$ is a newly observed
transition. Associate the block with this transition.

Each transition is associated with at most one block. A given layer
is visited once per episode. In later episodes, the successor remains
in the observed support, so a candidate predecessor action forces it
to have a nonempty candidate set by~\eqref{eq:vio-parent-closure}.
It therefore cannot start another empty block through that action.
There are at most $S^2A$ transitions at each parent layer $h$, and an
associated block has length at most $H-h$. Since violation per step
is at most one,
\[
\Viol_{\mathrm{emp}}(K)
\le S^2A\sum_{h=1}^{H-1}(H-h)
=\frac{S^2AH(H-1)}2.\qedhere
\]
\end{proof}

\begin{proof}[Proof of Theorem~\ref{thm:violation}]
Substituting Lemmas~\ref{lem:vio-nemp-contribution} and
\ref{lem:vio-emp-contribution} into~\eqref{eq:vio-exact-decomposition}
gives, on $\cE_{\mathrm{cs}}$,
\[
\Viol(K)\le4\sqrt2\,H\sqrt{SAK\iota}+\frac{S^2AH(H-1)}2.
\]
If Assumption~\ref{ass:margin} also holds, every positive cost gap is
at least $\Delta_c$. The second bound in
Lemma~\ref{lem:vio-nemp-contribution} then gives, on the same event,
\[
\Viol(K)\le SAH\left(1+\frac{8\iota}{\Delta_c}\right)
+\frac{S^2AH(H-1)}2.
\]
The event has probability at least $1-\delta/8$. Suppressing logarithmic
factors proves both claims.
\end{proof}

\section{Cost Confidence Bounds at Zero Empirical Variance}
\label{app:silent-cliff}

A sample with zero empirical variance may come from a cost distribution
whose mean exceeds the safety threshold. We compare two distributions
that can produce the same constant sample to obtain a necessary
additive term in the confidence bound. We then apply this comparison
to the lower confidence test used for action removal in SVAE.

Fix $\tau\in(0,1)$ and $\delta\in(0,1/2)$. A certification rule observes
independent samples $X_1,X_2,\ldots$ from a common distribution $\nu$ on $[0,1]$,
may use independent auxiliary randomness, and returns a stopping time
$T\in\mathbb N\cup\{\infty\}$ at which it declares that the mean is at
most $\tau$. Here $\mathbb N=\{1,2,\ldots\}$, and
$\{T\le n\}$ depends only on the first $n$ observations and the
auxiliary randomness. Call the rule $\delta$-safe if
$\Prob_\nu(T<\infty)\le\delta$ for every such distribution with
$\E_\nu[X]>\tau$. The probability includes both sources of randomness.

Set $\Delta_{\max}:=\min\{\tau,1-\tau\}$. For each
$\Delta\in(0,\Delta_{\max})$, let
$p_{\Delta}:=2\Delta/(1-\tau+\Delta)$.
Under $\nu_0$, each sample equals $\tau-\Delta$ with probability one.
Under $\nu_1$, it equals $1$ with probability $p_\Delta$ and
$\tau-\Delta$ otherwise. Both distributions are supported on $[0,1]$,
with $\E_{\nu_0}[X]=\tau-\Delta$ and
$\E_{\nu_1}[X]=\tau+\Delta$.

\begin{proposition}
\label{prop:silent-cliff}
Fix $\Delta\in(0,\Delta_{\max})$ and let $\nu_0,\nu_1$ be the
distributions defined above. For every $\delta$-safe certification
rule $T$ and every $n\in\mathbb N$, we have
\begin{equation}
\Prob_{\nu_0}(T\le n)
\le
\frac{\delta}{(1-p_{\Delta})^n}.
\label{eq:hidden-spike-body}
\end{equation}
\end{proposition}

\begin{proof}
Since $p_\Delta\in(0,1)$, under $\nu_1$ the first $n$ observations
all equal $\tau-\Delta$ with probability $(1-p_\Delta)^n$. On that event, the rule sees exactly the
observations under $\nu_0$, and its auxiliary randomness has the same
distribution. Thus
\[
\delta\ge\Prob_{\nu_1}(T\le n)
\ge(1-p_\Delta)^n\Prob_{\nu_0}(T\le n),
\]
which proves~\eqref{eq:hidden-spike-body}.
\end{proof}

We now apply this comparison to confidence bounds based on empirical
variance. For $t\in\mathbb N$, define
$\overline X_t:=t^{-1}\sum_{i=1}^t X_i$
and $\widehat v_t:=t^{-1}\sum_{i=1}^t(X_i-\overline X_t)^2$.
Let $(\rho_t)_{t\in\mathbb N}$ be deterministic functions from
$[0,\infty)\times(0,1/2)$ to $[0,\infty)$, and define the stopping rule
\begin{equation}
T := \inf\left\{
t\in\mathbb N:
\overline X_t+\rho_t(\widehat v_t,\delta)\le \tau
\right\},
\label{eq:variance-certification-rule}
\end{equation}
with the convention $\inf\varnothing=+\infty$.

\begin{corollary}
\label{cor:silent-cliff}
Suppose the stopping rule in~\eqref{eq:variance-certification-rule}
is $\delta$-safe. Then, for every $n\in\mathbb N$,
\begin{equation}
\rho_n(0,\delta)
\ge
\min\left\{\Delta_{\max},\frac{(1-\tau)(1-\delta^{1/n})}{1+\delta^{1/n}}\right\}.
\label{eq:zero-variance-floor-body}
\end{equation}
In particular, for every $n\in\mathbb N$ satisfying
$\log(1/\delta)/n\le\Delta_{\max}$,
$\rho_n(0,\delta)
\ge
\frac{(1-\tau)\log(1/\delta)}{4n}$.
\end{corollary}

\begin{proof}
Suppose~\eqref{eq:zero-variance-floor-body} fails at some $n$, and choose
\[
\rho_n(0,\delta)<\Delta<
\min\left\{\Delta_{\max},
\frac{(1-\tau)(1-\delta^{1/n})}{1+\delta^{1/n}}\right\}.
\]
Under $\nu_0$, all observations equal $\tau-\Delta$, so
$\overline X_n=\tau-\Delta$, $\widehat v_n=0$, and $T\le n$ almost surely.
Proposition~\ref{prop:silent-cliff} gives $(1-p_\Delta)^n\le\delta$.
However, the chosen upper bound on $\Delta$ is equivalent to
$p_\Delta<1-\delta^{1/n}$, a contradiction.

For the second claim, set $x=\log(1/\delta)/n\le\Delta_{\max}\le1/2$.
The inequalities $1-e^{-x}\ge x/2$ and $1+e^{-x}\le2$ give
$(1-\tau)(1-e^{-x})/(1+e^{-x})\ge(1-\tau)x/4$.
Also, $\Delta_{\max}\ge(1-\tau)x/4$. Substituting into
\eqref{eq:zero-variance-floor-body} proves the claim.
\end{proof}

For the lower confidence test, consider
$T=\inf\{t\in\mathbb N:\overline X_t-\rho_t(\widehat v_t,\delta)>\tau\}$
and suppose $\Prob_\nu(T<\infty)\le\delta$ for every distribution on $[0,1]$
with $\E_\nu[X]\le\tau$. Reflect the observations to $1-X$ and the
threshold to $1-\tau$. The empirical variance is unchanged.
The same constant-sample argument applies, since
$\rho_n(0,\delta)<\Delta$ ensures that the strict test fires. It gives
\begin{equation}
\rho_n(0,\delta)
\ge\min\left\{\Delta_{\max},
\frac{\tau(1-\delta^{1/n})}{1+\delta^{1/n}}\right\}.
\label{eq:lcb-variance-floor}
\end{equation}
In particular, $\rho_n(0,\delta)\ge\tau\log(1/\delta)/(4n)$ under the
same conditions on $n$ as in the corollary. Thus uniformly valid
action removal requires a positive radius at zero empirical variance.

\section{Safe-Subgraph Variance}
\label{app:variance-topology-tightness}

Throughout this section, Assumption~\ref{ass:init} holds.
We use the spans $S^\pi$ and $S^\star$ from Section~\ref{subsec:branching}
to bound the transition variances of safe-policy values and optimal safe
values, respectively. The branching depth bounds the number of nonzero
terms along a safe path. We then apply the resulting variance bounds to regret and provide matching cost constructions showing that the bounds are tight.

\subsection{Proof of Theorem~\ref{thm:variance-topology}}
\label{app:variance-topology-proof}
\label{app:variance-topology}

\begin{proof}
Fix $\pi\in\Pi_{\mathrm{safe}}$. Its trajectory is almost surely a safe
path in $\cG_{\mathrm{safe}}$, and every visited vertex belongs to
$\cV^\pi$. For a distribution $q$ and a bounded function $f$, let
$m:=\min_{x\in\supp(q)}f(x)$ and
$M:=\max_{x\in\supp(q)}f(x)$. Since $(f-m)(M-f)\ge0$ on
$\supp(q)$,
\begin{equation}
\Var(q,f)
\le (M-qf)(qf-m)
\le \frac14(M-m)^2.
\label{eq:popoviciu-direct}
\end{equation}
Apply~\eqref{eq:popoviciu-direct} with
$q=P_h(\cdot\mid s_h,a_h)$ and $f=V_{h+1}^\pi$. A deterministic
transition has zero variance. Otherwise, the definition of $S^\pi$
bounds the successor-value span by $S^\pi$. Hence, almost surely,
\begin{equation}
\Var(P_h(\cdot\mid s_h,a_h),V_{h+1}^\pi)
\le \frac14(S^\pi)^2
\Ind\!\left\{|\supp(P_h(\cdot\mid s_h,a_h))|>1\right\}.
\label{eq:one-step-topology-bound}
\end{equation}
A safe path contains at most $B(\cG_{\mathrm{safe}})$ stochastic
transitions. Summing~\eqref{eq:one-step-topology-bound} along the path
and taking expectation gives
\begin{equation}
\E^\pi\!\left[\sum_{h=1}^H
\Var(P_h(\cdot\mid s_h,a_h),V_{h+1}^\pi)
\,\middle|\,s_1\right]
\le \frac14(S^\pi)^2B(\cG_{\mathrm{safe}}).
\label{eq:safe-graph-variance-bound}
\end{equation}
For the second claim, fix a safe path
$\gamma=(s_1,a_1,\ldots,s_H,a_H,s_{H+1})$ and a step $h\in[H]$.
All successors of $a_h$ are viable, so $V_{h+1}^\star$ gives their
optimal safe continuation values by Lemma~\ref{lem:safe-comparator-publication}.
Its span on $\supp(P_h(\cdot\mid s_h,a_h))$ is at most $S^\star$
by definition. A deterministic transition has zero variance.
Thus~\eqref{eq:popoviciu-direct} gives
\begin{equation}
\Var(P_h(\cdot\mid s_h,a_h),V_{h+1}^\star)
\le \frac14(S^\star)^2
\Ind\!\left\{|\supp(P_h(\cdot\mid s_h,a_h))|>1\right\}.
\label{eq:optimal-one-step-topology}
\end{equation}
The steps before $\tau_k$ can be completed to a full safe path because
all successors of a safe action are viable. Each episode therefore
contains at most $B(\cG_{\mathrm{safe}})$ stochastic transitions before
$\tau_k$. Summing~\eqref{eq:optimal-one-step-topology} gives
\begin{equation}
\sum_{k=1}^K\sum_{h<\tau_k}
\Var(P_h(\cdot\mid s_h^k,a_h^k),V_{h+1}^\star)
\le \frac K4(S^\star)^2B(\cG_{\mathrm{safe}}),
\label{eq:optimal-trajectory-topology}
\end{equation}
which proves the second claim.
\end{proof}

\subsection{Proof of Corollary~\ref{cor:deterministic-topology-regret}}
\label{app:deterministic-topology-regret}

\begin{proof}
For deterministic rewards, the reward-variance terms in
\eqref{eq:trajectory-variance} vanish, so
\eqref{eq:optimal-trajectory-topology} gives $\Var_\Sigma\le\frac K4(S^\star)^2B(\cG_{\mathrm{safe}})$.
The total reward of every safe policy lies in $[0,H]$, so its variance
is at most $H^2/4$. Hence $\VarStar\le H^2/4$ by
\eqref{eq:safe-max-return-variance}. Combining these two bounds yields
\[
\min\{\Var_\Sigma,K\VarStar\}
\le\frac K4\min\{H^2,(S^\star)^2B(\cG_{\mathrm{safe}})\}.
\]
Substituting this inequality into Theorem~\ref{thm:regret} proves the
corollary on the same event, with probability at least $1-\delta$.
\end{proof}

\subsection{Proof of Proposition~\ref{prop:gadget-main-tightness}}
\label{app:n-gadget-tightness}
\label{subsec:gadget-chain}

To attain the variance bound, we keep the transition kernel and rewards
fixed and vary the costs that determine which gadget transitions are
safe. We first give the construction, then compute its branching depth,
successor-value span, and return variance. Figure~\ref{fig:gadget-chain}
shows the construction for $n=3$ with zero, one, two, and three open gadgets.

\begin{example}[Gadget chain]
\label{ex:gadget-chain}
Fix an integer $n\ge1$ and set $H=2n$. Gadget $i\in[n]$ starts at $s_i$ in
layer $h_i:=2i-1$, and $s_{n+1}:=t$ is the terminal state.
The transition kernel and deterministic rewards are fixed as follows.

At $s_i$, action $a_{\mathrm{det},i}$ transitions deterministically
to $s'_{A,i}$, while $a_{\mathrm{st},i}$ transitions uniformly to
$s'_{A,i}$ and $s'_{B,i}$. Both actions have reward zero.
At the next layer, action $\alpha_i$ at $s'_{A,i}$ has reward one,
and actions $\beta_{1,i},\beta_{2,i}$ at $s'_{B,i}$ have rewards one
and zero, respectively. All three actions transition deterministically
to $s_{i+1}$.

Take $\tau=1/2$, set $c_{h_i+1}(s'_{A,i},\alpha_i)=0$, and define
\begin{equation}
\begin{aligned}
&\bigl(
 c_{h_i}(s_i,a_{\mathrm{det},i}),
 c_{h_i}(s_i,a_{\mathrm{st},i}),
 c_{h_i+1}(s'_{B,i},\beta_{1,i}),
 c_{h_i+1}(s'_{B,i},\beta_{2,i})
\bigr)\\
&\hspace{3em}=
\begin{cases}
(0,0,1,1), & \text{if gadget $i$ is closed},\\
(1,0,1,0), & \text{if gadget $i$ is open}.
\end{cases}
\end{aligned}
\label{eq:gadget-costs}
\end{equation}
Every other state--action pair has cost one, reward zero, and a fixed
deterministic transition to the next layer. For a cost configuration $c$, let $b(c)\in\{0,\ldots,n\}$ denote the
number of open gadgets.
\end{example}

\begin{figure}[t]
\centering
\begin{tikzpicture}[
  >=Stealth,
  nd/.style={circle, draw, minimum size=11pt, inner sep=0pt, font=\tiny},
  gnd/.style={circle, draw=red!30, text=red!30, minimum size=11pt, inner sep=0pt, font=\tiny},
  tm/.style={circle, draw, fill=gray!20, minimum size=11pt, inner sep=0pt, font=\tiny},
  S/.style={->, thick, black},
  D/.style={->, dashed, red!40, thin},
  Sdot/.style={->, thick, black, dotted},
  lb/.style={font=\tiny, midway},
  row label/.style={font=\small\bfseries, anchor=east},
  var label/.style={font=\scriptsize, text=blue!70!black, anchor=west},
  xscale=1.05
]

\def\yU{0.55}
\def\yD{-0.55}

\begin{scope}[shift={(0, 3.3)}]
  \node[row label] at (-0.6, 0) {(a)};

  \node[nd]  (s1)  at (0,0)       {$s_1$};
  \node[nd]  (A1)  at (1.3,\yU)   {$s'_{\!A1}$};
  \node[gnd] (B1)  at (1.3,\yD)   {$s'_{\!B1}$};
  \node[nd]  (s2)  at (2.7,0)     {$s_2$};
  \draw[S] (s1) -- node[lb, above left] {$a_\mathrm{det}$} (A1);
  \draw[D] (s1) -- (B1);
  \draw[S] (A1) -- node[lb, above right] {\tiny$r\!=\!1$} (s2);
  \draw[D] (B1) -- (s2);

  \node[nd]  (A2)  at (4.0,\yU)   {$s'_{\!A2}$};
  \node[gnd] (B2)  at (4.0,\yD)   {$s'_{\!B2}$};
  \node[nd]  (s3)  at (5.4,0)     {$s_3$};
  \draw[S] (s2) -- node[lb, above left] {$a_\mathrm{det}$} (A2);
  \draw[D] (s2) -- (B2);
  \draw[S] (A2) -- node[lb, above right] {\tiny$r\!=\!1$} (s3);
  \draw[D] (B2) -- (s3);

  \node[nd]  (A3)  at (6.7,\yU)   {$s'_{\!A3}$};
  \node[gnd] (B3)  at (6.7,\yD)   {$s'_{\!B3}$};
  \node[tm]  (t)   at (8.1,0)     {$t$};
  \draw[S] (s3) -- node[lb, above left] {$a_\mathrm{det}$} (A3);
  \draw[D] (s3) -- (B3);
  \draw[S] (A3) -- node[lb, above right] {\tiny$r\!=\!1$} (t);
  \draw[D] (B3) -- (t);

  \node[var label] at (8.5, 0) {$\mathrm{Var}^\star = 0$};
\end{scope}

\begin{scope}[shift={(0, 1.1)}]
  \node[row label] at (-0.6, 0) {(b)};

  \node[nd] (s1) at (0,0)       {$s_1$};
  \node[nd] (A1) at (1.3,\yU)   {$s'_{\!A1}$};
  \node[nd] (B1) at (1.3,\yD)   {$s'_{\!B1}$};
  \node[nd] (s2) at (2.7,0)     {$s_2$};
  \draw[D]    (s1) -- node[lb, above left, text=red!40] {$a_\mathrm{det}$} (A1);
  \draw[S]    (s1) -- node[lb, below left] {$a_\mathrm{st}$} (B1);
  \draw[Sdot] (s1) to[out=50,in=180] (A1);
  \draw[S]    (A1) -- node[lb, above right] {\tiny$r\!=\!1$} (s2);
  \draw[D]    (B1) to[out=10,in=-150] node[lb, above, text=red!40] {\tiny$\beta_1$} (s2);
  \draw[S]    (B1) -- node[lb, below right] {\tiny$r\!=\!0$} (s2);

  \node[nd]  (A2)  at (4.0,\yU)   {$s'_{\!A2}$};
  \node[gnd] (B2)  at (4.0,\yD)   {$s'_{\!B2}$};
  \node[nd]  (s3)  at (5.4,0)     {$s_3$};
  \draw[S] (s2) -- node[lb, above left] {$a_\mathrm{det}$} (A2);
  \draw[D] (s2) -- (B2);
  \draw[S] (A2) -- node[lb, above right] {\tiny$r\!=\!1$} (s3);
  \draw[D] (B2) -- (s3);

  \node[nd]  (A3)  at (6.7,\yU)   {$s'_{\!A3}$};
  \node[gnd] (B3)  at (6.7,\yD)   {$s'_{\!B3}$};
  \node[tm]  (t)   at (8.1,0)     {$t$};
  \draw[S] (s3) -- node[lb, above left] {$a_\mathrm{det}$} (A3);
  \draw[D] (s3) -- (B3);
  \draw[S] (A3) -- node[lb, above right] {\tiny$r\!=\!1$} (t);
  \draw[D] (B3) -- (t);

  \node[var label] at (8.5, 0) {$\mathrm{Var}^\star = \frac{1}{4}$};
\end{scope}

\begin{scope}[shift={(0, -1.1)}]
  \node[row label] at (-0.6, 0) {(c)};

  \node[nd] (s1) at (0,0)       {$s_1$};
  \node[nd] (A1) at (1.3,\yU)   {$s'_{\!A1}$};
  \node[nd] (B1) at (1.3,\yD)   {$s'_{\!B1}$};
  \node[nd] (s2) at (2.7,0)     {$s_2$};
  \draw[D]    (s1) -- node[lb, above left, text=red!40] {$a_\mathrm{det}$} (A1);
  \draw[S]    (s1) -- node[lb, below left] {$a_\mathrm{st}$} (B1);
  \draw[Sdot] (s1) to[out=50,in=180] (A1);
  \draw[S]    (A1) -- node[lb, above right] {\tiny$r\!=\!1$} (s2);
  \draw[D]    (B1) to[out=10,in=-150] node[lb, above, text=red!40] {\tiny$\beta_1$} (s2);
  \draw[S]    (B1) -- node[lb, below right] {\tiny$r\!=\!0$} (s2);

  \node[nd] (A2) at (4.0,\yU)   {$s'_{\!A2}$};
  \node[nd] (B2) at (4.0,\yD)   {$s'_{\!B2}$};
  \node[nd] (s3) at (5.4,0)     {$s_3$};
  \draw[D]    (s2) -- node[lb, above left, text=red!40] {$a_\mathrm{det}$} (A2);
  \draw[S]    (s2) -- node[lb, below left] {$a_\mathrm{st}$} (B2);
  \draw[Sdot] (s2) to[out=50,in=180] (A2);
  \draw[S]    (A2) -- node[lb, above right] {\tiny$r\!=\!1$} (s3);
  \draw[D]    (B2) to[out=10,in=-150] node[lb, above, text=red!40] {\tiny$\beta_1$} (s3);
  \draw[S]    (B2) -- node[lb, below right] {\tiny$r\!=\!0$} (s3);

  \node[nd]  (A3)  at (6.7,\yU)   {$s'_{\!A3}$};
  \node[gnd] (B3)  at (6.7,\yD)   {$s'_{\!B3}$};
  \node[tm]  (t)   at (8.1,0)     {$t$};
  \draw[S] (s3) -- node[lb, above left] {$a_\mathrm{det}$} (A3);
  \draw[D] (s3) -- (B3);
  \draw[S] (A3) -- node[lb, above right] {\tiny$r\!=\!1$} (t);
  \draw[D] (B3) -- (t);

  \node[var label] at (8.5, 0) {$\mathrm{Var}^\star = \frac{1}{2}$};
\end{scope}

\begin{scope}[shift={(0, -3.3)}]
  \node[row label] at (-0.6, 0) {(d)};

  \node[nd] (s1) at (0,0)       {$s_1$};
  \node[nd] (A1) at (1.3,\yU)   {$s'_{\!A1}$};
  \node[nd] (B1) at (1.3,\yD)   {$s'_{\!B1}$};
  \node[nd] (s2) at (2.7,0)     {$s_2$};
  \draw[D]    (s1) -- node[lb, above left, text=red!40] {$a_\mathrm{det}$} (A1);
  \draw[S]    (s1) -- node[lb, below left] {$a_\mathrm{st}$} (B1);
  \draw[Sdot] (s1) to[out=50,in=180] (A1);
  \draw[S]    (A1) -- node[lb, above right] {\tiny$r\!=\!1$} (s2);
  \draw[D]    (B1) to[out=10,in=-150] node[lb, above, text=red!40] {\tiny$\beta_1$} (s2);
  \draw[S]    (B1) -- node[lb, below right] {\tiny$r\!=\!0$} (s2);

  \node[nd] (A2) at (4.0,\yU)   {$s'_{\!A2}$};
  \node[nd] (B2) at (4.0,\yD)   {$s'_{\!B2}$};
  \node[nd] (s3) at (5.4,0)     {$s_3$};
  \draw[D]    (s2) -- node[lb, above left, text=red!40] {$a_\mathrm{det}$} (A2);
  \draw[S]    (s2) -- node[lb, below left] {$a_\mathrm{st}$} (B2);
  \draw[Sdot] (s2) to[out=50,in=180] (A2);
  \draw[S]    (A2) -- node[lb, above right] {\tiny$r\!=\!1$} (s3);
  \draw[D]    (B2) to[out=10,in=-150] node[lb, above, text=red!40] {\tiny$\beta_1$} (s3);
  \draw[S]    (B2) -- node[lb, below right] {\tiny$r\!=\!0$} (s3);

  \node[nd] (A3) at (6.7,\yU)   {$s'_{\!A3}$};
  \node[nd] (B3) at (6.7,\yD)   {$s'_{\!B3}$};
  \node[tm] (t)  at (8.1,0)     {$t$};
  \draw[D]    (s3) -- node[lb, above left, text=red!40] {$a_\mathrm{det}$} (A3);
  \draw[S]    (s3) -- node[lb, below left] {$a_\mathrm{st}$} (B3);
  \draw[Sdot] (s3) to[out=50,in=180] (A3);
  \draw[S]    (A3) -- node[lb, above right] {\tiny$r\!=\!1$} (t);
  \draw[D]    (B3) to[out=10,in=-150] node[lb, above, text=red!40] {\tiny$\beta_1$} (t);
  \draw[S]    (B3) -- node[lb, below right] {\tiny$r\!=\!0$} (t);

  \node[var label] at (8.5, 0) {$\mathrm{Var}^\star = \frac{3}{4}$};
\end{scope}

\end{tikzpicture}
\caption{Four safe subgraph topologies on the same $H=6$ gadget-chain MDP (Example~\ref{ex:gadget-chain}). Each row shows the safe subgraph under a different cost function. Black solid edges are safe. Red dashed edges are unsafe. In a \emph{closed} gadget, both actions at $s'_{B,i}$ are unsafe, so $a_{\mathrm{st},i}$ fails the successor-viability condition and only $a_{\mathrm{det},i}$ remains. In an \emph{open} gadget, $a_{\mathrm{det},i}$ and $\beta_{1,i}$ are unsafe, leaving a successor-value gap of one by~\eqref{eq:gadget-successor-values}. Dotted black edges show the branch of $a_{\mathrm{st}}$ to $s'_{A,i}$. Gadget indices are omitted from the action labels.}
\label{fig:gadget-chain}
\end{figure}
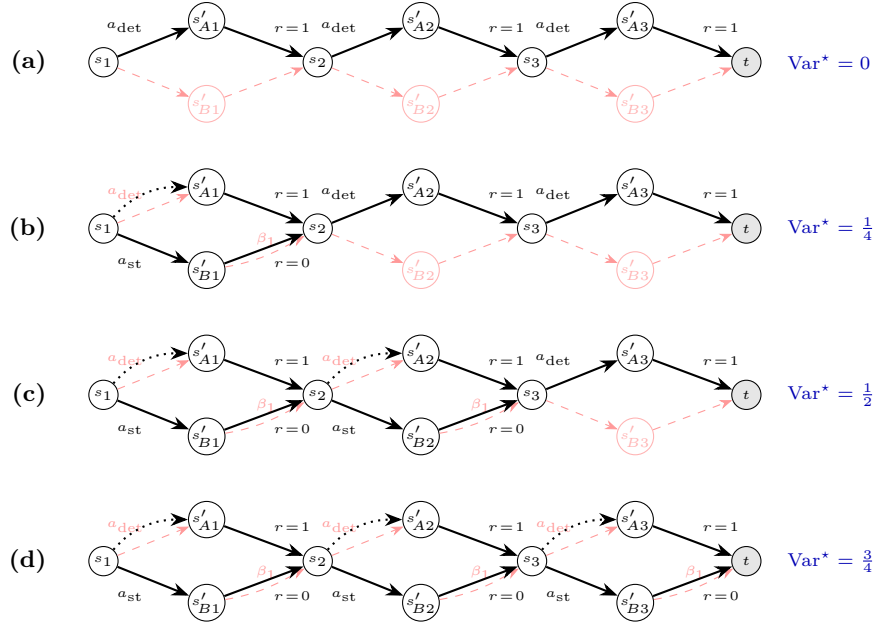

\begin{proof}[Proof of Proposition~\ref{prop:gadget-main-tightness}]
A backward induction from $s_{n+1}=t$ using \eqref{eq:gadget-costs} gives, for every $i\in[n]$,
\begin{equation}
\cA_{\mathrm{safe},h_i}(s_i)
=
\begin{cases}
\{a_{\mathrm{st},i}\}, & \text{if gadget $i$ is open},\\
\{a_{\mathrm{det},i}\}, & \text{if gadget $i$ is closed},
\end{cases}
\label{eq:gadget-safe-choice}
\end{equation}
and
\begin{equation}
\cA_{\mathrm{safe},h_i+1}(s'_{A,i})=\{\alpha_i\},
\qquad
\cA_{\mathrm{safe},h_i+1}(s'_{B,i})
=
\begin{cases}
\{\beta_{2,i}\}, & \text{if gadget $i$ is open},\\
\varnothing, & \text{if gadget $i$ is closed}.
\end{cases}
\label{eq:gadget-safe-reward}
\end{equation}
In a closed gadget, the state $s'_{B,i}$ is not viable. The
action $a_{\mathrm{st},i}$ satisfies the immediate cost constraint
but fails the successor-viability condition. In an open gadget, both successors of
$a_{\mathrm{st},i}$ are viable, while $a_{\mathrm{det},i}$ and $\beta_{1,i}$ violate the cost threshold.

It follows from \eqref{eq:gadget-safe-choice}--\eqref{eq:gadget-safe-reward} that all policies in
$\Pi_{\mathrm{safe}}$ select the same action at every state and step
reachable from $s_1$ under a safe policy.
Since both reward-layer branches merge at the next choice state,
\begin{equation}
\Prob^{\pi}(s_{h_i}=s_i\mid s_1)=1,
\qquad i\in[n].
\label{eq:gadget-choice-visited}
\end{equation}
The stochastic choice $(s_i,a_{\mathrm{st},i},h_i)$ is safe exactly
when gadget $i$ is open. These are the only stochastic safe transitions,
and every safe path visits all $n$ choice states. Hence
$B(\cG_{\mathrm{safe}})=b(c)$.

We next compute the successor-value spans and variances.
If gadget $i$ is closed, its safe choice action and its reward-layer transition are deterministic, so
its next-state value variance is zero. If gadget $i$ is open, then
\begin{equation}
V_{h_i+1}^{\pi}(s'_{A,i})
=
1+V_{h_i+2}^{\pi}(s_{i+1}),
\qquad
V_{h_i+1}^{\pi}(s'_{B,i})
=
V_{h_i+2}^{\pi}(s_{i+1}).
\label{eq:gadget-successor-values}
\end{equation}
The two successor values differ by one. All deterministic transitions
have successor-value span zero, so $S^\pi=\Ind\{b(c)>0\}$.
At each open gadget,
\begin{align*}
\Var\!\left(P_{h_i}(\cdot\mid s_i,a_{\mathrm{st},i}),V_{h_i+1}^{\pi}\right)
&=
\frac14\left(
V_{h_i+1}^{\pi}(s'_{A,i})-V_{h_i+1}^{\pi}(s'_{B,i})
\right)^2=
\frac14.
\end{align*}
All rewards are deterministic. Using~\eqref{eq:gadget-choice-visited}
and summing these next-state variances over the open gadgets gives
\begin{equation*}
\E^\pi\!\left[\sum_{h=1}^H
\Var(P_h(\cdot\mid s_h,a_h),V_{h+1}^\pi)
\,\middle|\,s_1\right]
=
\sum_{i:\,\text{gadget $i$ is open}}\frac14
=
\frac{b(c)}4.
\end{equation*}

These identities hold for every $\pi\in\Pi_{\mathrm{safe}}$.
Taking the supremum over this class gives $\VarStar=b(c)/4$.
The unique safe actions in~\eqref{eq:gadget-safe-choice}--\eqref{eq:gadget-safe-reward}
imply that these successor values coincide with $V^\star$.
Thus $S^{\star}=\Ind\{b(c)>0\}$ by definition.
Substitution into~\eqref{eq:safe-graph-variance-bound} proves equality.
For every $b\in\{0,\ldots,n\}$, choosing exactly $b$ open gadgets
gives the instance required by Proposition~\ref{prop:gadget-main-tightness}.
\end{proof}

\section{Lower Bounds}
\label{app:lower-bound}

Appendix~\ref{app:instance-lower-bound-proof} proves the regret lower
bound in Theorem~\ref{thm:instance-lower-bound}.
Appendix~\ref{app:witness-elimination} studies the sample requirements
for action elimination under additional conditions on the transition
structure and observations.

\subsection{Proof of Theorem~\ref{thm:instance-lower-bound}}
\label{app:instance-lower-bound-proof}

Assume the parameter conditions of Theorem~\ref{thm:instance-lower-bound}.
We use two families of CMDPs with the prescribed variance and safety
identification complexity. The first requires learning which action has
the better transition probability. The second requires identifying a safe
action at each of several state--step pairs. We prove a regret lower bound
for each family and then choose the instance giving the larger bound.
Learning algorithms follow the protocol in Section~\ref{sec:formulation},
with their policies chosen before each episode. For a fixed instance
$\cM$, the expectation $\E_{\cM}$ includes the algorithm's randomness
and its observations.

\paragraph{Common construction}
Set $n:=S-3$ and use the states $\{x,g,b,c_1,\ldots,c_n\}$, action set
$[A]$, initial state $x$, and threshold $\tau=1/2$.
Use $d:=\lfloor H/16\rfloor$ three-step gadgets in the first $3d$
steps, each of which starts at $x$ and returns to $x$.
The two families differ in these gadgets and in the costs and
transitions at a set of state--step pairs after them. The last
$R:=\lfloor H/8\rfloor$ steps provide the reward lost by an incorrect
choice in the second family. Since $4d+R+1\le H$, these parts fit
within the horizon.

Let $M:=\lfloor\Gamma/[8(A-1)]\rfloor$. The parameter assumptions give
$\Gamma/16\le M(A-1)\le\Gamma/8$ and $2\le M\le nd$, so we can choose $\mathcal J\subseteq\{3d+2,\ldots,4d+1\}\times[n]$ with $|\mathcal J|=M$.
Here $(h,s)\in\mathcal J$ denotes context state $c_s$ at step $h$.
We specify the transitions so that exactly one pair in $\mathcal J$
is reached in each episode, uniformly over the $M$ pairs.
For $h\in\{3d+2,\ldots,4d+1\}$, let
$M_h:=|\{(t,s)\in\mathcal J:t\ge h\}|$ and set $M_{4d+2}:=0$.
Whenever $M_h>0$, every action at $x$ at step $h-1$ satisfies
\[
P_{h-1}(c_s\mid x,a)=\frac{\Ind\{(h,s)\in\mathcal J\}}{M_h},
\qquad
P_{h-1}(x\mid x,a)=\frac{M_{h+1}}{M_h}.
\]
Every action at a pair in $\mathcal J$ moves to $g$ or $b$, as specified
below, and the process then remains in that state until the episode ends.
Since $\mu_{3d+1}^{\pi}(x)=1=M_{3d+2}/M$, induction gives
\begin{equation}
\mu_{h-1}^{\pi}(x)=\frac{M_h}{M},
\qquad
\mu_h^{\pi}(c_s)=\frac1M,
\quad (h,s)\in\mathcal J.
\label{eq:lb-query-occupancy}
\end{equation}
Thus the reached pair is uniform on $\mathcal J$, conditional on the
past and independent of the chosen actions. All rewards and costs
are deterministic. Unspecified rewards and costs are zero, and
unspecified transitions are self-loops.

Write $D:=nd(A-1)$ and $m:=\min\{K,\Gamma\}$.
We will use the bounds $n\ge5S/8$, $d\ge H/20$, and
$SAH/64\le D\le SAH/16$, which follow from the parameter assumptions.

\paragraph{Variance term}
The first family places an unknown best action at each context state
in each gadget. These actions determine the probabilities of receiving
the gadget rewards; the costs at the pairs in $\mathcal J$ determine
the safety identification complexity.

\begingroup
\SVAEtheoremheadingtarget
\begin{lemma}\label{lem:lb-variance-instance}
Assume the parameter conditions of
Theorem~\ref{thm:instance-lower-bound}. For every learning algorithm
following the protocol in Section~\ref{sec:formulation}, there exists
a CMDP $\cM_v$ with $S$ states, $A$ actions, and horizon $H$ satisfying
Assumption~\ref{ass:init},
$\frac{3v}{4}\le\VarStar(\cM_v)\le v$, and
$\frac{\Gamma}{4}\le \cC(\cM_v)\le\frac{\Gamma}{2}$.
On this instance, every realized interaction sequence satisfies
\begin{equation}
\frac{31Kv}{32}\le\Var_\Sigma\le Kv.
\label{eq:lb-variance-trajectory}
\end{equation}
The algorithm's expected regret on $\cM_v$ satisfies
\begin{equation}
\E_{\cM_v}[\cR(K)]
\ge \frac{3}{32\sqrt2}\sqrt{nAdKv}.
\label{eq:lb-variance-regret}
\end{equation}
\end{lemma}
\endgroup

\begin{proof}
\emph{Construction.}
Set $\varepsilon:=\frac{1}{8\sqrt2}\sqrt{nA/K}$ and
$\rho:=2\sqrt{v/d}$. Since $K\ge nA$ and $v\le d/4$, we have
$0<\varepsilon\le1/(8\sqrt2)$ and $0\le\rho\le1$.
Choose $(\omega_{i,s})_{(i,s)\in[d]\times[n]}$ independently and
uniformly from $[A]$. In gadget $i$, at step $3i-2$, every action at
$x$ transitions uniformly to one of $c_1,\ldots,c_n$. At step $3i-1$,
\[
P_{3i-1}(g\mid c_s,a)
=\frac12+\varepsilon\Ind\{a=\omega_{i,s}\},
\qquad
P_{3i-1}(b\mid c_s,a)=1-P_{3i-1}(g\mid c_s,a).
\]
At step $3i$, every action yields reward $\rho$ at $g$ and zero at
$b$, and both states return to $x$. All costs in these gadgets are zero.

At every pair in $\mathcal J$, action $1$ has cost zero and all other
actions have cost one. Every action there moves to $g$, and all
subsequent rewards are zero. A policy choosing action $1$ at every
such pair is safe, so Assumption~\ref{ass:init} holds. The transition
probabilities and rewards lie in $[0,1]$ by the bounds on
$\varepsilon$ and $\rho$.

\emph{Instance properties.}
Fix $\pi\in\Pi_{\mathrm{safe}}$ and gadget $i$, and let
$\alpha_i:=n^{-1}\sum_{s=1}^{n}
\Ind\{\pi_{3i-1}(c_s)=\omega_{i,s}\}$.
Each gadget returns to $x$, so its contribution to the return variance
is the next-state value variance at $x$ plus the average over its
context states. These
two terms satisfy
\begin{align}
&\Var\!\left(P_{3i-2}(\cdot\mid x,\pi_{3i-2}(x)),
V_{3i-1}^{\pi}\right)
+\frac1n\sum_{s=1}^{n}
\Var\!\left(P_{3i-1}(\cdot\mid c_s,\pi_{3i-1}(c_s)),
V_{3i}^{\pi}\right)
\notag\\
&\qquad
=\varepsilon^2\rho^2\alpha_i(1-\alpha_i)
+\frac{\rho^2}{4}-\varepsilon^2\rho^2\alpha_i
\notag\\
&\qquad
=\frac{\rho^2}{4}-\varepsilon^2\rho^2\alpha_i^2
\in\left[\frac{3\rho^2}{16},\frac{\rho^2}{4}\right].
\label{eq:lb-gadget-variance}
\end{align}
The first equality is the variance decomposition across the two
transitions. The last inclusion uses $\varepsilon^2\le1/128$. Summing
\eqref{eq:lb-gadget-variance} over $i$ and using $d\rho^2/4=v$ gives
$3v/4\le\VarStar(\cM_v)\le v$. The later steps contribute no
variance because their rewards are zero.

The only unsafe actions are the $M(A-1)$ actions with cost one at
pairs in $\mathcal J$. All these pairs are reachable under a safe
policy, and each such action has cost gap $1/2$. Therefore,
$\cC(\cM_v)=4M(A-1)\in[\Gamma/4,\Gamma/2]$.

We next compute $\Var_\Sigma$. Under the optimal safe value function,
$V_{3i-1}^\star(c_s)=(1/2+\varepsilon)\rho+V_{3i+1}^\star(x)$ for every
$s\in[n]$, so the transition from $x$ to the context states has zero
next-state value variance. In episode $k$, let $s_{i,k}$ be the
visited context index and $a_{i,k}$ the selected action. Write
$p_{i,k}:=1/2+\varepsilon\Ind\{a_{i,k}=\omega_{i,s_{i,k}}\}$.
The two successor values differ by $\rho$, so the context transition
contributes $p_{i,k}(1-p_{i,k})\rho^2$. These gadgets are safe and occur
before any pair in $\mathcal J$ is reached, so their contributions are included in
$\Var_\Sigma$. All rewards are deterministic, and the remaining transitions
before $\tau_k$ are deterministic or connect states with the same optimal
safe value. Hence, pathwise,
\[
\Var_\Sigma
=\sum_{k=1}^K\sum_{i=1}^d p_{i,k}(1-p_{i,k})\rho^2.
\]
Since $p_{i,k}\in\{1/2,1/2+\varepsilon\}$,
$(1-4\varepsilon^2)Kv\le\Var_\Sigma\le Kv$. The bound
$\varepsilon^2\le1/128$ proves~\eqref{eq:lb-variance-trajectory}.

\emph{Regret.}
Each choice different from $\omega_{i,s}$ loses $\varepsilon\rho$ in
expected reward. Since every gadget returns to $x$, these losses add
over gadgets and episodes. We lower bound the number of such choices
by averaging over the hidden actions.

Fix $(i,s)$ and fix all hidden actions except $\omega_{i,s}$.
Let $\Prob_a$ be the interaction distribution when $\omega_{i,s}=a$,
and let $\Prob_0$ be the reference distribution in which every action
at this coordinate reaches $g$ with probability $1/2$.
All other transition, reward, and cost distributions are unchanged.
Write $\E_a$ and $\E_0$ for the corresponding expectations, and define
\[
N_a(k):=\sum_{\ell=1}^{k}
\Ind\{s_{3i-1}^{\ell}=c_s,\ a_{3i-1}^{\ell}=a\}.
\]
Let $\Prob_a^{k-1}$ and $\Prob_0^{k-1}$ denote the distributions of
the history through episode $k-1$, including the algorithm's random
choices. Write $\operatorname{KL}$ for the Kullback--Leibler divergence.
The two instances differ only in the transition following action $a$
at the fixed coordinate, and
the KL divergence of one such observation is
$-\frac12\log(1-4\varepsilon^2)\le4\varepsilon^2$. The adaptive chain rule gives
\[
\operatorname{KL}(\Prob_0^{k-1}\Vert\Prob_a^{k-1})
\le4\varepsilon^2\E_0[N_a(k-1)].
\]
Conditional on the history before episode $k$, the probability of choosing
$a$ at $c_s$ is determined by the algorithm's policy distribution. The
probability of reaching $c_s$ is $1/n$, independently of this choice.
Applying Pinsker's inequality to the history distributions therefore gives
\[
\Prob_a\{s_{3i-1}^{k}=c_s,\ a_{3i-1}^{k}=a\}
\le
\Prob_0\{s_{3i-1}^{k}=c_s,\ a_{3i-1}^{k}=a\}
+\frac{\sqrt2\varepsilon}{n}
\sqrt{\E_0[N_a(k-1)]}.
\]
Summing over $k$ and averaging over $a$ yields
\begin{align}
\frac1A\sum_{a=1}^{A}\E_a[N_a(K)]
&\le\frac{K}{nA}
+\frac{K\sqrt2\varepsilon}{nA}
\sum_{a=1}^{A}\sqrt{\E_0[N_a(K)]}
\notag\\
&\le\frac{K}{nA}
+\frac{K\sqrt2\varepsilon}{n}\sqrt{\frac{K}{nA}}
\le\frac{5K}{8n}.
\label{eq:lb-action-count}
\end{align}
The second inequality uses Cauchy--Schwarz and
$\sum_a\E_0[N_a(K)]=K/n$.  The last inequality uses the definition of
$\varepsilon$ and $A\ge2$.

By~\eqref{eq:lb-action-count}, averaging over all hidden actions and
summing over the $nd$ coordinates gives at least $3dK/8$ incorrect
choices in expectation. Multiplying by the loss $\varepsilon\rho$ gives
\[
\E[\cR(K)]
\ge\frac{3dK}{8}\varepsilon\rho
=\frac{3}{32\sqrt2}\sqrt{nAdKv}.
\]
The expectation here also averages over the hidden actions. Some fixed
choice therefore satisfies~\eqref{eq:lb-variance-regret}. The variance
and safety identification complexity bounds hold for every such choice.
\end{proof}

\paragraph{Safety identification term}
The second family uses the same state--step pairs, but the unknown
action at each pair now determines safety. The initial gadgets supply
the target variance without contributing regret. An incorrect choice
at a pair in $\mathcal J$ loses the reward from the last $R$ steps.

\begin{lemma}\label{lem:lb-witness-instance}
Assume the parameter conditions of
Theorem~\ref{thm:instance-lower-bound}. For every learning algorithm
following the protocol in Section~\ref{sec:formulation}, there exists
a CMDP $\cM_w$ with $S$ states, $A$ actions, and horizon $H$ satisfying
Assumption~\ref{ass:init}, $\VarStar(\cM_w)=v$, and
$\frac{\Gamma}{4}\le \cC(\cM_w)\le\frac{\Gamma}{2}$.
On this instance, $\Var_\Sigma=Kv$ for every realized interaction
sequence. The algorithm's expected regret on $\cM_w$ satisfies
\begin{equation}
\E_{\cM_w}[\cR(K)]\ge\frac1{512}\sqrt{Dm}.
\label{eq:lb-witness-regret}
\end{equation}
\end{lemma}

\begin{proof}
\emph{Construction.}
Use $d$ action-independent gadgets in the first $3d$ steps.
In gadget $i$, every action at $x$ moves to $c_1$ at step $3i-2$,
and every action at $c_1$ moves to $g$ or $b$, each with probability
$1/2$, at step $3i-1$. At step $3i$, the rewards are
$\rho:=2\sqrt{v/d}$ at $g$ and zero at $b$, and both states return to
$x$. These rewards lie in $[0,1]$ because $v\le d/4$.

Choose $(\theta_j)_{j\in\mathcal J}$ independently and uniformly from
$[A]$. At pair $j$, action $\theta_j$ has cost zero and moves to $g$.
Every other action has cost one and moves to $b$. Both states have
zero reward until the last $R$ steps. At each of those steps, the reward
is $\Delta_r/R$ at $g$ and zero at $b$, where $\Delta_r:=\sqrt{D/m}$.
The bounds $R\ge H/10$, $m\ge16SA/H$, and $D\le SAH/16$ imply
$D\le R^2m$, so $\Delta_r/R\le1$.
A policy choosing $\theta_j$ at every pair $j$ is safe, which proves
Assumption~\ref{ass:init}.

\emph{Instance properties.}
The initial gadgets contribute $d\rho^2/4=v$ to the return variance
of every policy. Under every safe policy, the value is $\Delta_r$ at
each pair in $\mathcal J$ and at $x$ whenever $x$ is reached after the
gadgets. The action chosen at each pair moves deterministically to
$g$. Thus the later steps contribute zero variance,
and $\VarStar(\cM_w)=v$.
The only unsafe actions are the $M(A-1)$ incorrect actions at these
pairs. Every pair is reachable under a safe policy, and each incorrect
action has cost gap $1/2$, so $\cC(\cM_w)=4M(A-1)
\in[\Gamma/4,\Gamma/2]$.

The initial gadgets lie in the true safe graph and occur before every
pair in $\mathcal J$. Each contributes $\rho^2/4=v/d$ to
$\Var_\Sigma$ in every episode. All rewards are deterministic.
Under $V^\star$, every pair in $\mathcal J$ has continuation value
$\Delta_r$, while safe actions at these pairs and all subsequent
transitions are deterministic. Hence $\Var_\Sigma=Kv$ for every
realized interaction sequence.

\emph{Regret.}
An incorrect action at a pair in $\mathcal J$ loses $\Delta_r$ in
reward. We bound the number of incorrect choices, first averaging
over the hidden actions. Let $T_j(t)$ count visits to pair $j$ in the
first $t$ episodes, set $T_j:=T_j(K)$, and let $W_j$ count choices
different from $\theta_j$ over the $K$ episodes.

Fix the sequence of reached pairs, all other hidden actions, the
algorithm's random seed, and the transition randomness in the initial
gadgets. Consider the actions tested at $j$ along the hypothetical
history in which each test there has cost one and moves to $b$.
List the distinct actions in the order in which they are first tested, stopping once all
$A$ actions have been listed or the visits end. Append any untested
actions in a fixed order. This ordering is independent of $\theta_j$,
and the actual history agrees with the hypothetical one until
$\theta_j$ is first tested. Also, $T_j$ depends only on the sequence
of reached pairs. Conditional on $T_j=q$, the rank of $\theta_j$ is
therefore uniform on $[A]$. If that rank is $r$, at least
$\min\{q,r-1\}$ choices are incorrect; repeated tests cannot reduce
this number. Averaging over the rank gives
\begin{equation}
\E[W_j\mid T_j=q]
\ge\frac1A\sum_{r=1}^{A}\min\{q,r-1\}
\ge\frac14\min\{q,A-1\}.
\label{eq:lb-rank-bound}
\end{equation}

It remains to bound the total number of visits before individual
counts reach $A-1$. Let $J_k$ be the pair reached in episode $k$.
Pathwise,
\[
\sum_{j\in\mathcal J}\min\{T_j,A-1\}
=\sum_{k=1}^{K}\Ind\{T_{J_k}(k-1)<A-1\}.
\]
By~\eqref{eq:lb-query-occupancy}, $J_k$ is uniform on $\mathcal J$
conditional on the past. Since $\sum_j T_j(k-1)=k-1$, we have
$\E[T_{J_k}(k-1)]=(k-1)/M$.
For $k\le\min\{K,\lfloor M(A-1)/2\rfloor\}$, Markov's inequality
therefore bounds the expectation of each corresponding indicator
from below by $1/2$. Since $M(A-1)\ge2$,
$\lfloor M(A-1)/2\rfloor\ge M(A-1)/4$, and hence
\begin{equation}
\E\!\left[\sum_{j\in\mathcal J}\min\{T_j,A-1\}\right]
\ge\frac18\min\{K,M(A-1)\}
\ge\frac{m}{128}.
\label{eq:lb-query-count}
\end{equation}
The last inequality uses $M(A-1)\ge\Gamma/16$ and
$\min\{K,M(A-1)\}\ge m/16$. Combining
\eqref{eq:lb-rank-bound} and \eqref{eq:lb-query-count} with the loss
$\Delta_r$ per incorrect choice gives
\[
\E[\cR(K)]
=\Delta_r\sum_{j\in\mathcal J}\E[W_j]
\ge\frac{\Delta_r m}{512}
=\frac1{512}\sqrt{Dm}.
\]
The expectation also averages over the hidden actions, so some fixed
choice satisfies~\eqref{eq:lb-witness-regret}. The variance and
safety identification complexity bounds hold for every such choice.
\end{proof}

\begin{proof}[Proof of Theorem~\ref{thm:instance-lower-bound}]
By Lemmas~\ref{lem:lb-variance-instance} and
\ref{lem:lb-witness-instance}, both instances lie in the required ranges
of $\VarStar$ and $\cC$. The instance $\cM_v$ satisfies
$31Kv/32\le\Var_\Sigma\le Kv$ for every realized interaction sequence,
while $\cM_w$ satisfies $\Var_\Sigma=Kv$. Thus either
instance satisfies the stated range for $\Var_\Sigma$.
Choose $\cM=\cM_v$ when the right-hand side of
\eqref{eq:lb-variance-regret} is at least that of
\eqref{eq:lb-witness-regret}, and choose $\cM=\cM_w$ otherwise. Then
\begin{align*}
\E_{\cM}[\cR(K)]
&\ge
\max\!\left\{
\frac{3}{32\sqrt2}\sqrt{nAdKv},
\frac1{512}\sqrt{Dm}
\right\}\\
&=\Omega\!\left(
\sqrt{SAHKv}+\sqrt{SAH\min\{K,\Gamma\}}
\right).
\end{align*}
The assumptions of Theorem~\ref{thm:instance-lower-bound} give
$n\ge5S/8$, $d\ge H/20$, and $D\ge SAH/64$. Substituting these bounds
and using $\max\{x,y\}\ge(x+y)/2$ gives the final line.
\end{proof}

\subsection{Lower Bounds for Action Elimination}
\label{app:witness-elimination}

We next prove the elimination lower bound discussed in
Section~\ref{subsec:safety-identification}.
The learner must distinguish an unsafe root from alternatives in which
changing one leaf makes it safe. This gives a lower bound on the
required leaf observations, which we convert into root executions.

\paragraph{Instances and alternatives}

Fix $\delta\in(0,1/4]$ and $\tau=1/2$, and consider an instance $\cM$
with an unsafe root $u=(s,a,h)$. Use the witness-tree representation in
Lemma~\ref{lem:witness-tree-publication}. Let $\mathcal L_u$ be its leaves and,
for $\ell\in\mathcal L_u$, let $\Delta_\ell$ be its cost gap and
$q_\ell>0$ the product of the transition probabilities on its path.
Assume the following conditions on the tree and its observations.
\begin{enumerate}[label=(\roman*),ref=\roman*,leftmargin=*]
\item Every nonterminal action satisfies the immediate cost constraint
and has exactly one non-viable successor. All its other successors are
viable. The root state has another action in
$\cA_{\mathrm{safe},h}(s)$.
\item Each nonroot state--step pair in the tree can be reached only
through its designated parent action. Different branches use disjoint
state--action pairs at each step.
\item At every leaf $\ell$, $0<\Delta_\ell\le1/4$, all successors are
viable, and the cost has distribution $\text{Bernoulli}(1/2+\Delta_\ell)$.
Define $\cM^\ell$ by replacing this cost distribution with
$\text{Bernoulli}(1/2-\Delta_\ell)$ and leaving all other observations
unchanged.
\item In $\cM$ and $\cM^\ell$, conditional on the leaf state--action
pair, the cost is independent of the pair consisting of the reward and
next state. The reward and next state have the same joint distribution
in both instances.
\end{enumerate}
By Lemma~\ref{lem:witness-tree-publication},
\begin{equation}
 \kappa_h(s,a)
 =\sum_{\ell\in\mathcal L_u}\frac{1}{q_\ell\Delta_\ell^2}.
\label{eq:private-witness-expansion}
\end{equation}
The leaf action is safe in $\cM^\ell$. At its parent, this provides a
safe action at the previously non-viable successor, while all other
successors were already viable. Repeating this argument makes every
ancestor action, including $u$, safe.

\paragraph{Elimination rules}
The learner chooses its policy before each episode, as in
Section~\ref{sec:formulation}. An elimination rule removes $u$ at an
episode boundary and never executes it
afterward. Let $\sigma_u\in\{0,1,\ldots\}\cup\{\infty\}$ be the number of
completed episodes when it is removed, and let $N_u(t)$ count its executions
through the end of episode $t$. Thus $\{\sigma_u\ge k\}$ is known before
episode $k$ begins.  Write
$\operatorname{kl}(p,q):=p\log(p/q)+(1-p)\log((1-p)/(1-q))$.
Terms with zero numerator are zero, and a positive numerator divided by
zero contributes $+\infty$. The rule is
$\delta$-valid at $u$ if
\begin{equation}
\begin{aligned}
 \Prob_{\cM}(\sigma_u<\infty)&\ge1-\delta,\\
 \Prob_{\cM^\ell}(\sigma_u<\infty)&\le\delta
 \quad(\ell\in\mathcal L_u).
\end{aligned}
\label{eq:root-elimination-validity}
\end{equation}

\begin{lemma}[Elimination lower bound]
\label{lem:recursive-elimination-lower}
Consider the instance $\cM$ and alternatives $\cM^\ell$,
$\ell\in\mathcal L_u$, satisfying conditions (i)--(iv) above.
Every $\delta$-valid elimination rule at $u$, as defined
in~\eqref{eq:root-elimination-validity}, satisfies
\begin{equation}
 \E_{\cM}[N_u(\sigma_u)]
 \ge\frac{1}{32}\kappa_h(s,a)\log\frac1{2\delta}.
\label{eq:lb-root-execution}
\end{equation}
\end{lemma}

\begin{proof}
\emph{Leaf observations.}
Fix $\ell\in\mathcal L_u$, and let $N_\ell(t)$ count observations of
its cost through episode $t$. By condition (iv), the KL divergence of
the full observation at this leaf equals that of its Bernoulli cost.
For $0<\Delta_\ell\le1/4$,
\[
\operatorname{kl}(1/2+\Delta_\ell,1/2-\Delta_\ell)
=2\Delta_\ell\log\frac{1+2\Delta_\ell}{1-2\Delta_\ell}
\le16\Delta_\ell^2,
\]
where the inequality uses $\log((1+x)/(1-x))\le4x$ for $x\in[0,1/2]$.

Let $E_n:=\{\sigma_u\le n\}$, and stop the interaction history at the
end of episode $\sigma_u\wedge n$. The stopped history, including the
elimination decision, determines $E_n$. Since only the cost at leaf
$\ell$ differs between the instances, the chain rule expresses the KL
divergence of the stopped histories as the expected leaf count times
the single-observation divergence. Data processing then gives
\[
 16\Delta_\ell^2\E_{\cM}[N_\ell(\sigma_u\wedge n)]
 \ge
 \operatorname{kl}\!\left(\Prob_{\cM}(E_n),
 \Prob_{\cM^\ell}(E_n)\right).
\]
Set $p:=\Prob_{\cM}(\sigma_u<\infty)$ and
$q:=\Prob_{\cM^\ell}(\sigma_u<\infty)$. Letting $n$ tend to infinity,
monotone convergence and lower semicontinuity give
\begin{equation}
 16\Delta_\ell^2\E_{\cM}[N_\ell(\sigma_u)]
 \ge \operatorname{kl}(p,q)
 \ge \operatorname{kl}(1-\delta,\delta),
\label{eq:lb-leaf-kl}
\end{equation}
where the last inequality uses $p\ge1-\delta>\delta\ge q$ and the
monotonicity of binary KL.

\emph{Root executions.}
For each execution of $u$, the fixed episode policy determines a path
through the witness tree before any transition along that path is
observed. For $\ell\in\mathcal L_u$, let $T_\ell(t)$ count executions
assigned to the path ending at $\ell$ through episode $t$.
For every $n$,
\begin{align*}
 \E_{\cM}[N_\ell(\sigma_u\wedge n)]
 &=q_\ell\E_{\cM}[T_\ell(\sigma_u\wedge n)],\\
 N_u(\sigma_u\wedge n)
 &=\sum_{\ell\in\mathcal L_u}T_\ell(\sigma_u\wedge n).
\end{align*}
To obtain the first identity, include episode $k$ exactly when
$\sigma_u\ge k$, an event known before that episode. Conditional on
an execution assigned to leaf $\ell$, all transitions on its path occur
with probability $q_\ell$. Condition (ii) ensures that every observation
at the leaf comes from this path. The second identity holds because
each root execution is assigned to exactly one leaf.
Letting $n$ tend to infinity and applying monotone convergence yields
\begin{align*}
 \E_{\cM}[N_u(\sigma_u)]
 &=\sum_{\ell\in\mathcal L_u}
 \frac{\E_{\cM}[N_\ell(\sigma_u)]}{q_\ell}\\
 &\ge\frac{1}{16}\operatorname{kl}(1-\delta,\delta)
 \sum_{\ell\in\mathcal L_u}\frac{1}{q_\ell\Delta_\ell^2}\\
 &=\frac{1}{16}\kappa_h(s,a)\operatorname{kl}(1-\delta,\delta).
\end{align*}
The inequality uses~\eqref{eq:lb-leaf-kl}, and the last equality uses
\eqref{eq:private-witness-expansion}. Finally, since $\delta\le1/4$,
\[
\operatorname{kl}(1-\delta,\delta)
=(1-2\delta)\log\frac{1-\delta}{\delta}
\ge\frac12\log\frac1{2\delta}.
\]
Substitution proves~\eqref{eq:lb-root-execution}.
\end{proof}

\begin{samepage}
\begin{corollary}
\label{cor:private-forest-elimination}
Let $\mathcal U$ consist of all triples $(s,a,h)$ with $h\in[H]$,
$a\notin\cA_{\mathrm{safe},h}(s)$, and $(s,h)$ reachable under a safe
policy, as indexed in~\eqref{eq:C-wit}. Suppose these triples are the
roots of pairwise-disjoint witness trees satisfying conditions
(i)--(iv) above. For
$u\in\mathcal U$ and $\ell\in\mathcal L_u$, let $\cM^{u,\ell}$ be the corresponding
single-leaf alternative. For a learner using the elimination rules
defined above at all roots, suppose
$\Prob_{\cM}\!\left(\bigcap_{u\in\mathcal U}
 \{\sigma_u<\infty\}\right)\ge1-\delta$, and $\Prob_{\cM^{u,\ell}}(\sigma_u<\infty)\le\delta
 \quad(u\in\mathcal U, \ell\in\mathcal L_u)$.
Then
\[
 \E_{\cM}\!\left[\sum_{u\in\mathcal U}N_u(\sigma_u)\right]
 \ge\frac{1}{32} \cC(\cM)\log\frac1{2\delta}.
\]
\end{corollary}

\begin{proof}
The probability conditions imply the hypotheses of
Lemma~\ref{lem:recursive-elimination-lower} for every root. Applying
\eqref{eq:lb-root-execution} to every root and summing over
$u\in\mathcal U$ proves the claim because
$\sum_{u=(s,a,h)\in\mathcal U}\kappa_h(s,a)=\cC(\cM)$.
\end{proof}
\end{samepage}

\section{Experiments}\label{app:experiments}

In this section, we conduct experiments to illustrate our theoretical
findings and evaluate the performance of our algorithm. We focus on the following
four questions:
\begin{itemize}[leftmargin=0pt, label={}, itemindent=0pt,
                itemsep=3pt, topsep=3pt, parsep=0pt]
    \item \textbf{Q1:} \textit{How does SVAE compare with baselines
    in regret and step-wise violation?}

    \item \textbf{Q2:} \textit{Can SVAE exploit low return variance
    induced by safety constraints?}

    \item \textbf{Q3:} \textit{Can safety identification remain difficult
    when safe-policy return variance is zero?}

    \item \textbf{Q4:} \textit{Does recursive elimination exclude
    unsafe branches while preserving safe continuations?}
\end{itemize}
Appendix~\ref{app:experimental-setup} describes the experimental setup and baselines.
We evaluate regret and step-wise violation on stochastic navigation tasks in
Appendix~\ref{app:h-navigation} (Q1). We then study variance adaptation, safety identification complexity,
and recursive elimination using controlled CMDPs in
Appendices~\ref{app:h-variance}--\ref{app:h-propagation} (Q2--Q4).
Hyperparameter selection and implementation details are given in
Appendix~\ref{app:experimental-details}.

\raggedbottom
\renewcommand{\topfraction}{0.95}
\renewcommand{\bottomfraction}{0.95}
\renewcommand{\textfraction}{0.05}
\renewcommand{\floatpagefraction}{0.85}
\setlength{\textfloatsep}{12pt plus 2pt minus 2pt}
\setlength{\floatsep}{10pt plus 2pt minus 2pt}
\setlength{\intextsep}{10pt plus 2pt minus 2pt}
\subsection{Experimental Setup}
\label{app:experimental-setup}

\paragraph{Learning setting}
The reward means, mean costs, and transition probabilities are unknown. All algorithms learn from online interaction without a known safe policy. Rewards are Bernoulli in Q1 and deterministic in Q2--Q4, where we study the safe subgraph and safety identification without reward noise. The reward means remain fixed throughout learning. Costs are Bernoulli with mean $c_h(s,a)$, and next states follow the transition kernel. Estimates are maintained separately for each $(h,s,a)$. We set $\tau=0.6$, $\delta=0.05$, and run each algorithm for $K=10{,}000$ episodes.

\paragraph{Baselines}
For navigation (Q1), we compare SVAE with SUCBVI~\citep{xiong2023provably}. Both methods use the same recursive safety filter, so the comparison focuses on their planning rules. Our SUCBVI implementation estimates unknown rewards and costs separately for each $(h,s,a)$. We also include MVP-V~\citep{zhou2023sharp} and a variant with the recursive safety filter applied at every episode, denoted by \emph{MVP-V + filter}.

\paragraph{Ablations}
We consider two variants of SVAE with Hoeffding-type bonuses. \emph{SVAE (Hoeffding)} replaces the empirical reward and transition variances in~\eqref{eq:variance-bonus} with $1/4$ and $(H-h)^2/4$, respectively. \emph{SVAE (Hoeffding-P)} replaces only the transition variance, with P denoting the transition component. Both retain the finite-sample correction and recursive safety filter. We also consider \emph{SVAE w/o propagation}, which retains variance-adaptive planning but eliminates actions only through their own cost lower confidence bounds. Q2 compares SVAE with both bonus variants, and Q4 compares it with the variant without propagation. Q3 evaluates SVAE as the environment parameters vary.

\paragraph{Metrics}
We evaluate regret using exact policy values and the optimal safe value. Step-wise violation is the cumulative positive excess of mean costs over the threshold along sampled trajectories, as defined in Section~\ref{sec:formulation}. The true model and safe sets are used only for offline tuning and evaluation.
For navigation, we also report the strong regret~\citep{stradi2025strong} $\mathcal R_+(K):=\sum_{k=1}^K[V_1^\star(s_1)-V_1^{\pi^k}(s_1)]_+$, which does not allow cancellation across episodes. We track policy value and expected violation per episode using a moving average over up to 500 episodes.

We use the same observation streams and state/action relabelings to pair methods. For navigation (Q1), we average five paired seeds within each of eight independent environments, then average across environments. Each condition in Q2--Q4 uses 20 seeds. We report means with sample standard deviations (SDs) in parentheses. SDs are computed across environment means for Q1 and across seeds for Q2--Q4. Shading indicates one SD around the mean, with lower bounds clipped at zero for nonnegative metrics.

We tune all navigation methods with the same search budget on separate development and validation environments, then fix their hyperparameters before testing. SVAE and its ablations use the same planning and cost-confidence multipliers in the controlled experiments. Appendix~\ref{app:experimental-details} gives the tuning procedure and selected values.

\FloatBarrier

\subsection{Stochastic Navigation}
\label{app:experimental-results}
\label{app:h-navigation}

We first compare the regret and step-wise violation of SVAE and the baselines in a $5\times5$ gridworld with four movement actions and horizon $20$. The agent starts at the middle-left cell, and the middle-right cell is an absorbing goal. At each step, rewards are sampled independently from a Bernoulli distribution with mean $0.9$ at the goal and $0.1$ elsewhere. Reward noise is independent of costs and transitions. Each episode lasts $H$ steps, with the agent remaining at the goal after reaching it.

For each state--action pair, we sample the probability of the intended move uniformly from $[0.8,0.95]$ when generating the environment. With the remaining probability, the agent moves in either perpendicular direction with equal probability. Attempts to move outside the grid leave the agent in the same cell. An action has mean cost in $[0.7,0.85]$ if it can cross the boundary between columns 3 and 4 in rows 2--4. All other non-goal actions have mean costs in $[0.15,0.3]$, and the cost at the goal is zero.

\begin{figure}[!htb]
\centering
\includegraphics[width=.49\linewidth]{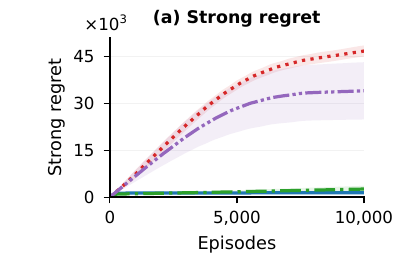}\hfill
\includegraphics[width=.49\linewidth]{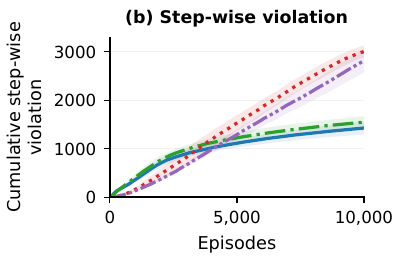}\\[3pt]
\includegraphics[width=.49\linewidth]{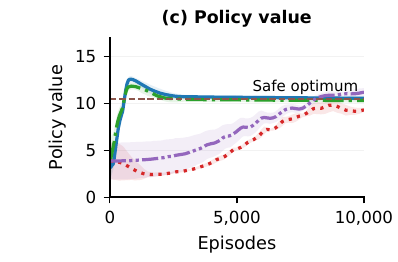}\hfill
\includegraphics[width=.49\linewidth]{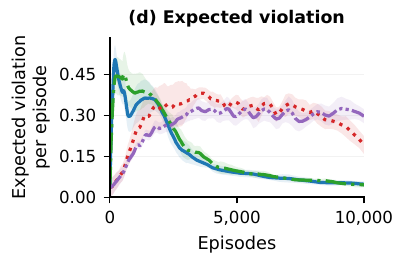}\\
\includegraphics[width=\linewidth]{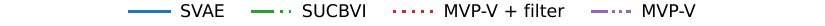}
\caption{\textbf{Regret and safety in stochastic navigation.} Panels (a) and (b) show strong regret and cumulative step-wise violation. Panels (c) and (d) show policy value and expected violation per episode, using a moving average over up to 500 episodes. At $k=0$, we evaluate the initial policy. Curves show means across eight environments after averaging five paired seeds per environment. Shaded areas show one SD across environment means. The horizontal line in (c) marks the mean optimal safe value.}
\label{fig:h-navigation}
\end{figure}

\paragraph{Results}
Compared with SUCBVI, SVAE reduces final strong regret by $44.9\%$ and cumulative step-wise violation by $7.7\%$ (Table~\ref{tab:h-navigation}). The mean (SD) paired differences, computed as SVAE minus SUCBVI, are $-1{,}162.23$ ($947.46$) for $\mathcal R_+$ and $-119.70$ ($45.10$) for violation.

\FloatBarrier
\begin{table}[!htb]
\centering\small
\caption{\textbf{Regret and step-wise violation in stochastic navigation.} We report cumulative regret, strong regret, and step-wise violation at $K=10{,}000$. Entries are means (SDs) across eight environments after averaging five paired seeds per environment. Bold numbers mark the lowest mean in each column, including ties at the displayed precision.}
\label{tab:h-navigation}
\setlength{\tabcolsep}{16pt}\renewcommand{\arraystretch}{1.13}
\begin{tabular}{@{}lrrr@{}}
\toprule
Method & $\mathcal R(K)/10^3$ & $\mathcal R_+(K)/10^3$ & $\mathrm{Viol}(K)/10^3$ \\
\midrule
SVAE & \textbf{-2.406 (0.947)} & \textbf{1.428 (0.112)} & \textbf{1.426 (0.132)} \\
SUCBVI & 0.150 (1.707) & 2.590 (1.037) & 1.546 (0.121) \\
MVP-V + filter & 45.989 (1.973) & 46.756 (1.776) & 3.005 (0.124) \\
MVP-V & 31.238 (9.735) & 34.020 (9.161) & 2.816 (0.246) \\
\bottomrule
\end{tabular}
\end{table}%
\FloatBarrier

SVAE's expected violation per episode decreases from $0.383$ in the first 500 episodes to $0.0474$ in the last 500 (Figure~\ref{fig:h-navigation}). Over the last 500 episodes, its mean policy value is $10.553$, above the average optimal safe value of $10.476$ while some constraint violation remains. Unconstrained MVP-V reaches a higher mean policy value of $11.206$ over the last 500 episodes, with greater expected violation of $0.2961$ per episode.

\subsection{Variance Adaptation}
\label{app:h-variance}

To study whether SVAE can exploit low return variance in an unknown safe subgraph, we use a two-step CMDP with eight states and four actions. The reward at the root is zero. Each of three root actions leads to two possible second-step states with probability $1/2$ each. The corresponding reward pairs are $(0.4,0.4)$, $(0.25,0.55)$, and $(0.1,0.7)$. The fourth root action gives a deterministic return of $0.5$.

All three branching actions are available in every instance. We keep rewards and transitions fixed and vary costs to select the safe branch. The selected branching action and the deterministic action have mean cost $0.2$, while the other two root actions have mean cost $0.8$. All second-step mean costs are $0.2$. Thus, the optimal safe value is $0.5$, the expected-return gap between the two safe root actions is $0.1$, and $\mathcal C=50$ in every instance.

Let $u\in\{0,0.15,0.30\}$ denote half the reward difference on the safe branch. Then $S^\star=2u$, $\mathrm{Var}^{\star}=u^2$, and $B(\cG_{\mathrm{safe}})=1$. The safe branching policy attains the variance bound in Theorem~\ref{thm:variance-topology}. When $u=0$, both successor states have the same value, so the transition variance is zero despite the random next state. This is the zero-span case of Corollary~\ref{cor:deterministic-topology-regret}.

\begin{figure}[!htb]
\centering
\includegraphics[width=.329\linewidth]{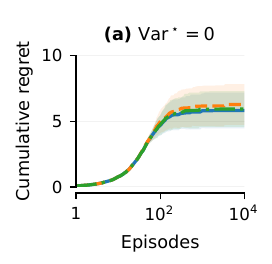}\hfill
\includegraphics[width=.329\linewidth]{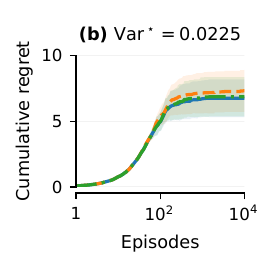}\hfill
\includegraphics[width=.329\linewidth]{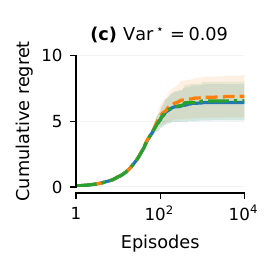}\\
\includegraphics[width=\linewidth]{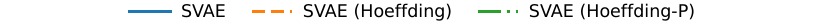}
\caption{\textbf{Regret under different safe-policy return variances.} The three panels compare SVAE with its Hoeffding-type variants as $\mathrm{Var}^{\star}$ varies. Curves show means over 20 paired seeds, with shading indicating one SD. The episode axes use a logarithmic scale.}
\label{fig:h-variance}
\end{figure}

\paragraph{Results}
SVAE has lower final regret on average than both Hoeffding-type variants in all three conditions (Figure~\ref{fig:h-variance} and Table~\ref{tab:h-variance}). Compared with SVAE (Hoeffding-P), SVAE selects suboptimal branching actions $1.10$--$1.20$ fewer times on average. This comparison isolates the effect of using empirical next-state value variance in the transition bonus, as both methods share the reward-variance estimates and safety filter.

\begin{table}[!htb]
\centering
\caption{\textbf{Accumulated variance and regret differences.} We report SVAE's accumulated variance $\mathbb V_\Sigma$ and paired regret differences at $K=10{,}000$ as means (SDs) over 20 seeds. Each difference $\Delta\mathcal R(K)$ is SVAE's regret minus the regret of the indicated Hoeffding-type variant.}
\label{tab:h-variance}
\small\setlength{\tabcolsep}{16pt}\renewcommand{\arraystretch}{1.13}
\begin{tabular}{@{}rrrcc@{}}
\toprule
& & & \multicolumn{2}{c}{$\Delta\mathcal R(K)$} \\
\cmidrule(lr){4-5}
$\mathrm{Var}^{\star}$ & $K\mathrm{Var}^{\star}$ & $\mathbb V_\Sigma$ & Hoeffding-P & Hoeffding \\
\midrule
0 & 0 & 0.000 (0.000) & $-0.110$ (0.031) & $-0.460$ (0.268) \\
0.0225 & 225 & 0.739 (0.263) & $-0.115$ (0.198) & $-0.560$ (0.376) \\
0.09 & 900 & 2.691 (1.078) & $-0.120$ (0.202) & $-0.495$ (0.386) \\
\bottomrule
\end{tabular}
\end{table}%
\FloatBarrier

The optimal safe action gives a deterministic return even when another safe action has positive return variance. As SVAE increasingly selects the optimal action, the variance it accumulates can therefore be much smaller than $K\mathrm{Var}^{\star}$. Let $a^\star$ and $a_b$ denote the deterministic optimal and safe branching root actions, respectively. Following Appendix~\ref{app:empirical-estimators}, $N_1^{K+1}(s_1,a)$ is the number of times action $a$ is taken at the root in the first $K$ episodes. Each branching action has an expected return $0.1$ below that of $a^\star$. Only $a_b$ contributes to $\mathbb V_\Sigma$, since this sum ends before the first unsafe action. Hence,
\begin{equation*}
\mathcal R(K)=0.1\sum_{a\ne a^\star}N_1^{K+1}(s_1,a),
\qquad
\mathbb V_\Sigma=u^2N_1^{K+1}(s_1,a_b).
\end{equation*}
SVAE's mean $\mathbb V_\Sigma$ is about $0.3\%$ of $K\mathrm{Var}^{\star}$ in both conditions with positive variance (Table~\ref{tab:h-variance}). In these instances, the term $\min\{\mathbb V_\Sigma,K\mathrm{Var}^{\star}\}$ in Theorem~\ref{thm:regret} reflects the smaller variance along the learning trajectory.

\subsection{Safety Identification Complexity}
\label{app:h-identification}

We study safety identification in a two-step CMDP with four states and two actions. At the root, one action is safe and gives a return of $0.4$. The other gives a return of $0.5$ but reaches a state with no safe action with probability $p$. Both actions at this successor have mean cost $\tau+\Delta_c$, while all other mean costs are $0.2$. The root action with the higher return is therefore unsafe because it can reach a state with no safe continuation. Every safe policy has a deterministic return, so $\mathrm{Var}^{\star}=0$. By~\eqref{eq:recursive-kappa}, the identification complexity is $\mathcal C=2/(p\Delta_c^2)$. We vary $p$ with $\Delta_c=0.2$ fixed, then vary $\Delta_c$ with $p=0.2$ fixed. We measure how often the unsafe root action is taken and how often its non-viable successor is visited.

\begin{figure}[!htb]
\centering
\includegraphics[width=.49\linewidth]{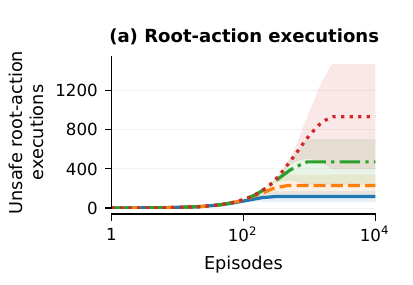}\hfill
\includegraphics[width=.49\linewidth]{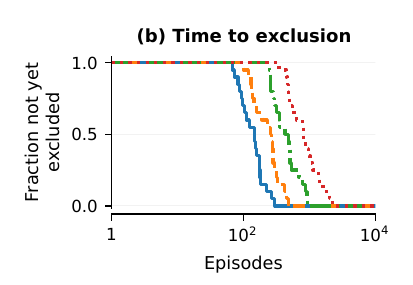}\\
\includegraphics[width=\linewidth]{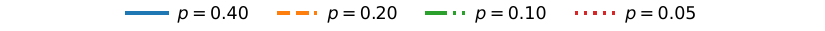}
\caption{\textbf{Effect of transition probability on safety identification.} Panel (a) shows the mean cumulative number of times the unsafe root action is taken, with shading indicating one SD. Panel (b) shows the fraction of runs in which this action has not yet been excluded. For each $p$, we use 20 paired seeds and fix $\Delta_c=0.2$. Both episode axes use a logarithmic scale.}
\label{fig:h-identification}
\end{figure}

\paragraph{Results}
As $p$ decreases from $0.40$ to $0.05$, the mean number of times the unsafe root action is taken increases from $117.1$ to $931.9$ (Figure~\ref{fig:h-identification}). The mean number of visits to the successor remains $46.5$, and cumulative violation remains $9.29$ (Table~\ref{tab:h-identification}). For each paired seed, all four conditions require the same number of cost observations at the successor. A smaller transition probability increases the number of interactions needed to collect these observations, as reflected by the reachability factor in~\eqref{eq:recursive-kappa}.

\begin{table}[!htb]
\centering
\caption{\textbf{Effect of transition probability and cost gap on safety identification.} We report executions of the unsafe root action, visits to its non-viable successor, and cumulative step-wise violation at $K=10{,}000$. Entries are means (SDs) over 20 seeds. The two parameter sweeps use separate seed sets. Bold rows mark the lowest identification complexity $\mathcal C$ in each sweep.}
\label{tab:h-identification}
\small\setlength{\tabcolsep}{16pt}\renewcommand{\arraystretch}{1.13}
\begin{tabular}{@{}rrrrrr@{}}
\toprule
$p$ & $\Delta_c$ & $\mathcal C$ & Root executions & Successor visits & $\mathrm{Viol}(K)$ \\
\midrule
\multicolumn{6}{@{}l@{}}{\textit{Varying the successor probability ($\Delta_c=0.20$)}} \\
\textbf{0.40} & \textbf{0.20} & \textbf{125.0} & \textbf{117.1 (57.8)} & \textbf{46.5 (21.3)} & \textbf{9.29 (4.27)} \\
0.20 & 0.20 & 250.0 & 230.1 (112.1) & 46.5 (21.3) & 9.29 (4.27) \\
0.10 & 0.20 & 500.0 & 471.0 (231.5) & 46.5 (21.3) & 9.29 (4.27) \\
0.05 & 0.20 & 1,000.0 & 931.9 (535.4) & 46.5 (21.3) & 9.29 (4.27) \\
\addlinespace[4pt]
\multicolumn{6}{@{}l@{}}{\textit{Varying the cost gap ($p=0.20$)}} \\
0.20 & 0.10 & 1,000.0 & 1,056.9 (599.7) & 214.1 (126.0) & 21.41 (12.60) \\
0.20 & 0.15 & 444.4 & 591.7 (391.5) & 118.7 (82.5) & 17.81 (12.37) \\
0.20 & 0.20 & 250.0 & 296.2 (123.0) & 59.3 (26.4) & 11.85 (5.28) \\
\textbf{0.20} & \textbf{0.30} & \textbf{111.1} & \textbf{126.7 (40.6)} & \textbf{26.0 (6.9)} & \textbf{7.79 (2.07)} \\
\bottomrule
\end{tabular}
\end{table}%
\FloatBarrier

Reducing the cost gap increases the number of observations needed at the successor. With $p=0.2$ fixed, decreasing $\Delta_c$ from $0.30$ to $0.10$ increases the mean number of successor visits from $26.0$ to $214.1$. Root-action executions and cumulative violation also increase. Together, the two sweeps show how rare successors and small cost gaps increase $\mathcal C$ at zero safe-policy return variance, illustrating the separate safety identification term in Theorem~\ref{thm:regret}. In every tested condition, SVAE excludes the unsafe root action within the episode budget and never takes it afterward. Cumulative violation then remains constant, consistent with the positive-gap bound in Theorem~\ref{thm:violation}.

\FloatBarrier

\subsection{Recursive Elimination}
\label{app:h-propagation}

We study recursive elimination by comparing SVAE with SVAE w/o propagation in a five-step CMDP with three states, two actions, and deterministic transitions. The two branches give returns of $0.4$ and $1$. On the branch with the higher return, the state reached after $d=3$ transitions has two actions with mean cost $0.8$. All other costs are zero. An otherwise identical instance sets one of these two costs to zero, leaving a safe continuation. Both instances have $\mathrm{Var}^{\star}=0$. We define detection as the first time the cost lower confidence bounds of both actions at the later state exceed $\tau$. We count entries into the branch with the higher return after detection.

\begin{figure}[!htb]
\centering
\includegraphics[width=.49\linewidth]{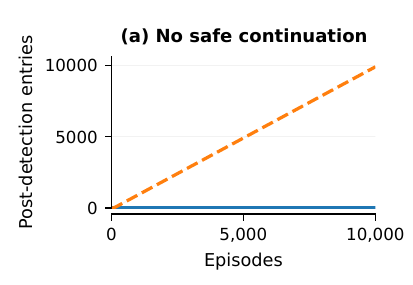}%
\includegraphics[width=.49\linewidth]{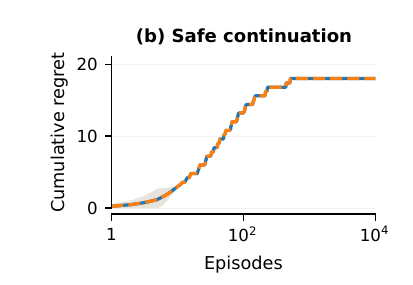}\\
\includegraphics[width=\linewidth]{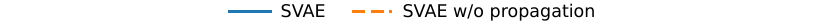}
\caption{\textbf{Recursive elimination at propagation distance $d=3$.} Panel (a) shows cumulative entries into the unsafe branch after detection. Panel (b) shows cumulative regret when the branch has a safe continuation, with a logarithmic episode axis. Curves show means over 20 paired seeds, with shading indicating one SD. The two methods have identical regret in (b) for every paired seed.}
\label{fig:h-propagation}
\end{figure}

\paragraph{Results}
At $d=3$, SVAE excludes the root action in the same backward pass that detects the absence of a safe continuation. It never enters the branch after detection in any of the 20 runs (Figure~\ref{fig:h-propagation}a). The variant without propagation continues to enter the branch, with $9{,}910.05$ entries after detection on average. SVAE incurs cumulative violation of $11.99$, compared with $1{,}994.00$ for the ablation. The $d=1$ condition gives the same result: zero entries after detection for SVAE and $9{,}921.95$ for the ablation. Recursive elimination uses evidence from the later state to exclude the root action, illustrating the control of unsafe candidate selections in Lemma~\ref{lem:false-candidate-control}.

When one action at the later state remains safe, SVAE retains the branch. SVAE and the ablation follow the same trajectories, with final regret $18.00$ and zero violation over the last 500 episodes (Figure~\ref{fig:h-propagation}b). SVAE's recursive update excludes the branch only when both downstream actions have been ruled out.

\FloatBarrier

\subsection{Additional Experimental Details}
\label{app:experimental-details}

\paragraph{Hyperparameter selection}
We scale the planning bonus by $\alpha$ and the cost-confidence radius by $\beta$. For navigation, each method is evaluated with 18 configurations. Methods with safety filtering use six planning multipliers combined with three cost-confidence multipliers, while unconstrained MVP-V uses 18 planning multipliers. To be eligible, configurations with safety filtering must have mean expected violation of at most $0.02$ per step. They must also retain all truly safe actions and keep a nonempty candidate set at every visited state that admits a safe continuation. Eligible configurations are ranked by increasing strong regret $\mathcal R_+$, and unconstrained MVP-V configurations by decreasing policy value. We first evaluate all configurations on two development environments with two seeds per environment. The best three are then evaluated on two separate validation environments, again with two seeds per environment. Both stages use $K=10{,}000$ episodes. After validation, we select the best configuration and fix its hyperparameters before testing.

For navigation, $(\alpha,\beta)$ is $(0.0003,0.30)$ for SVAE and $(0.001,0.30)$ for SUCBVI. MVP-V + filter uses $(3\times10^{-6},0.30)$, and unconstrained MVP-V uses \mbox{$\alpha\simeq2.405\times10^{-6}$}. In Q2--Q4, SVAE and its ablations use the same empirically chosen coefficients, \mbox{$(\alpha,\beta)=(0.003,0.17)$}. The complete search grids, coefficients, and random seeds are included with the code.

\paragraph{Implementation}
For MVP-V, we treat each state--step pair as a separate state and normalize rewards by $H$. The empirical model is updated whenever a visit count doubles. We report results on the original reward scale. For paired comparisons, we compute the difference between the two methods' results for each paired seed. For navigation, we average these differences over the five seeds within each environment, then report the mean (SD) across the eight environments.

\end{document}